\documentclass[12pt]{article}

\usepackage{fullpage}
\usepackage{amsfonts}
\usepackage{amssymb}
\usepackage{amstext}
\usepackage{amsmath}
\usepackage{natbib}
\bibpunct{(}{)}{;}{a}{,}{,}
  
\usepackage{jstyle}
\usepackage{mathtime}
\usepackage{graphicx}
\usepackage{epsfig}

\newcommand{\R}{\ensuremath{\mathbb R}}
\newcommand{\Q}{\ensuremath{\mathbb Q}}

\newcommand{\Z}{\ensuremath{\mathbb Z}}

\newcommand{\prob}[1]{\ensuremath{{\mathbb P}\left(#1\right)}}

\newcommand{\expct}[1]{\ensuremath{{\mathbb E}\left(#1\right)}}

\newcommand{\size}[1]{\ensuremath{\left|#1\right|}}

\newcommand{\argmin}{\operatorname{argmin}}

\newcommand{\e}{\epsilon}
\newcommand{\ve}{\varepsilon}

\newcommand{\half}{\ensuremath{\frac{1}{2}}}
\newcommand{\silent}[1]{}

\newcommand{\ip}[1]{\;\langle{\,#1\,}\rangle\;}

\newcommand{\mnote}[1]{\normalmarginpar \marginpar{\tiny #1}}

\newcommand{\Ball}{{B}}

\newcommand{\event}{{\mathcal E}}

\newcommand{\T}{{\mathcal T}}
\newcommand{\W}{{\mathcal W}}

\newcommand{\lasso}{\textsf{lasso2}}
\newcommand{\lars}{\textsf{LARS}}
\newcommand{\sign}{\text{sgn}}

\newcommand{\inv}[1]{\frac{1}{#1}}
\newcommand{\abs}[1]{\left\lvert#1\right\rvert}
\newcommand{\twonorm}[1]{\left\lVert#1\right\rVert_2}
\newcommand{\norm}[1]{\left\lVert#1\right\rVert}

\def\reals{{\mathbb R}}
\def\P{{\mathbb P}}
\def\E{{\mathbb E}}
\def\supp{\mathop{\text{supp}\kern.2ex}}
\def\argmin{\mathop{\text{arg\,min}\kern.2ex}}
\let\hat\widehat
\let\tilde\widetilde
\def\csd{${}^*$}
\def\mld{${}^\dag$}
\def\dos{${}^\ddag$}
\def\W{\widetilde Y}
\def\Z{\widetilde X}

\begin{document}

\mbox{\ }
\vskip.2in
\centerline{\LARGE\bf Compressed Regression}
\vskip .2in
\begin{center}
\begin{tabular}{c}
{\large Shuheng Zhou\csd\;\; John Lafferty\csd\mld\;\; Larry Wasserman\dos\mld} \\[15pt]
{\csd}Computer Science Department \\
{\mld}Machine Learning Department \\
{\dos}Department of Statistics \\[10pt]
Carnegie Mellon University \\
Pittsburgh, PA 15213 \\[20pt]
\today \\[5pt]
\end{tabular}
\end{center}

\begin{abstract}
\noindent\normalsize Recent research has studied the role of sparsity
in high dimensional regression and signal reconstruction, establishing
theoretical limits for recovering sparse models from sparse data.
This line of work shows that $\ell_1$-regularized least squares
regression can accurately estimate a sparse linear model from $n$
noisy examples in $p$ dimensions, even if $p$ is much larger than $n$.
In this paper we study a variant of this problem where the original
$n$ input variables are compressed by a random linear transformation
to $m \ll n$ examples in $p$ dimensions, and establish conditions
under which a sparse linear model can be successfully recovered from
the compressed data.  A primary motivation for this compression
procedure is to anonymize the data and preserve privacy by revealing
little information about the original data.  We characterize the
number of random projections that are required for
$\ell_1$-regularized compressed regression to identify the nonzero
coefficients in the true model with probability approaching one, a
property called ``sparsistence.''  In addition, we show that
$\ell_1$-regularized compressed regression asymptotically predicts as
well as an oracle linear model, a property called
``persistence.''  Finally, we characterize the privacy properties of
the compression procedure in information-theoretic terms, establishing
upper bounds on the mutual information between the compressed and
uncompressed data that decay to zero.  \vskip30pt
\noindent
\begin{quote}
\begin{itemize}
\item[\bf Keywords:] Sparsity, $\ell_1$ regularization, lasso,
high dimensional regression, privacy, capacity of multi-antenna channels, compressed sensing.
\end{itemize}
\end{quote}
\end{abstract}

\newpage
\tableofcontents
\newpage

\section{Introduction}
\label{sec:introduction}

Two issues facing the use of statistical learning methods in
applications are \textit{scale} and \textit{privacy}.  Scale is an
issue in storing, manipulating and analyzing extremely large, high
dimensional data.  Privacy is, increasingly, a concern whenever large
amounts of confidential data are manipulated within an organization.
It is often important to allow researchers to analyze data without
compromising the privacy of customers or leaking confidential
information outside the organization.  In this paper we show that
sparse regression for high dimensional data can be carried out
directly on a compressed form of the data, in a manner that can be
shown to guard privacy in an information theoretic sense.

The approach we develop here compresses the data by a random linear or
affine transformation, reducing the number of data records
exponentially, while preserving the number of original input
variables.  These compressed data can then be made available for
statistical analyses; we focus on the problem of sparse linear
regression for high dimensional data.  Informally, our theory ensures
that the relevant predictors can be learned from the compressed data
as well as they could be from the original uncompressed data.
Moreover, the actual predictions based on new examples are as accurate
as they would be had the original data been made available.  However,
the original data are not recoverable from the compressed data, and
the compressed data effectively reveal no more information than would
be revealed by a completely new sample.  At the same time, the
inference algorithms run faster and require fewer resources than the
much larger uncompressed data would require.  In fact, the original
data need never be stored; they can be transformed ``on the fly'' as
they come in.

In more detail, the data are represented as a $n\times
p$ matrix $X$.  Each of the $p$ columns is an attribute, and each of
the $n$ rows is the vector of attributes for an individual record.  The data
are compressed by a random linear transformation
\begin{eqnarray}
X &\mapsto& \Z \;\equiv\; \Phi X
\end{eqnarray} 
where $\Phi$ is a random $m\times n$ matrix with $m \ll n$. It
is also natural to consider a random affine transformation
\begin{eqnarray}
X &\mapsto& \Z \;\equiv\; \Phi X + \Delta
\end{eqnarray} 
where $\Delta$ is a random $m\times p$ matrix. Such transformations
have been called ``matrix masking'' in the privacy literature
\citep{duncan:91}.  
The entries of $\Phi$ and $\Delta$ are taken to be independent Gaussian random
variables, but other distributions are possible.  We think of $\Z$ as
``public,'' while $\Phi$ and $\Delta$ are private and only needed at
the time of compression.  However, even with $\Delta = 0$ and $\Phi$
known, recovering $X$ from $\Z$ requires solving a highly
under-determined linear system and comes with information theoretic
privacy guarantees, as we demonstrate.

In standard regression, a response 
$Y = X \beta + \epsilon \in \reals^n$ is associated with
the input variables, where $\epsilon_i$ are independent, mean zero 
additive noise variables. In compressed regression, we assume that the
response is also compressed, resulting in the transformed response
$\W \in\reals^m$ given by
\begin{eqnarray}
Y \;\mapsto\; \W  &\equiv& \Phi Y\\
  & =& \Phi X \beta + \Phi \epsilon \\
  & = & \Z \beta + \tilde\epsilon
\end{eqnarray}
Note that under compression,
the transformed noise $\tilde \epsilon = \Phi\epsilon$ is not
independent across examples.

In the sparse setting, the parameter vector
$\beta\in\reals^p$ is sparse, with a relatively small number $s$ of
nonzero coefficients $\supp(\beta) = \left\{j \,:\, \beta_j \neq
  0\right\}$.  Two key tasks are to identify the relevant variables,
and to predict the response $x^T \beta$ for a new input vector $x\in
\reals^p$.  The method we focus on is $\ell_1$-regularized least
squares, also known as the lasso \citep{Tib96}.
The main contributions of this paper
are two technical results on the performance of this estimator,
and an information-theoretic analysis of the privacy properties
of the procedure.   Our first result shows that
the lasso is {\it sparsistent} under compression, meaning that
the correct sparse set of relevant variables is identified asymptotically.
Omitting details and technical assumptions for clarity, our result 
is the following.

{\bf Sparsistence (Theorem~\ref{thm:recovery}):}
\begin{sl} 
If the number of compressed examples $m$ satisfies
\begin{eqnarray}
C_1 s^2 \log nps  \; \leq\;  m \; \leq \; \sqrt{\frac{C_2 n}{\log  n}},
\end{eqnarray}
and the regularization parameter $\lambda_m$ satisfies
\begin{eqnarray}
\lambda_m \rightarrow 0 \quad \text{and}\quad \frac{m \lambda_m^2}{\log p} \rightarrow \infty,
\end{eqnarray}
then the compressed lasso solution
\begin{eqnarray}
\tilde \beta_{m} = \arg\min_\beta \, \frac{1}{2m} \| \W - \Z\beta\|_2^2 + \lambda_m
\|\beta\|_1
\end{eqnarray}
includes the  correct variables, asymptotically:
\begin{eqnarray}
\P\left(\supp(\tilde\beta_m) = \supp(\beta)\right) \rightarrow 1.
\end{eqnarray}
\end{sl}
Our second result shows that the lasso is {\it persistent} under
compression.  Roughly speaking, persistence \citep{GR04} 
means that the procedure predicts well, as measured by 
the predictive risk
\begin{equation}
R(\beta) = \E\left(Y- X \beta\right)^2,
\end{equation}
where now $X\in\reals^p$ is a new input vector and $Y$ is the
associated response. Persistence is a weaker condition than sparsistency, 
and in particular does not assume that the true model is linear.

{\bf Persistence (Theorem~\ref{thm:persistence}):}  
\begin{sl}
Given a sequence of sets of estimators $\Ball_{n, m}$, 
the sequence of compressed lasso estimators
\begin{eqnarray}
\tilde\beta_{n,m} &=& \argmin_{\|\beta\|_1 \leq L_{n,m}} \|\W - \Z \beta\|_2^2
\end{eqnarray}
is persistent with the oracle risk over uncompressed data with respect to
$\Ball_{n, m}$, meaning that
\begin{eqnarray}
R(\tilde\beta_{n,m}) - \inf_{\|\beta\|_1 \leq L_{n,m}} 
R(\beta) \;\stackrel{P}{\longrightarrow}\; 0, \; \; \text{ as } n \to \infty.
\end{eqnarray}
in case $\log^2(n p) \leq m \leq n$ and the radius of the $\ell_1$ ball satisfies
$L_{n,m} = o\left(m/\log (n p)\right)^{1/4}$.
\end{sl}

Our third result analyzes the privacy properties of compressed
regression.  We consider the problem of recovering the uncompressed
data $X$ from the compressed data $\Z = \Phi X + \Delta$.  To preserve
privacy, the random matrices $\Phi$ and $\Delta$ should remain
private.  However, even in the case where $\Delta=0$ and $\Phi$ is
known, if $m \ll \min(n,p)$ the linear system $\Z = \Phi X$ is highly
underdetermined.  We evaluate privacy in information theoretic terms
by bounding the average mutual information $I(\Z; X)/np$ per matrix
entry in the original data matrix $X$, which can be viewed as a
communication rate. Bounding this mutual information is intimately
connected with the problem of computing the channel capacity of
certain multiple-antenna wireless communication systems
\citep{Marzetta:99,Telatar:99}.

{\bf Information Resistence (Propositions~\ref{prop:priva} and~\ref{prop:privb}):}  
\begin{sl}
The rate at which information about $X$ is revealed 
by the compressed data $\Z$ satisfies
\begin{equation}
r_{n,m} = \sup\, \frac{I(X;\Z)}{np} \;=\; O\left(\frac{m}{n}\right) \rightarrow 0,
\end{equation}
where the supremum is over distributions on the original data $X$.
\end{sl}

As summarized by these results, compressed regression is a practical
procedure for sparse learning in high dimensional data that has
provably good properties.  This basic technique has connections in the
privacy literature with matrix masking and other methods, yet most of
the existing work in this direction has been heuristic and without
theoretical guarantees; connections with this literature are briefly
reviewed in Section~\ref{sec:background:privacy}.  Compressed regression builds
on the ideas underlying compressed sensing and sparse inference in
high dimensional data, topics which have attracted a great deal of
recent interest in the statistics and signal processing communities;
the connections with this literature are reviewed in
Section~\ref{sec:background:cs} and~\ref{sec:background:sparse}.

The remainder of the paper is organized as follows.  In
Section~\ref{sec:background} we review relevant work from high
dimensional statistical inference, compressed sensing and privacy.
Section~\ref{sec:sparsistence} presents our analysis of the
sparsistency properties of the compressed lasso.  Our approach follows
the methods introduced by \cite{Wai06} in the
uncompressed case.  Section~\ref{sec:persistence} proves that
compressed regression is persistent.  Section~\ref{sec:privacy}
derives upper bounds on the mutual information between the compressed
data $\Z$ and the uncompressed data $X$, after identifying a
correspondence with the problem of computing channel capacity for a
certain model of a multiple-antenna mobile communication channel.
Section~\ref{sec:experiments} includes the results of experimental
simulations, showing that the empirical performance of the compressed
lasso is consistent with our theoretical analysis.  We evaluate the
ability of the procedure to recover the relevant variables
(sparsistency) and to predict well (persistence).  The
technical details of the proof of sparsistency are collected at the
end of the paper, in Section~\ref{sec:proofs}.  The paper
concludes with a discussion of the results and directions
for future work in Section~\ref{sec:discuss}.

\section{Background and Related Work}
\label{sec:background}
In this section we briefly review relevant related work
in high dimensional statistical inference, compressed sensing, and
privacy, to place our work in context.

\subsection{Sparse Regression}
\label{sec:background:sparse}
We adopt standard notation where a data matrix $X$ has $p$ variables
and $n$ records; in a linear model the response $Y = X\beta + \epsilon \in\reals^n$ 
is thus an $n$-vector, and the noise $\epsilon_i$ is independent and
mean zero, $\E(\epsilon) = 0$. The usual estimator of
$\beta$ is the least squares estimator
\begin{equation}
\hat\beta = (X^T X)^{-1} X^T Y .
\end{equation}
However, this estimator has very large variance when $p$ is large,
and is not even defined when $p > n$.
An estimator that has received much attention in the recent literature
is the \textit{lasso} $\hat\beta_n$ \citep{Tib96},
defined as
\begin{eqnarray}
\hat\beta_n &=& \argmin \frac{1}{2n} \sum_{i=1}^n (Y_i - X_i^T\beta)^2 + \lambda_n
\sum_{j=1}^p |\beta_j| \\
       &=& \argmin \frac{1}{2n} \|Y - X\beta\|_2^2 + \lambda_n \|\beta\|_1,
\end{eqnarray}
where $\lambda_n$ is a regularization parameter.  The practical
success and importance of the lasso can be attributed to the fact that
in many cases $\beta$ is sparse, that is, it has few large components.
For example, data are often collected with many variables in the hope
that at least a few will be useful for prediction.  The result is that
many covariates contribute little to the prediction of $Y$, although
it is not known in advance which variables are important.  Recent work
has greatly clarified the properties of the lasso estimator in the
high dimensional setting.

One of the most basic desirable properties of an estimator is
consisistency; an estimator $\hat\beta_n$ is
\textit{consistent} in case
\begin{equation}
\|\hat\beta_n - \beta\|_2 \stackrel{P}{\to} 0.
\end{equation}
\cite{MY06} have recently shown that the lasso is consistent
in the high dimensional setting.   If the underlying model
is sparse, a natural yet more demanding criterion
is to ask that the estimator correctly identify the relevant
variables.  This may be useful for interpretation, dimension reduction
and prediction.  For example, if an effective  procedure for
high-dimensional data can be used to
identify the relevant variables in the model, then these variables
can be isolated and their coefficients estimated by a separate procedure
that works well for low-dimensional data.
An estimator is {\em sparsistent}\footnote{This
terminology is due to Pradeep Ravikumar.} if
\begin{eqnarray}
\mathbb{P}\left({\rm supp}(\hat\beta_n) = {\rm supp}(\beta)\right) \to 1,
\end{eqnarray}
where ${\rm supp}(\beta) = \{j: \ j\neq 0\}$.  Asymptotically, a sparsistent
estimator has nonzero coefficients only for the true relevant
variables.  Sparsistency proofs for high dimensional problems have
appeared recently in a number of settings.  \cite{MB06}
consider the problem of estimating the graph underlying a sparse
Gaussian graphical model by showing sparsistency of
the lasso with exponential rates of convergence on the probability
of error. \cite{ZY07} show sparsistency of the lasso under more 
general noise distributions. \cite{Wai06} characterizes 
the sparsistency properties of the lasso by showing that there
is a threshold sample size $n(p,s)$ above which the relevant variables
are identified, and below which the relevant variables fail to be
identified, where $s=\|\beta\|_0$ is the number of relevant variables.  
More precisely, \cite{Wai06} shows that when $X$
comes from a Gaussian ensemble, there exist fixed constants 
$0  < \theta_{\ell} \leq 1$ and $1 \leq \theta_u < +\infty$, 
where 
$\theta_{\ell} = \theta_u =1$ when each row of $X$ is chosen as an independent
Gaussian random vector $\sim N(0, I_{p \times p})$,
then for any $\nu > 0$, if 
\begin{eqnarray}
\label{eq:wain-succ-bound}
n > 2(\theta_u+\nu) s \log(p - s) + s + 1,
\end{eqnarray}
then the lasso identifies the true variables with probability
approaching one.  Conversely, if
\begin{eqnarray}
\label{eq:wain-fail-bound}
n < 2(\theta_{\ell}-\nu) s \log(p - s) + s + 1,
\end{eqnarray}
then the probability of recovering the true variables using the lasso
approaches zero.  These results require certain \textit{incoherence}
assumptions on the data $X$; intuitively, it is required that an
irrelevant variable cannot be too strongly correlated with the set of
relevant variables.  This result and Wainwright's method of analysis
are particularly relevant to the current paper; the details will be
described in the following section. In particular, we refer to this
result as the Gaussian Ensemble result. However, it is
important to point out that under compression, the noise $\tilde
\epsilon = \Phi\epsilon$ is not independent.  This prevents one from
simply applying the Gaussian Ensemble results to the compressed case.
Related work that studies information theoretic limits of sparsity
recovery, where the particular estimator is not specified, includes
\citep{Wai07,Donoho:Tanner:06}.  Sparsistency in the classification
setting, with exponential rates of convergence for
$\ell_1$-regularized logistic regression, is studied by
\cite{wain:07}.

An alternative goal is accurate prediction.  In high dimensions it is
essential to regularize the model in some fashion in order to control
the variance of the estimator and attain good predictive risk.
Persistence for the lasso was first defined and studied by~\cite{GR04}.
Given a sequence of sets of estimators $\Ball_n$, the sequence of estimators
$\hat\beta_n\in \Ball_n$ is called \textit{persistent} in case
\begin{equation}
R(\hat\beta_n) - \inf_{\beta\in \Ball_n} R(\beta) \stackrel{P}{\to} 0,
\end{equation}
where $R(\beta) = \mathbb{E}(Y-X^T \beta)^2$
is the prediction risk of a new pair $(X,Y)$.
Thus, a sequence of estimators is persistent if 
it asymptotically predicts as well as the oracle 
within the class, which minimizes the population risk;
it can be achieved under weaker assumptions than are required 
for sparsistence.  
In particular, persistence does not assume the true
model is linear, and it does not require strong incoherence
assumptions on the data.  The results of the current
paper show that sparsistence and persistence are
preserved under compression.

\subsection{Compressed Sensing}
\label{sec:background:cs}

Compressed regression has close connections to, and draws motivation
from, compressed sensing
\citep{Donoho:cs,Candes:Romberg:Tao:06,Candes:Tao:06,RSV07}.  However, in a sense,
our motivation here is the opposite to that of compressed sensing.
While compressed sensing of $X$ allows a sparse $X$ to be reconstructed
from a small number of random measurements, our goal is
to reconstruct a sparse function of $X$.  Indeed,
from the point of view of privacy, approximately reconstructing $X$, which
compressed sensing shows is possible if $X$ is sparse, should
be viewed as undesirable; we return to this point in Section~\ref{sec:privacy}.

Several authors have considered variations on compressed sensing for
statistical signal processing tasks
\citep{Duarte:06,Davenport:06,Haupt:06,Davenport:06b}.  The focus of
this work is to consider certain hypothesis testing problems under
sparse random measurements, and a generalization to classification of
a signal into two or more classes.  Here one observes $y = \Phi x$,
where $y\in\reals^m$, $x\in\reals^n$ and $\Phi$ is a known random
measurement matrix.  The problem is to select between the hypotheses
\begin{eqnarray}
\tilde H_i: \; y = \Phi(s_i + \epsilon),
\end{eqnarray}
where $\epsilon\in\reals^n$ is additive Gaussian noise.  Importantly, the
setup exploits the ``universality'' of the matrix $\Phi$, which is
not selected with knowledge of $s_i$. The proof techniques
use concentration properties of random projection, which underlie the
celebrated lemma of \cite{JL:84}. The compressed
regression problem we introduce can be considered as a more
challenging statistical inference task, where the problem is to select
from an exponentially large set of linear models, each with a certain
set of relevant variables with unknown parameters, or to predict as
well as the best linear model in some class.  Moreover, a key
motivation for compressed regression is privacy; if privacy is not a
concern, simple subsampling of the data matrix could be an effective
compression procedure.

\subsection{Privacy}
\label{sec:background:privacy}
Research on privacy in statistical data analysis has a long history,
going back at least to \cite{Dalenius:77}; we refer to
\cite{duncan:91} for discussion and further pointers into this
literature. The compression method we employ has been called
\textit{matrix masking} in the privacy literature.  In the general
method, the $n\times p$ data matrix $X$ is transformed by
pre-multiplication, post-multiplication, and addition into a new
$m\times q$ matrix
\begin{eqnarray}
\Z = A X B + C.
\end{eqnarray}
The transformation $A$ operates on data records for fixed
covariates, and the transformation $B$ operates on 
covariates for a fixed record.  The method encapsulated in
this transformation is quite general,
and allows the possibility of deleting records, suppressing
subsets of variables, data swapping, and including
simulated data.  In our use of matrix masking, we
transform the data by replacing each variable with
a relatively small number of random averages
of the instances of that variable in the data.  
In other work, \cite{Sanil:04} consider the problem of privacy preserving regression
analysis in distributed data, where different variables appear
in different databases but it is of interest to integrate data
across databases.
The recent work of \cite{Ting:07} considers
random orthogonal mappings $X \mapsto RX = \Z$ where 
$R$ is a random rotation (rank $n$), designed to preserve the sufficient
statistics of a multivariate Gaussian and therefore
allow regression estimation, for instance.  This use of matrix
masking does not share the information theoretic guarantees
we present in Section~\ref{sec:privacy}.  
We are not aware of previous work that analyzes the asymptotic properties
of a statistical estimator under matrix masking in the high
dimensional setting.

The work of \cite{Liu:06} is closely related to the current paper at a
high level, in that it considers low rank random linear
transformations of either the row space or column space of the data
$X$.  \cite{Liu:06} note the Johnson-Lindenstrauss lemma, which
implies that $\ell_2$ norms are approximately preserved under random
projection, and argue heuristically that data mining procedures that
exploit correlations or pairwise distances in the data, such as
principal components analysis and clustering, are just as effective
under random projection.  The privacy analysis is restricted to
observing that recovering $X$ from $\Z$ requires solving an
under-determined linear system, and arguing that this prevents the
exact values from being recovered.

An information-theoretic quantification 
of privacy was formulated by \cite{Agrawal:01}.  Given a random
variable $X$ and a transformed variable $\Z$, \cite{Agrawal:01}
define the conditional privacy loss of $X$ given $\Z$ as
\begin{eqnarray}
{\mathcal  P}(X\,|\, \Z) = 1 - 2^{-I(X;\Z)},
\end{eqnarray}
which is simply a transformed measure of the mutual information
between the two random variables.  In our work we identify
privacy with the rate of information communicated about $X$
through $\Z$ under matrix masking, maximizing over all distributions
on $X$.  We furthermore identify this with the problem of computing, or bounding,
the Shannon capacity of a multi-antenna wireless communication
channel, as modeled by \cite{Telatar:99} and \cite{Marzetta:99}.

Finally, it is important to mention the extensive and currently active
line of work on cryptographic approaches to privacy, which have come
mainly from the theoretical computer science community.  For instance,
\cite{feigen:06} develop a framework for secure computation of
approximations; intuitively, a private approximation of a function $f$
is an approximation $\hat f$ that does not reveal information about
$x$ other than what can be deduced from $f(x)$.  \cite{indyk:06}
consider the problem of computing private approximate nearest
neighbors in this setting. \cite{Dwork:06} revisits the notion of
privacy formulated by \cite{Dalenius:77b}, which intuitively demands
that nothing can be learned about an individual record in a database
that cannot be learned without access to the database.  An
impossibility result is given which shows that, appropriately
formalized, this strong notion of privacy cannot be achieved.  An
alternative notion of \textit{differential privacy} is proposed, which
allows the probability of a disclosure of private information to
change by only a small multiplicative factor, depending on whether or
not an individual participates in the database.  This line of work has
recently been built upon by \cite{Dwork:07}, with connections to
compressed sensing, showing that any method that gives accurate
answers to a large fraction of randomly generated subset sum queries
must violate privacy.

\def\mnote#1{\marginpar{\raggedright\small #1}}
\def\ip#1#2{\langle #1, #2\rangle}
\section{Compressed Regression is Sparsistent}
\label{sec:sparsistence}

In the standard setting, $X$ is a $n \times p$ matrix, $Y = X\beta +
\epsilon$ is a vector of noisy observations under a linear model, and 
$p$ is considered to be a constant.  In the high-dimensional setting
we allow $p$ to grow with $n$.  The lasso refers to the following
quadratic program:
\begin{gather}
(P_1) \; \; \text{minimize} \;\; \|Y-X\beta\|_2^2 \;\; \text{such that} \;\; \|\beta\|_1 \leq L.
\end{gather}
In Lagrangian form, this becomes the optimization problem
\begin{eqnarray}
(P_2) && \text{minimize} \;\; \frac{1}{2n}\|Y-X\beta\|_2^2 + \lambda_n \|\beta\|_1,
\end{eqnarray}
where the scaling factor $1/2n$ is chosen by convention and convenience.
For an appropriate choice of the regularization parameter
$\lambda = \lambda(Y, L)$, the solutions of these two 
problems coincide.

In compressed regression we project each column $X_{j}\in\reals^n$ 
of $X$ to a subspace of $m$ dimensions, using an $m \times n$ random 
projection matrix $\Phi$.  We shall assume that the entries 
of $\Phi$ are independent Gaussian random variables:
\begin{eqnarray}
\label{eq:phi-define}
\Phi_{ij} &\sim& N(0,1/n).
\end{eqnarray}
Let $\Z = \Phi X$ be the compressed matrix of covariates, and
let $\W = \Phi Y$ be the compressed response.  
Our objective is to estimate $\beta$ in order to 
determine the relevant variables, or to predict well.  
The compressed lasso is the optimization problem, for
$\W = \Phi X \beta + \Phi \epsilon  = \Phi \Z + \tilde \epsilon$:
\begin{eqnarray}
\label{eq:compressedlasso}
(\tilde P_2) && \text{minimize} \;\; \frac{1}{2m}\|\W-\Z\beta\|_2^2 + \lambda_m \|\beta\|_1,
\end{eqnarray}
with $\tilde\Omega_m$ being the set of optimal solutions:
\begin{equation}
\label{eq:solution-set}
\tilde\Omega_m = \argmin_{\beta \in \R^p} \; 
\frac{1}{2m}\|\W-\Z\beta\|_2^2 + \lambda_m \|\beta\|_1.
\end{equation}
Thus, the transformed noise $\tilde \epsilon$ is no longer i.i.d.,
a fact that complicates the analysis. 
It is convenient to formalize the model selection problem using
the following definitions.
\begin{definition}\textnormal{\bf{(Sign Consistency)}} 
A set of estimators $\Omega_n$ is \textit{sign consistent} 
with the true $\beta$ if
\begin{gather}
\prob{\exists \hat\beta_n \in \Omega_n \, \text{s.t.} \, 
\sign(\hat{\beta}_n) = \sign(\beta)} \rightarrow 1 
\;\;\text{as $n\rightarrow \infty$},
\end{gather}
where $\sign(\cdot)$ is given by
\begin{eqnarray}
\sign(x) = \begin{cases}
1 & \text{if $x > 0$} \\
0 & \text{if $x=0$} \\
-1 & \text{if $x < 0$}.
\end{cases}
\end{eqnarray}
As a shorthand, we use
\begin{eqnarray}
\event\left(\sign(\hat\beta_n) = \sign(\beta^*)\right) & := &
\left\{ \exists \hat\beta \in \Omega_n \, \text{such that $
\sign(\hat\beta) = \sign(\beta^*)$}\right\}
\end{eqnarray}
to denote the event that a sign consistent solution exists.
\end{definition}

\silent{
}
The lasso objective function is convex in $\beta$, and strictly
convex for $p \leq n$.  Therefore the set of solutions to the lasso
and compressed lasso \eqref{eq:compressedlasso} is convex: if
$\hat\beta$ and $\hat\beta'$ are two solutions, then by convexity
$\hat\beta + \rho(\hat\beta'-\hat\beta)$ is also a solution for any
$\rho\in[0,1]$.
\begin{definition}\textnormal{\bf{(Sparsistency)}} 
A set of estimators $\Omega_n$ is \textit{sparsistent} with the true $\beta$ if
\begin{gather}
\prob{\exists \hat\beta_n \in \Omega_n \, \text{s.t.} \, 
\supp(\hat{\beta}_n) = \supp(\beta)} \rightarrow 1 
\;\;\text{as $n\rightarrow \infty$},
\end{gather}
\end{definition}
Clearly, if a set of estimators is sign consistent then it is sparsistent.
Although sparsistency is the primary goal in selecting the correct
variables, our analysis establishes conditions for the slightly
stronger property of sign consistency.

All recent work establishing results on sparsity recovery assumes some
form of \textit{incoherence condition} on the data matrix $X$.  Such a
condition ensures that the irrelevant variables are not too strongly
correlated with the relevant variables.  Intuitively, without such
a condition the lasso may be subject to false positives and
negatives, where an relevant variable is replaced by a highly
correlated relevant variable.
\def\Sc{{S^c}}
\def\onen{{\textstyle \frac{1}{n}}}
\def\onem{{\textstyle \frac{1}{m}}}
To formulate such a condition,
it is convenient to introduce an additional piece of notation.
Let $S = \{j: \beta_j \neq 0\}$ be the set of relevant variables and let 
$S^c = \{1, \ldots, p\} \setminus S$ be the set of irrelevant variables.
Then $X_S$ and $X_{\Sc}$ denote the corresponding sets of columns
of the matrix $X$. We will impose the following incoherence condition; 
related conditions are used by \cite{DET06} and
\cite{Tro04} in a deterministic setting.
\begin{definition}\textnormal{\bf{($S$-Incoherence)}}
\label{def:incoh-cond}
Let $X$ be an $n \times p$ matrix and let 
$S\subset \{1,\ldots, p\}$ be nonempty.
We say that $X$ is \textit{$S$-incoherent} in case
\begin{gather}
\label{eq:eta}
\norm{\onen X_{S^c}^T X_S}_{\infty} + \norm{\onen X_{S}^T X_S - I_{\size{S}}}_{\infty} \leq 1 - \eta, \;\;\text{for some $\eta\in(0,1],$}
\end{gather}
where $\|A\|_\infty = \max_i \sum_{j=1}^p |A_{ij}|$ denotes the matrix
$\infty$-norm.
\end{definition}
Although it is not explicitly required, we only apply this definition to 
$X$ such that columns of $X$ satisfy 
$\twonorm{X_j}^2 = \Theta(n), \forall j \in \{1, \ldots, p\}$.
We can now state the main result of this section.
\begin{theorem}
\label{thm:recovery}
Suppose that, before compression, we have $Y = X \beta^* + \e$, 
where each column of $X$ is normalized to have $\ell_2$-norm $n$,
and $\ve \sim N(0, \sigma^2 I_n)$.  Assume that $X$ is 
$S$-incoherent, where $S = \supp(\beta^*)$, and define
$s = |S|$ and $\rho_m = \min_{i \in S} |\beta_i^*|$. We observe, after compression,
\begin{gather}
\label{eq:compressed-noise}
\W = \Z \beta^* + \tilde\e,
\end{gather}  
where $\W = \Phi Y$, $\Z = \Phi X$, and $\tilde\e = \Phi \e$, where 
$\Phi_{ij} \sim N(0, 1/n)$.
Suppose 
\begin{equation}
\label{eq:thm-m-bounds}
\left(\frac{16 C_1 s^2}{\eta^2} + \frac{4 C_2 s}{\eta}\right)
(\ln p + 2 \log n + \log 2(s+1)) \leq m \leq \sqrt{\frac{n}{16 \log n}}
\end{equation}
with $C_1 = \frac{4 e}{\sqrt{6\pi}} \approx 2.5044$ and 
$C_2 = \sqrt{8e} \approx 7.6885$, and
$\lambda_m \rightarrow 0$ satisfies
\begin{eqnarray}
\label{eq:thm-cond-lambda}
(a)\;\;\frac{m \eta^2 \lambda_m^2}{\log(p-s)} \rightarrow \infty,\;\;\text{and}\;\;
(b)\;\; \inv{\rho_m}\left\{\sqrt{\frac{\log s}{m}}+ 
\lambda_m \norm{(\onen X_S^T X_S)^{-1}}_{\infty} \right \} \rightarrow 0.
\end{eqnarray}
Then the compressed lasso is sparsistent:
\begin{equation}
\prob{\supp(\tilde{\beta}_m) = \supp(\beta)} \rightarrow 1 
\;\;\text{as $m\rightarrow \infty$},
\end{equation}
where $\tilde\beta_m$ is an optimal solution to \eqref{eq:compressedlasso}.
\end{theorem}

\def\mnote#1{\marginpar{\raggedright\small #1}}
\def\ip#1#2{\langle #1, #2\rangle}
\def\Sc{{S^c}}
\def\onen{{\textstyle \frac{1}{n}}}
\def\onem{{\textstyle \frac{1}{m}}}

\subsection{Outline of Proof for Theorem~\ref{thm:recovery}}
\label{sec:thm-sparsity-proof-outline}
Our overall approach is to follow a deterministic analysis,
in the sense that we analyze $\Phi X$ as a realization from the
distribution of $\Phi$ from a Gaussian ensemble.  Assuming that
$X$ satisfies the $S$-incoherence condition, we show that
with high probability $\Phi X$ also satisfies the $S$-incoherence 
condition, and hence the incoherence conditions \eqref{eq:incoa} and \eqref{eq:incob}
used by \cite{Wai06}.  In addition, we make use of 
a large deviation result that shows $\Phi\Phi^T$ 
is concentrated around its mean $I_{m\times m}$, which is
crucial for the recovery of the true sparsity pattern.
It is important to note that the compressed noise $\tilde\e$ is not 
independent and identically distributed, even when conditioned on $\Phi$.

In more detail, we first show that with high probability
$1 - n^{-c}$ for some $c \geq 2$, the projected data
$\Phi X$ satisfies the following properties:
\begin{enumerate}
\item  Each column of $\Z = \Phi X$ has $\ell_2$-norm at most $m(1 + \eta/4s)$;
\item  $\Z$ is $S$-incoherent, and also satisfies the incoherence
 conditions \eqref{eq:incoa} and \eqref{eq:incob}.
\end{enumerate}
In addition, the projections satisfy the following properties:
\begin{enumerate}
\item Each entry of $\Phi\Phi^T -I$
is at most $\sqrt{b \log n / n}$ for some constant $b$,
with high probability;
\item $\prob{|\frac{n}{m}\ip{\Phi x}{\Phi y} - \ip{x}{y}| \geq \tau} \leq 
2 \exp\left(-\frac{m \tau^2}{C_1 + C_2 \tau}\right)$ for any
$x,y\in\reals^n$ with $\|x\|_2, \|y\|_2 \leq 1$.
\end{enumerate}
These facts allow us to condition on a ``good'' $\Phi$ and
incoherent $\Phi X$, and to proceed as in the deterministic
setting with Gaussian noise.  
Our analysis then follows that of \cite{Wai06}.  
Recall $S$ is the set of relevant variables in $\beta$ and
$S^c = \{1, \ldots, p\} \setminus S$ is the set of irrelevant variables.
To explain the basic approach, first observe that
the KKT conditions imply that 
$\tilde{\beta} \in \R^p$ is an optimal solution to
~(\ref{eq:compressedlasso}), i.e., $\tilde{\beta} \in \tilde\Omega_m$, 
if and only if there exists a subgradient 
\begin{gather}
\tilde{z} \in 
\partial \|\tilde{\beta}\|_1 = 
\left\{
z \in \R^p \,|\, \text{$z_i = \sign(\tilde{\beta}_i)$
 for $\tilde{\beta}_i \neq 0$, and
$\abs{\tilde{z}_j} \leq 1$ otherwise}
\right\}
\end{gather}
such that
\begin{gather}
\inv{m} \Z^T \Z \tilde{\beta} - \inv{m} \Z^T \W + \lambda \tilde{z} = 0.
\end{gather}
Hence, the $\event\left(\sign(\tilde\beta) = \sign(\beta^*)\right)$ 
can be shown to be equivalent to requiring the existence of 
a solution $\tilde{\beta} \in \R^p$ such that
$\sign(\tilde\beta) = \sign(\beta^*)$, 
and a subgradient $\tilde{z}\in\partial\|\tilde \beta\|_1$,
such that the following equations hold:
\begin{subeqnarray}
\label{leq:Sc}
\inv{m} \Z_{S^c}^T \Z_S(\tilde{\beta_S} - \beta_S^*) - 
\inv{m} \Z_{S^c}^T \tilde\e & = & -\lambda \tilde{z}_{S^c}, \\ 
\label{leq:S}
\inv{m} \Z_{S}^T \Z_S(\tilde{\beta_S} - \beta_S^*) - 
\inv{m} \Z_{S}^T \tilde\e & = & -\lambda \tilde{z}_{S}
= -\lambda \sign(\beta_S^*),
\end{subeqnarray}
where $\tilde{z}_{S} = \sign(\beta_S^*)$ and $\abs{\tilde{z}_{S^c}} \leq 1$
by definition of $\tilde{z}$. The existence of solutions to 
equations \eqref{leq:Sc} and \eqref{leq:S} 
can be characterized in terms of two events 
$\event(V)$ and $\event(U)$.  The proof proceeds by showing 
that $\P(\event(V)) \rightarrow 1$ and 
$\P(\event(U))\rightarrow 1$ as $m\rightarrow\infty$.

In the remainder of this section we present the main
steps of the proof, relegating the technical details to
Section~\ref{sec:proofs}. To avoid unnecessary clutter in notation,
we will use $Z$ to denote the compressed data $\Z = \Phi X$ and 
$W$ to denote the compressed response $\W = \Phi Y$,
and $\omega = \tilde\e$ to denote the compressed noise.  

\def\Z{Z}

\subsection{Incoherence and Concentration Under Random Projection}
\label{sec:concentrate}
In order for the estimated $\tilde{\beta}_m$ to be close to the
solution of the uncompressed lasso, we require the stability of inner
products of columns of $X$ under multiplication with the random matrix
$\Phi$, in the sense that
\begin{gather}
\ip{\Phi X_{i}}{\Phi X_{j}} \approx \ip{X_{i}}{X_{j}}.
\end{gather}
Toward this end we have the following result, adapted from \cite{RSV07}, 
where for each
entry in $\Phi$, the variance is $\inv{m}$ instead of $\inv{n}$.  

\begin{lemma}\textnormal{\bf{(Adapted from \cite{RSV07})}}
\label{lemma:adapt-RSV}
Let $x, y \in \R^n$ with $\twonorm{x}, \twonorm{y} \leq 1$. Assume
that $\Phi$ is an $m \times n$ random matrix with independent $N(0,n^{-1})$
entries (independent of $x, y$). Then for all $\tau > 0$
\begin{gather}
\prob{\abs{\frac{n}{m}\ip{\Phi x}{\Phi y} - \ip{x}{y}} \geq \tau} \leq 
2 \exp \left(\frac{- m \tau^2}{C_1 + C_2 \tau}\right)
\end{gather}
with $C_1 = \frac{4 e}{\sqrt{6\pi}} \approx 2.5044$ and 
$C_2 = \sqrt{8e} \approx 7.6885$.
\end{lemma}

We next summarize the properties of $\Phi X$ that we require. The
following result implies that, with high probability, incoherence is
preserved under random projection.

\begin{proposition}
\label{pro:Phi-X}
Let $X$ be a (deterministic) design matrix that is $S$-incoherent with
$\ell_2$-norm $n$,
and let $\Phi$ be a $m \times n$ random matrix with 
independent $N(0, n^{-1})$ entries.  Suppose that
\begin{equation}
m \geq \left(\frac{16 C_1 s^2}{\eta^2} + \frac{4 C_2 s}{\eta}\right) 
(\ln p + c \ln n + \ln 2(s+1))
\end{equation}
for some $c\geq 2$, where $C_1, C_2$ are defined in
Lemma~\ref{lemma:adapt-RSV}.  Then 
with probability at least $1 - 1/n^c$ the following properties 
hold for $Z = \Phi X$:
\begin{enumerate}
\item
$Z$ is $S$-incoherent; in particular:
\begin{subeqnarray}
\label{eq:CNorm2}
\abs{\norm{\onem Z_S^T Z_S - I_s}_\infty - 
\norm{\onen X_S^T X_S-I_s}_\infty} &\leq& \frac{\eta}{4}, \\[5pt]
\label{eq:inner-product}
\norm{\onem Z_{\Sc}^T Z_S}_{\infty} + 
\norm{\onem Z_S^T Z_S - I_s}_{\infty} &\leq& 1 - \frac{\eta}{2}.
\end{subeqnarray}
\item $\Z = \Phi X$ is incoherent in the sense of \eqref{eq:incoa}
  and \eqref{eq:incob}:
\begin{subeqnarray}
\norm{\Z_{S^c}^T \Z_S\left(\Z_{S}^T \Z_S\right)^{-1}}_{\infty} 
&\leq& 1 - \eta/2, \\
\label{eq:small-eigen}
\Lambda_{\min}\left(\onem \Z^T_S \Z_S\right) &\geq&\frac{3\eta}{4}.
\end{subeqnarray}
\item
The $\ell_2$ norm of each column is approximately preserved, for all $j$:
\begin{gather}
\abs{\twonorm{\Phi X_j}^2 - m} \leq \frac{m \eta}{4 s}.
\end{gather}
\end{enumerate}
\end{proposition}

Finally, we have the following large deviation result
for the projection matrix $\Phi$, which guarantees
that $R = \Phi \Phi^T - I_{m \times m}$ is small entrywise.


\begin{theorem}
\label{thm:tight-Phi}
If $\Phi$ is $m\times n$ random matrix with independent entries 
$\Phi_{ij} \sim N(0, \inv{n})$, then $R = \Phi \Phi^T - I$ satisfies
\begin{eqnarray}
\P\left(
\left\{\max_{i} |R_{ii}| \geq  \sqrt{{16\log n}/{n}}\right\} 
\cup
\left\{\max_{i\neq j} |R_{ij}| \geq \sqrt{{2\log n}/{n}}\right\} 
\right) \;\leq\; \frac{m^2}{n^3}.
\end{eqnarray}
\end{theorem}

\subsection{Proof of Theorem~\ref{thm:recovery}}
\label{sec:thm-sparsity-proof}
\begin{proofof}{Theorem~\ref{thm:recovery}}
\def\mnote#1{\marginpar{\raggedright\small #1}}
\def\ip#1#2{\langle #1, #2\rangle}
\def\Sc{{S^c}}
\def\onen{{\textstyle \frac{1}{n}}}
\def\onem{{\textstyle \frac{1}{m}}}
\def\Z{Z}
We first state necessary and sufficient conditions on the event
$\event(\sign(\tilde \beta_m) = \sign(\beta^*))$.  Note that this is
essentially equivalent to Lemma~$1$ in \cite{Wai06};  a proof of this
lemma is included in Section~\ref{sec:append-KKT} for completeness.

\begin{lemma}
\label{lemma:KKT}
Assume that the matrix $\Z_S^T \Z_S$ is invertible. Then for any given
$\lambda_m > 0$ and noise vector $\omega \in \R^m$, 
$\event\left(\sign(\tilde \beta_m) = \sign(\beta^*)\right)$ holds 
if and only if the following two conditions hold:
\begin{subeqnarray}
\label{eq:lemma-Sc}
\abs{\Z_{S^c}^T \Z_S (\Z_S^T \Z_S)^{-1} \left[
\onem\Z_S^T \omega - \lambda_m \sign(\beta^*_S)\right] - 
\onem\Z_{S^c}^T \omega} & \leq & \lambda_m, \hspace{1cm} \\
\label{eq:lemma-S}
\sign\left(\beta^*_S + (\onem \Z_S^T \Z_S)^{-1} 
\left[\onem \Z_S^T \omega - \lambda_m \sign(\beta^*_S)\right]
\right)
 & = & \sign(\beta^*_S).
\end{subeqnarray}
\end{lemma}

Let $\vec{b} := \sign(\beta^*_S)$ and $e_i \in \R^s$ be the vector with $1$
in $i^{th}$ position, and zeros elsewhere; hence $\twonorm{e_i} = 1$. 
Our proof of Theorem~\ref{thm:recovery} follows 
that of~\cite{Wai06}. 
We first define a set of random variables that are relevant to 
~(\ref{eq:lemma-Sc}) and~(\ref{eq:lemma-S}):
\begin{subeqnarray}
\forall j \in S^c,  \hspace{1cm}
V_j & := &
\Z_j^T \left\{\Z_S (\Z_S^T \Z_S)^{-1} \lambda_m \vec{b}
+ \left[I_{m \times m} - \Z_S (\Z_S^T \Z_S)^{-1} \Z_S^T \right] \frac{\omega}{m}
\right\}, \\
\forall i \in S, \hspace{1cm} U_i & := & e_i^T + 
\left(\onem \Z_S^T \Z_S\right)^{-1} 
\left[\onem \Z_S^T \omega - \lambda_m \vec{b}\right].
\end{subeqnarray}
We first define a set of random variables that are relevant to 
Condition~(\ref{eq:lemma-Sc}), which holds if and if only the event
\begin{gather}
\event(V) := \left\{\max_{j \in S^c} \abs{V_j} \leq \lambda_m \right\}
\end{gather}
holds. For Condition~(\ref{eq:lemma-S}), the event
\begin{gather}
\event(U) := \left\{\max_{i \in S} \abs{U_i} \leq \rho_m \right\},
\end{gather}
where $\rho_m := \min_{i \in S} |\beta_i^*|$, is sufficient to guarantee that
Condition~(\ref{eq:lemma-S}) holds.

Now, in the proof of Theorem~\ref{thm:recovery}, 
we assume that $\Phi$ has been fixed, and 
$\Z = \Phi X$ and $\Phi\Phi^T$ behave nicely, in accordance
with the results of Section~\ref{sec:concentrate}.
Let $R = \Phi \Phi^T - I_{m \times m}$ as defined in
Theorem~\ref{thm:tight-Phi}. From here on, we use $(\abs{r_{i, j}})$ to
denote a fixed symmetric matrix with diagonal entries that are
$\sqrt{{16 \log n}/{n}}$ and off-diagonal entries that are
$\sqrt{{2 \log n}/{n}}$.

We now prove that $\prob{\event(V)}$ and $\prob{\event(U)}$ both converge to one.
We begin by stating two technical lemmas that will be required.

\begin{lemma}\textnormal{\bf{(Gaussian Comparison)}}
\label{lemma:gaussian-maxima}
For any Gaussian random vector $(X_1, \ldots, X_n)$, 
\begin{equation}
\expct{\max_{1 \leq i \leq n} \abs{X_i}} 
\leq 3 \sqrt{\log n} 
\max_{1 \leq i \leq n} \sqrt{\expct{X_i^2}}.
\end{equation}
\end{lemma}

\begin{lemma}
\label{lemma:convergence}
Suppose that $\norm{\onen X_{S}^T X_S -I_s}_{\infty}$ 
is bounded away from $1$ and 
\begin{equation}
m \geq \left(\frac{16 C_1 s^2}{\eta^2} + \frac{4 C_2
    s}{\eta}\right)(\log p + 2 \log n + \log 2(s+1)).
\end{equation}
Then
\begin{equation}
\inv{\rho_m}\left\{ 
\sqrt{\frac{\log s}{m}}+ 
\lambda_m \norm{(\inv{n} X_S^T X_S)^{-1}}_{\infty} \right \}
\rightarrow 0
\end{equation}
implies that
\begin{equation}
\inv{\rho_m}\left\{ 
\sqrt{\frac{\log s}{m}}+ 
\lambda_m \norm{(\onem \Z_S^T \Z_S)^{-1}}_{\infty} \right \}
\rightarrow 0.
\end{equation}
\end{lemma}

{\bf{Analysis of $\event(V)$}.}\enspace
Note that for each $V_j$, for $j \in S^c$, 
\begin{gather}
\mu_j  = \expct{V_j}
= \lambda_m  \Z_j^T \Z_S (\Z_S^T \Z_S)^{-1} \vec{b}.
\end{gather}
By Proposition~\ref{pro:Phi-X}, we have that 
\begin{gather}
\mu_j \leq
\lambda_m \norm{\Z_{S^c}^T \Z_S\left(\Z_{S}^T \Z_S\right)^{-1}}_{\infty}
\leq (1 - \eta/2) \lambda_m, \forall j \in S^c,
\end{gather}

Let us define 
\begin{gather}
\tilde{V_j} = \Z_j^T \left\{
\left[I_{m \times m} - \Z_S (\Z_S^T \Z_S)^{-1} \Z_S^T \right] \frac{\omega}{m}
\right\}, 
\end{gather}
from which we obtain
\begin{subeqnarray}
\max_{j \in S^c} \abs{V_j}  \leq 
\lambda_m \norm{\Z_{S^c}^T \Z_S\left(\Z_{S}^T \Z_S\right)^{-1}}_{\infty} 
+ \max_{j \in S^c} \abs{\tilde{V_j}} \leq 
\lambda_m (1 - \eta/2) + \max_{j \in S^c}\abs{\tilde{V_j}}.
\end{subeqnarray}
Hence we need to show that 
\begin{gather}
\prob{\frac{\max_{j\in S^c}|\tilde{V}_j|}{\lambda_m} \geq \eta/2} \rightarrow 0.
\end{gather}
It is sufficient to show 
$\prob{\max_{j\in S^c} \abs{\tilde{V}_j} \geq \eta/2} \rightarrow 0$.

By Markov's inequality and the Gaussian comparison
lemma~\ref{lemma:gaussian-maxima}, we obtain that
\begin{eqnarray}
\prob{\max_{j\in S^c} \tilde{V}_j \geq \eta/2}
 \leq 
\frac{\expct{\max_{j\in S^c} \tilde{V}_j}}{\lambda_m(\eta/2)}
 \leq 
\frac{6 \sqrt{\log (p-s)}}{\lambda_m \eta} 
\max_{j\in S^c}\sqrt{\expct{\tilde{V}_j^2}}.
\end{eqnarray}

Finally, let us use $P = \Z_S (\Z_S^T \Z_S)^{-1} \Z_S^T = P^2$ to represent the
projection matrix.
\begin{subeqnarray}
\text{Var}(\tilde{V_j}) 
& = & \expct{\tilde{V_j}^2} \\
& = &
\label{eq:var-V-phi} 
\frac{\sigma^2}{m^2}\Z_j^T\left\{
\left[\left(I_{m \times m} - P \right) \Phi\right]
\left[\left(I_{m \times m} - P \right) \Phi\right]^T \right\}\Z_j \\
& = &  \frac{\sigma^2}{m^2}\Z_j^T \left[I_{m \times m} - P \right] \Z_j
+ 
\frac{\sigma^2}{m^2}\Z_j^T (R - PR - RP + P RP) \Z_j \\
& \leq & \frac{\sigma^2}{m^2} \twonorm{\Z_j}^2 + 
\frac{\sigma^2}{m^2} \twonorm{R - PR - RP + P RP} \twonorm{\Z_j}^2\\
& \leq & 
\left(1 + 4 (m + 2) \sqrt{\frac{2 \log n}{n}}\right)
\frac{\sigma^2 (1 + \frac{\eta}{4s})}{m},
\end{subeqnarray}
where $\twonorm{\Z_j}^2 \leq m + \frac{m \eta}{4 s}$ by 
Proposition~\ref{pro:Phi-X}, and
\begin{subeqnarray}
\nonumber
\lefteqn{
\twonorm{R - PR - RP + PRP} \leq } \\
& & \twonorm{R} + \twonorm{P} \twonorm{R} + 
\twonorm{R} \twonorm{P} + \twonorm{P}\twonorm{R} \twonorm{P} \\
& \leq & 4 \twonorm{R}  \leq  
4 \twonorm{(\abs{r_{i, j}})} \leq 4 (m +2) \sqrt{\frac{2 \log n}{n}},
\end{subeqnarray}
given that $\twonorm{I - P} \leq 1$ and $\twonorm{P} \leq 1$ and
the fact that $(|r_{i, j}|)$ is a symmetric matrix,
\begin{subeqnarray}
\twonorm{R} & \leq &  \twonorm{(\abs{r_{i, j}})} \leq 
\sqrt{\norm{(|r_{i, j}|)}_{\infty} \norm{(|r_{i, j}|)}_{1}} = 
\norm{(|r_{i, j}|)}_{\infty} \\
& \leq &
(m-1) \sqrt{\frac{2 \log n}{n}} + \sqrt{\frac{16 \log n}{n}} 
\leq  (m + 2) \sqrt{\frac{2 \log n}{n}}.
\end{subeqnarray}
Consequently Condition~(\ref{eq:thm-cond-lambda}$a$) is sufficient to ensure
that 
$\frac{\expct{max_{j \in S^c}\abs{\tilde{V}_j}}}{\lambda_m} \rightarrow 0$. 
Thus $\prob{\event(V)} \rightarrow 1$ as $m \rightarrow  \infty$
so long as $m \leq \sqrt{\frac{n}{2 \log n}}$.

{\bf{Analysis of $\event(U)$}.}\enspace
We now show that $\prob{\event(U)} \rightarrow 1$.
Using the triangle inequality, we obtain the upper bound
\begin{gather}
\max_{i \in S} \abs{U_i} \leq 
\norm{\left(\onem \Z_S^T \Z_S\right)^{-1}
\onem \Z_S^T \omega}_{\infty} + 
\norm{\left(\onem \Z_S^T \Z_S\right)^{-1}}_{\infty}\lambda_m.
\end{gather}

The second $\ell_{\infty}$-norm is a fixed value given a 
deterministic $\Phi X$. Hence we focus on the first norm.
We now define, for all $i \in S$, the Gaussian random variable 
\begin{gather}
G_i = e_i^T \left(\onem \Z_S^T \Z_S\right)^{-1} \onem\Z_S^T \omega
= e_i^T \left(\onem \Z_S^T \Z_S\right)^{-1} \onem \Z_S^T \Phi \e.
\end{gather}

Given that $\e \sim N(0, \sigma^2 I_{n \times n})$, we have for all $i \in
S$ that
\begin{subeqnarray}
\expct{G_i} & = & 0, \\
\text{Var}(G_i) & = & \expct{G_i^2} \\
& = & 
\left\{e_i^T \left(\inv{m} \Z_S^T \Z_S\right)^{-1} \onem\Z_S^T \Phi\right\}
\left\{e_i^T \left(\onem \Z_S^T \Z_S\right)^{-1} \onem\Z_S^T \Phi\right\}^T
\text{Var}(\e_i) \\
& = & 
\frac{\sigma^2}{m}e_i^T\left\{
\left(\onem \Z_S^T \Z_S\right)^{-1} \onem\Z_S^T \Phi\Phi^T \Z_S
\left(\onem \Z_S^T \Z_S\right)^{-1} \right\} e_i \\
& = & 
\frac{\sigma^2}{m}e_i^T\left\{
\left(\onem \Z_S^T \Z_S\right)^{-1} \onem\Z_S^T (I + R) \Z_S
\left(\onem \Z_S^T \Z_S\right)^{-1} \right\} e_i \\
& =& 
\label{eq:two-terms}
\frac{\sigma^2}{m} e_i^T \left(\onem \Z_S^T \Z_S\right)^{-1} e_i
+
\frac{\sigma^2}{m}
e_i^T\left\{
\left(\onem \Z_S^T \Z_S\right)^{-1} \onem\Z_S^T R \Z_S
\left(\onem \Z_S^T \Z_S\right)^{-1} \right\} e_i.
\end{subeqnarray}
We first bound the first term
of~(\ref{eq:two-terms}). By \eqref{eq:small-eigen}, we have that
for all $i \in S$, 
\begin{eqnarray}
\label{eq:first-term}
\frac{\sigma^2}{m} e_i^T \left(\inv{m} \Z_S^T \Z_S\right)^{-1} e_i
\leq 
\frac{\sigma^2}{m}
\twonorm{\left(\onem \Z_S^T \Z_S\right)^{-1}} = 
\frac{\sigma^2}{m \Lambda_{\min}\left(\onem \Z^T_S \Z_S\right)} 
\leq \frac{4 \sigma^2}{3 m \eta}.
\end{eqnarray}
We next bound the second term of~(\ref{eq:two-terms}).
Let $M = \frac{C B C}{m}$, where $C =  \left(\onem \Z_S^T \Z_S\right)^{-1}$
and $B = \Z_S^T R \Z_S$. By definition,
\begin{gather}
e_i = [e_{i, 1}, \ldots, e_{i, s}] = [0, \ldots, 1, 0, \ldots],\;
\text{where } e_{i, i} = 1, e_{i, j} = 0, \forall j \not= i.
\end{gather}
Thus, for all $i \in S$,
\begin{eqnarray}
\label{eq:second-term}
e_i^T\left\{
\left(\onem \Z_S^T \Z_S\right)^{-1} \onem\Z_S^T R \Z_S
\left(\onem \Z_S^T \Z_S\right)^{-1} \right\} e_i =
\sum_{j = 1}^s \sum_{k=1}^s e_{i, j} e_{i, k} M_{j, k} = M_{i, i}.
\end{eqnarray}
We next require the following fact.
\begin{claim}
\label{claim:MNorm}
If $m$ satisfies~(\ref{eq:thm-m-bounds}), then for all $i \in S$, we have
$\max_{i} M_{i, i} \leq (1 + \frac{\eta}{4s})\left(\frac{4}{3
    \eta}\right)^2$.
\end{claim}
The proof appears in Section~\ref{sec:append-MNorm}.  
Using Claim~\ref{claim:MNorm}, 
we have by~(\ref{eq:first-term}),~(\ref{eq:second-term}) that
\begin{subeqnarray}
\max_{1 \leq i \leq s} \sqrt{\expct{G_i^2}} 
\leq 
\sqrt{\left(\frac{4\sigma}{3 \eta}\right)^2 \inv{m} \left(\frac{3\eta}{4} + 1 + 
\frac{\eta}{4s}\right)} \leq
\frac{4\sigma}{3 \eta} \sqrt{\inv{m} \left(1 + \frac{3}{4} + \inv{4s}\right)}.
\end{subeqnarray}
By the Gaussian comparison lemma~\ref{lemma:gaussian-maxima}, we have
\begin{subeqnarray}
\expct{\max_{1 \leq i \leq s} \abs{G_i}}
& = &
\expct{\norm{\left(\onem \Z_S^T \Z_S\right)^{-1} 
\onem\Z_S^T \omega}_{\infty}} \\
& \leq & 
3 \sqrt{\log s} \max_{1 \leq i \leq s} \sqrt{\expct{G_i^2}} 
\leq \frac{4 \sigma}{\eta} \sqrt{\frac{2\log s}{m}}.
\end{subeqnarray}

We now apply Markov's inequality to show that
$\prob{\E(U)} \rightarrow 1$ due to Condition~(\ref{eq:thm-cond-lambda}$b$) in the 
Theorem statement and Lemma~\ref{lemma:convergence},
\begin{subeqnarray}
\nonumber
\lefteqn{
1 - \prob{\sign\left(\beta^*_S + (\onem \Z_S^T \Z_S)^{-1} 
\left[\onem\Z_S^T \omega - \lambda_m \sign(\beta^*_S)\right]
\right) = \sign(\beta^*_S)}} \\ 
& \leq & \prob{\max_{i \in S} \abs{U_i} \geq \rho_m} \\
& \leq & 
\prob{\max_{i \in S} \abs{G_i} + 
\lambda_m \norm{\left(\onem \Z_S^T \Z_S\right)^{-1}}_{\infty}
\geq \rho_m} \\
& \leq & \inv{\rho_m} \left(\expct{\max_{i \in S} \abs{G_i}} + 
\lambda_m \norm{\left(\onem \Z_S^T \Z_S\right)^{-1}}_{\infty}\right) \\
& \leq & 
\inv{\rho_m} 
\left(\frac{4 \sigma}{\eta} \sqrt{\frac{2\log s}{m}}  
+ 
\lambda_m \norm{\left(\onem \Z_S^T \Z_S\right)^{-1}}_{\infty}\right) \\
& \rightarrow & 0.
\end{subeqnarray}
which completes the proof.
\end{proofof}

\def\Z{\tilde X}

\section{Compressed Regression is Persistent}
\label{sec:persistence}

Persistence (~\cite{GR04}) is a weaker 
condition than sparsistency.
In particular, we drop the assumption that
$\mathbb{E}(Y|X) = \beta^T X$.
Roughly speaking, persistence implies that
a procedure predicts well.
Let us first review the Greenshtein-Ritov argument; we then
adapt it to the compressed case.

\subsection{Uncompressed Persistence}
Consider a new pair $(X,Y)$
and suppose we want to predict $Y$ from $X$.
The predictive risk 
using predictor $\beta^T X$ is
\begin{equation}
\label{eq:predictive-risk}
R(\beta) = \mathbb{E}(Y - \beta^T X)^2.
\end{equation}
Note that this is a well-defined quantity
even though we do not assume that
$\mathbb{E}(Y|X) = \beta^T X$.
It is convenient to write the risk in
the following way.
Define $Q=(Y,X_1,\ldots, X_p)$
and denote $\gamma$ as
\begin{gather}
\label{eq:gamma}
\gamma = (-1,\beta_1,\ldots,\beta_p)^T = (\beta_0, \beta_1, \ldots, \beta_p)^T.
\end{gather}
Then we can rewrite the risk as
\begin{equation}
R(\beta) = \gamma^T \Sigma \gamma,
\end{equation}
where $\Sigma = \mathbb{E}(Q Q^T)$.
The training error is then
$\hat{R}_n(\beta) = \frac{1}{n}\sum_{i=1}^n (Y_i - X_i^T \beta)^2 =
\gamma^T \hat\Sigma^n \gamma,$
where 
\begin{equation}
\label{eq:hatSigma-define}
\hat\Sigma^n = \frac{1}{n}\mathbb{Q}^T \mathbb{Q}
\end{equation}
and $\mathbb{Q} = (Q^{\dagger}_1 \ Q^{\dagger}_2\ \cdots Q^{\dagger}_n)^T$ where 
$Q^{\dagger}_i = (Y_i, X_{1i},\ldots, X_{pi})^T \sim Q, \forall i =1, \ldots, n$, 
are i.i.d. random vectors. Let 
\begin{gather}
\Ball_n = \{ \beta \; : \; \|\beta\|_1 \leq L_n\},
\;\; \text{for $L_n  = o\left( (n/\log n)^{1/4}\right)$}.
\end{gather}
Let $\beta_*$ minimize $R(\beta)$ subject to $\beta\in \Ball_n$:
\begin{gather}
\label{eq:oracle-unc-estimator}
\beta_* = \argmin_{\|\beta\|_1 \leq L_n} R(\beta).
\end{gather}
Consider the uncompressed lasso estimator $\hat\beta_n$ which minimizes
$\hat{R}_n(\beta)$ subject to
$\beta\in \Ball_n$:
\begin{gather}
\label{eq:lasso-estimator}
\hat\beta_n 
= \argmin_{\|\beta\|_1 \leq L_n} \hat{R}_n(\beta).
\end{gather}

\noindent\textnormal{\bf{Assumption 1.}}\enspace
\label{def:asp1}
Let $Q_j, Q_k$ denote elements of $Q$. Suppose that, for each $j$ and $k$, 
\begin{gather}
\expct{\abs{Z}^q} \leq q! M^{q-2} s/2,
\end{gather}
for every $q \geq 2$ and some constants $M$ and $s$,
where $Z = Q_j Q_k - \mathbb{E}(Q_j Q_k)$.
Then, by Bernstein's inequality,
\begin{equation}
\prob{\left|\hat\Sigma^n_{jk} - \Sigma_{jk}\right|  > \epsilon} \leq e^{-cn\epsilon^2}
\end{equation}
for some $c>0$.
Hence, if $p_n \leq e^{n^\xi}$ for some $0\leq \xi < 1$
then
\begin{gather}
\label{eq:uncomp-bound}
\prob{\max_{j,k}\left|\hat\Sigma^n_{jk} - \Sigma_{jk}\right|  > \epsilon}
\leq p_n^2 e^{-cn\epsilon^2} \leq  e^{-cn\epsilon^2/2}.
\end{gather}
Hence, if $\epsilon_n = \sqrt{\frac{2\log n}{cn}}$, then
\begin{gather}
\prob{\max_{j,k}\left|\hat\Sigma^n_{jk} - \Sigma_{jk}\right|  > \epsilon_n}
\leq \frac{1}{n} \to 0.
\end{gather}
Thus,
\begin{equation}
\max_{j,k}| \hat\Sigma^n_{jk} - \Sigma_{jk}| =  
O_P\left(\sqrt{\frac{\log n}{n}}\right).
\end{equation}
Then,
\begin{gather}
\sup_{\beta\in \Ball_n}|R(\beta) - \hat{R}_n(\beta)| 
=\sup_{\beta\in \Ball_n}|\gamma^T (\Sigma - \hat\Sigma^n)\gamma| \leq  
(L_n + 1)^2 \max_{j,k}| \hat\Sigma^n_{jk} - \Sigma_{jk}|.
\end{gather}
Hence, given a sequence of sets of estimators $\Ball_n$,
\begin{equation}
\sup_{\beta\in \Ball_n}|R(\beta) - \hat{R}_n(\beta)| = o_P(1)
\end{equation}
for $L_n = o( (n/\log n)^{1/4})$.

We claim that under Assumption~$1$, the sequence of uncompressed lasso 
procedures as given in~(\ref{eq:lasso-estimator}) is persistent, i.e., 
$R(\hat\beta_n) - R(\beta_*)  \stackrel{P}{\to} 0.$
By the definition of $\beta_* \in \Ball_n$ and $\hat\beta_n \in \Ball_n$,
we immediately have
$R(\beta_*) \leq R(\hat\beta_n)$ and 
$\hat{R}_n(\hat\beta_n) \leq \hat{R}_n(\beta_*)$;
combining with the following inequalities,
 \begin{eqnarray}
R(\hat\beta_n) - \hat{R}_n(\hat\beta_n)
& \leq & 
\sup_{\beta\in \Ball_n}|R(\beta) - \hat{R}_n(\beta)|, \\
\hat{R}_n(\beta_*)  - R(\beta_*) 
& \leq & 
\sup_{\beta\in \Ball_n}|R(\beta) - \hat{R}_n(\beta)|,
\end{eqnarray}
we thus obtain
\begin{gather}
\abs{R(\hat\beta_n) - R(\beta_*)} \leq 
2 \sup_{\beta\in \Ball_n}|R(\beta) - \hat{R}_n(\beta)|.
\end{gather}

For every $\epsilon > 0$,
the event 
$\left\{\abs{R(\hat\beta_n) - R(\beta_*)} > \epsilon \right\}$
is contained in the event
\begin{equation}
\left\{\sup_{\beta\in \Ball_n}|R(\beta) - \hat{R}_n(\beta)| > \epsilon/2
\right\}.
\end{equation}
Thus, for $L_n = o( (n/\log n)^{1/4})$, and for all $\epsilon > 0$
\begin{equation}
\prob{\abs{R(\hat\beta_n) - R(\beta_*) }
> \epsilon} \leq 
\prob{\sup_{\beta\in \Ball_n}|R(\beta) - \hat{R}_n(\beta)| > \epsilon/2} \to 0,
\text{ as } n \to \infty.
\end{equation}
The claim follows from the definition of persistence.

\subsection{Compressed Persistence}

\label{sec:comp-pers}
Now we turn to the compressed case.
Again we want to predict $(X,Y)$,
but now the estimator $\hat\beta_{n,m}$ is based on the 
lasso from the compressed data of dimension $m_n$; we omit the subscript $n$ 
from $m_n$ wherever we put $\{n, m\}$ together.

Let $\gamma$ be as in~(\ref{eq:gamma}) and 
\begin{equation}
\label{eq::Sigma}
\hat\Sigma^{n,m} =  \frac{1}{m_n}\mathbb{Q}^T \Phi^T \Phi \mathbb{Q}.
\end{equation} Let us replace $\hat{R}_n$ with
\begin{gather}
\label{eq:com-emp-risk}
\hat{R}_{n,m}(\beta) = \gamma^T \hat\Sigma^{n,m} \gamma.
\end{gather}
Given compressed dimension $m_n$, the original design matrix dimension 
$n$ and $p_n$, let 
\begin{gather}
\Ball_{n, m} = \{ \beta\;:\; \|\beta\|_1 \leq L_{n,m}\},
\; \text{for }
L_{n,m} = o\left(\frac{m_n}{\log(np_n)}\right)^{1/4}.
\end{gather}
Let $\beta_*$ minimize $R(\beta)$ subject to $\beta\in \Ball_{n, m}$:
\begin{gather}
\label{eq:comp-beta-star}
\beta_* = \argmin_{\beta\,:\, \|\beta\|_1 \leq L_{n,m}} R(\beta).
\end{gather}
Consider the compressed lasso estimator $\hat\beta_{n,m}$ which minimizes
$\hat{R}_{n, m}(\beta)$ subject to $\beta\in \Ball_{n,m}$:
\begin{gather}
\label{eq:com-lasso-estimator}
\hat\beta_{n,m}
= \argmin_{\beta\,: \, \|\beta\|_1 \leq L_{n,m}} \hat{R}_{n,m}(\beta).
\end{gather}

\noindent\textnormal{\bf{Assumption 2.}}\enspace
\label{def:asp2}
Let $Q_j$ denote the $j^{th}$ element of $Q$. 
There exists a constant $M_1 > 0$ such that 
\begin{gather}
\label{eq::means}
\E(Q_j^2) < M_1,
\; \; \forall j \in \left\{1, \ldots, p_n+1\right\},
\end{gather}

\begin{theorem}
\label{thm:persistence}
Under Assumption~$1$ and~$2$, given a sequence of sets of 
estimators $\Ball_{n,m} \subset \reals^p$ for $\log^2 (n p_n) \leq m_n \leq n$, 
where $\Ball_{n,m}$ consists of all coefficient vectors $\beta$ such that 
$\norm{\beta}_1 \leq L_{n,m} = o\left((m_n/\log(np_n))^{1/4}\right)$, 
the sequence of compressed lasso procedures as 
in~(\ref{eq:com-lasso-estimator}) is persistent:
\begin{equation}
R(\hat\beta_{n,m}) - R(\beta_*)  \stackrel{P}{\to} 0,
\end{equation}
when $p_n = O\left(e^{n^c}\right)$ for some $c<1/2$.
\end{theorem}

\begin{proof}
First note that
\begin{gather}
\expct{\hat\Sigma^{n,m}} = 
\frac{1}{m_n}\expct{\mathbb{Q}^T \expct{\Phi^T \Phi} \mathbb{Q}} = 
\frac{1}{m_n}\expct{\frac{m_n}{n} \mathbb{Q}^T \mathbb{Q}} = \Sigma.
\end{gather}
We have that
\begin{gather}
\label{eq::pers2}
\sup_{\beta\in \Ball_{n,m}}\left|R(\beta) - \hat{R}_{n,m}(\beta)\right| =
\sup_{\beta\in \Ball_{n,m}}\left|\gamma^T (\Sigma - \hat\Sigma^{n,m})\gamma\right|\leq
(L_{n,m} + 1)^2 \ 
\max_{j,k}\left| \hat\Sigma^{n,m}_{jk} - \Sigma_{jk}\right|.
\end{gather}
We claim that, given $p_n = O\left(e^{n^c}\right)$ with $c<1/2$ chosen 
so that $\log^2 (n p_n) \leq m_n \leq n$ holds, then
\begin{equation}
\label{eq::www}
\max_{j,k}\left| \hat\Sigma^{n,m}_{jk} - \Sigma_{jk}\right|
= O_P\left(\sqrt{\frac{\log np_n}{m_n}}\right),
\end{equation}
where $\Sigma = \frac{1}{n}\expct{\mathbb{Q}^T \mathbb{Q}}$ is the same 
as~(\ref{eq:hatSigma-define}), but~(\ref{eq::Sigma}) defines the matrix 
$\hat\Sigma^{n,m}$. 

Hence, given $p_n = O\left(e^{n^c}\right)$ for some $c<1/2$, 
combining~(\ref{eq::pers2}) and~(\ref{eq::www}), we have 
for $L_{n,m} = o\left((m_n/\log(np_n))^{1/4}\right)$ and 
$n \geq m_n \geq \log^2 (n p_n)$,
\begin{equation}
\sup_{\beta\in \Ball_{n,m}}|R(\beta) - \hat{R}_{n,m}(\beta)| = o_P(1).
\end{equation}
By the definition of $\beta_* \in \Ball_{n,m}$ as in~(\ref{eq:comp-beta-star}) 
and $\hat\beta_{n,m} \in \Ball_{n, m}$, we immediately have
\begin{equation}
\abs{R(\hat\beta_{n,m}) - R(\beta_*)} \leq 
2 \sup_{\beta\in \Ball_{n,m}}|R(\beta) - \hat{R}_{n,m}(\beta)|,
\end{equation}
given that
\begin{subeqnarray}
R(\beta_*) \leq R(\hat\beta_{n,m}) 
& \leq & 
\hat{R}_{n,m}(\hat\beta_{n,m}) + 
\sup_{\beta\in \Ball_{n,m}}|R(\beta) - \hat{R}_{n,m}(\beta)| \\
& \leq & 
\hat{R}_{n,m}(\beta_*) + \sup_{\beta\in \Ball_{n,m}}|R(\beta) - \hat{R}_{n,m}(\beta)| \\
& \leq & R(\beta_*) + 2
\sup_{\beta\in \Ball_{n,m}}|R(\beta) - \hat{R}_{n,m}(\beta)|.
\end{subeqnarray}
Thus for every $\epsilon > 0$, the event 
$\left\{\abs{R(\hat\beta_{n,m}) - R(\beta_*)} > \epsilon \right\}$
is contained in the event
\begin{equation}
\left\{\sup_{\beta\in \Ball_{n,m}}\left|R(\beta) - \hat{R}_{n,m}(\beta)\right| > \epsilon/2
\right\}.
\end{equation}
It follows that $\forall \epsilon > 0$, given $p_n = O\left(e^{n^c}\right)$ 
for some $c<1/2$, $n \geq m_n \geq \log^2 (n p_n)$, and
$L_{n,m} = o( (m_n/\log (n p_n))^{1/4})$,
\begin{equation}
\prob{\abs{R(\hat\beta_{n,m}) - R(\beta_*) }
> \epsilon} \leq 
\prob{\sup_{\beta\in \Ball_{n,m}}
|R(\beta) - \hat{R}_{n,m}(\beta)| > \epsilon/2} \to 0,
\text{ as } n \to \infty.
\end{equation}
Therefore, $R(\hat\beta_{n,m}) - R(\beta_*)  \stackrel{P}{\to} 0$. The theorem follows 
from the definition of persistence.

It remains to to show (\ref{eq::www}).
We first show the following claim; note that
$p_n = O\left(e^{n^c}\right)$ with $c<1/2$ clearly satisfies the condition.
\begin{claim}
\label{claim:C1}
Let $C = 2M_1$. Then $\prob{\max_j \twonorm{Q_j}^2 > C n} < \frac{1}{n}$
so long as $p_n \leq \frac{e^{c_1 M_1^2 n}}{n}$
for some chosen constant $c_1$ and $M_1$ satisfying Assumption~$2$,
\end{claim}

\begin{proof}
To see this,
let $A=(A_1,\ldots, A_n)^T$ denote a generic column vector
of $\mathbb{Q}$.
Let $\mu = \mathbb{E}(A_i^2)$.
Under our assumptions,
there exists $c_1 >0$ such that
\begin{equation}
\label{eq::c1}
\mathbb{P}\left(\frac{1}{n}\sum_{i=1}^n V_i > t\right)\leq e^{-n c_1 t^2},
\end{equation}
where $V_i = A_i^2 - \mu$.
We have $C = 2 M_1 \geq \mu + \sqrt{\frac{\log(np_n)}{c_1 n}}$ so long
as $p_n \leq \frac{e^{c_1 M_1^2 n}}{n}$.

Then
\begin{subeqnarray}
\prob{\sum_i A_i^2 > Cn}  & \leq & 
\prob{\sum_i (A_i^2 - \mu) > n \sqrt{\frac{\log(np_n)}{c_1 n}}} \\
& = &
\prob{\frac{1}{n}\sum_{i=1}^n V_i > \sqrt{\frac{\log(np_n)}{c_1 n}}} 
< \frac{1}{np_n}.
\end{subeqnarray}
We have with probability $1 -1/n$, that 
\begin{eqnarray}
\twonorm{Q_j} \leq 2 M_1 n, \; \; \forall j = 1, \ldots, p_n+1. 
\end{eqnarray}
The claim follows by the union bound for $C = 2M_1$.
\end{proof}

Thus we assume that $\twonorm{Q_j}^2 \leq C n$ for all $j$, and
use the triangle inequality to bound 
\begin{gather}
\label{eq::tri}
\max_{jk}|\hat\Sigma^{n,m}_{jk} - \Sigma_{jk}|  
\leq 
\max_{jk}\abs{\hat\Sigma^{n,m}_{jk} -  (\onen \mathbb{Q}^T \mathbb{Q})_{jk}}
+ 
\max_{jk}\abs{\left(\onen \mathbb{Q}^T \mathbb{Q}\right)_{jk} - \Sigma_{jk}},
\end{gather}
where, using $p$ as a shorthand for $p_n$,
\begin{subeqnarray}
\hat\Sigma^{n,m} & = & \inv{m_n}\left[
\begin{array}{cccc} 
\twonorm{\Phi Y}^2 & \ip{\Phi Y}{\Phi X_1} & \ldots & \ip{\Phi Y}{\Phi X_p}\\
\ip{\Phi X_1}{\Phi Y} & \twonorm{\Phi X_1}^2 & \ldots & \ip{\Phi X_1}{\Phi X_p}\\
\ldots &  & & \\
\ip{\Phi X_p}{\Phi Y} & \ip{\Phi X_p}{\Phi X_1} & \ldots & \twonorm{\Phi X_p}^2
\end{array}
\right]_{(p+1) \times (p+1)}, \\
\inv{n} \mathbb{Q}^T \mathbb{Q} & = & \inv{n} \left[
\begin{array}{cccc} 
\twonorm{Y}^2 & \ip{Y}{X_1} & \ldots & \ip{Y}{X_p}\\
\ip{X_1}{Y} & \twonorm{X_1}^2 & \ldots & \ip{X_1}{X_p}\\
\ldots &  & & \\
\ip{X_p}{Y} & \ip{X_p}{X_1} & \ldots & \twonorm{X_p}^2
\end{array}
\right]_{(p+1) \times (p+1)}.
\end{subeqnarray}

We first compare 
each entry of $\hat\Sigma^{n,m}_{jk}$ with that of 
$\inv{n}\left(\mathbb{Q}^T \mathbb{Q}\right)_{j, k}$.

\begin{claim}
\label{claim:C2}
Assume that $\twonorm{Q_j}^2 \leq C n = 2M_1 n, \forall j$. By taking 
$\epsilon = C \sqrt{\frac{8 C_1 \log(np_n)}{m_n}}$,
\begin{gather}
\prob{\max_{j, k}\abs{\inv{m_n}\ip{\Phi Q_j}{\Phi Q_k} -
\inv{n}\ip{Q_j}{Q_k}} \geq \frac{\epsilon}{2}}
\leq \inv{n^2},
\end{gather}
where $C_1 = \frac{4 e}{\sqrt{6\pi}} \approx 2.5044$ as in 
Lemma~\ref{lemma:adapt-RSV} and $C$ is defined in Claim~\ref{claim:C1}. 
\end{claim}

\begin{proof}
Following arguments that appear before~(\ref{eq:ip-bound}), and 
by Lemma~\ref{lemma:adapt-RSV}, it is straight forward to verify:
\begin{eqnarray}
\prob{\abs{\inv{m_n}\ip{\Phi Q_{j}}{\Phi Q_{k}} - \inv{n}\ip{Q_{j}}{Q_{k}}} 
\geq \ve} \leq 2 \exp \left(\frac{- m_n \ve^2}{C_1 C^2 + C_2 C \ve}\right),
\end{eqnarray}
where $C_2 = \sqrt{8e} \approx 7.6885$ as in Lemma~\ref{lemma:adapt-RSV}.
There are at most $\frac{(p_n+1)p_n}{2}$ unique events given that both
matrices are symmetric; the claim follows by the union bound.
\end{proof}

We have by the union bound and~(\ref{eq:uncomp-bound}),~(\ref{eq::tri}),
Claim~\ref{claim:C1}, and Claim~\ref{claim:C2},
\begin{subeqnarray}
\lefteqn{\prob{\max_{jk}|\hat\Sigma^{n,m}_{jk} - \Sigma_{jk}|  >
    \epsilon} \;\leq\;} && \\
&& \prob{\max_{jk}
\abs{\onen\left(\mathbb{Q}^T\mathbb{Q}\right)_{jk} - \Sigma_{jk}}
> \frac{\epsilon}{2}} +
\prob{\max_j \twonorm{Q_j}^2 > C n} \;+\;\\
& & 
\prob{\max_{j, k}\abs{{\textstyle \frac{1}{m_n}} \ip{\Phi Q_j}{\Phi Q_k} 
- \onen\ip{Q_j}{Q_k}} \geq \frac{\epsilon}{2} \ \ \lvert \ \ 
\max_j \twonorm{Q_j}^2 \leq C n } \\
& & \leq 
e^{-cn\epsilon^2/8} + \inv{n} + \inv{n^2}.
\end{subeqnarray}
Hence, given $p_n = O\left(e^{n^c}\right)$ with $c<1/2$, by taking
\begin{gather}
\e = \e_{m, n} = O\left(\sqrt{\frac{\log (np_n)}{m_n}}\right),
\end{gather}
we have
\begin{gather}
\label{eq:prob-bound}
\mathbb{P}\left(\max_{jk}\left|\hat\Sigma^{n,m}_{jk} - \Sigma_{jk}\right|  > \epsilon\right) 
\leq \frac{2}{n} \to 0,
\end{gather}
which completes the proof of the theorem.

\end{proof}

\begin{remark}
The main difference between the sequence of compressed lasso estimators and
the original uncompressed sequence is that $n$ and $m_n$ together define the 
sequence of estimators for the compressed data.
Here $m_n$ is allowed to grow from $\Omega(\log^2 (n p_n))$ to $n$; hence for each
fixed $n$, 
\begin{equation}
\left\{\hat\beta_{n,m} \,,\, \text{$\forall m_n$ such that $\log^2(n p_n) <
  m_n \leq n$} \right\}
\end{equation}
defines a subsequence of estimators.
In Section~\ref{sec:experiments} we run simulations that 
compare the empirical risk to the oracle risk on such a subsequence
for a fixed $n$, to illustrate the compressed lasso persistency property.
\end{remark}

\section{Information Theoretic Analysis of Privacy}
\label{sec:privacy}

\long\def\ignore#1{}
\def\ip#1#2{\left\langle #1 ,#2\right\rangle}
\def\sbullet{{\mbox{\tiny\textbullet}}}
\def\S{{\mathcal S}}
\def\T{{\mathcal T}}
\def\pZ{\overline{Z}}
\def\ds{\displaystyle}
\def\given{\,|\,}
\def\Z{\tilde X}

In this section we derive bounds on the rate at which the compressed
data $\Z$ reveal information about the uncompressed data $X$.  Our
general approach is to consider the mapping $X \mapsto \Phi X +
\Delta$ as a noisy communication channel, where the channel is
characterized by multiplicative noise $\Phi$ and additive noise
$\Delta$.  Since the number of symbols in $X$ is $np$ we normalize by
this effective block length to define the information rate $r_{n,m}$
per symbol as
\begin{eqnarray}
r_{n, m} = \sup_{p(X)} \frac{I(X; \Z)}{np}.
\end{eqnarray}
Thus, we seek bounds on the capacity of this channel, where several
independent blocks are coded.  A privacy guarantee is given in terms of
bounds on the rate $r_{n,m} \rightarrow 0$ decaying to zero.
Intuitively, if $I(X;\Z) = H(X) - H(X\given \Z) \approx 0$, then the
compressed data $\Z$ reveal, on average, no more information about the original
data $X$ than could be obtained from an independent sample.

Our analysis yields the rate bound $r_{n, m} = O(m/n)$.  Under the lower bounds
on $m$ in our sparsistency and persistence analyses, this leads to the
information rates
\begin{eqnarray}
r_{n,m} = O\left(\frac{\log(np)}{n}\right)\;\; \text{(sparsistency)}\qquad
r_{n,m} = O\left(\frac{\log^2(np)}{n}\right)\;\; \text{(persistence)}\;\;
\end{eqnarray}
It is important to note, however that these bounds
may not be the best possible since they are
obtained assuming knowledge of the compression matrix $\Phi$,
when in fact the privacy protocol requires that $\Phi$ and $\Delta$ are
not public.  Thus, it may be possible to show
a faster rate of convergence to zero. 
We make this simplification since the capacity of
the underlying communication channel does not have a closed form,
and appears difficult to analyze in general.  Conditioning
on $\Phi$ yields the familiar Gaussian channel in the case of nonzero
additive noise $\Delta$.

In the following subsection we first consider the case where additive
noise $\Delta$ is allowed; this is equivalent to a multiple antenna
model in a Rayleigh flat fading environment.  While our sparsistency
and persistence analysis has only considered $\Delta=0$, additive
noise is expected to give greater privacy guarantees.  Thus, extending
our regression analysis to this case is an important direction for
future work.  In Section~\ref{sec:zerodelta} we consider the case
where $\Delta=0$ with a direct analysis.  This special case does not
follow from analysis of the multiple antenna model.

\subsection{Privacy Under the Multiple Antenna Channel Model}
\label{sec:delta}

In the multiple antenna model for wireless communication
\citep{Marzetta:99,Telatar:99}, there are $n$ transmitter and $m$ receiver
antennas in a Raleigh flat-fading environment.  The propagation
coefficients between pairs of transmitter and receiver antennas are
modeled by the matrix entries $\Phi_{ij}$; they remain constant for a
coherence interval of $p$ time periods.  Computing the channel
capacity over multiple intervals requires optimization of the joint
density of $pn$ transmitted signals.  \cite{Marzetta:99} prove that
the capacity for $n > p$ is equal to the capacity for $n=p$, and is
achieved when $X$ factors as a product of a $p\times p$ isotropically
distributed unitary matrix and a $p\times n$ random matrix that is
diagonal, with nonnegative entries.  They also show that as $p$ gets
large, the capacity approaches the capacity obtained as if the matrix
of propagation coefficients $\Phi$ were known.  Intuitively, this
is because the transmitter could send several ``training'' messages used
to estimate $\Phi$, and then send the remaining information based
on this estimate.

More formally, the channel is modeled as
\begin{equation}
Z =  \Phi X + \gamma \Delta
\end{equation}
where $\gamma > 0$, $\Delta_{ij}\sim N(0,1)$, $\Phi_{ij} \sim N(0,1/n)$ and $\frac{1}{n}
\sum_{i=1}^n \E[X_{ij}^2] \leq P$, where the latter is a power constraint.
The compressed data are then conditionally Gaussian, with
\begin{eqnarray}
\E(Z\given X) &=& 0 \\
\E(Z_{ij} Z_{kl} \given X) &=& \delta_{ik}\left( \gamma^2 \delta_{jl} +
\sum_{t=1}^n X_{tj} X_{tl}\right).
\end{eqnarray}
Thus the conditional density $p(Z\given X)$ is given by
\begin{equation}
p(Z\given X) = \frac{\exp\left\{-\text{tr}\left[\left(\gamma^2 I_p + X^T
      X\right)^{-1} Z^T Z\right]\right\}}{(2\pi)^{pm/2}
  \det^{m/2}(\gamma^2 I_p + X^T X)}
\end{equation}
which completely determines the channel.  Note that this distribution
does not depend on $\Phi$, and the transmitted signal affects only the
variance of the received signal.

The channel capacity is difficult to compute or accurately bound in full
generality.  However, an upper bound is obtained by assuming that the
multiplicative coefficients $\Phi$ are known to the receiver.  
In this case, we have that $p(Z,\Phi\given X) = p(\Phi) \,p(Z\given
\Phi, X)$, and the mutual information $I(Z,\Phi; X)$ is given by
\begin{subeqnarray}
I(Z,\Phi; X) &=& \E\left[ \log \frac{p(Z,\Phi\given X)}{p(Z,\Phi)}\right] \\
  &=& \E\left[\log \frac{p(Z\given X,\Phi)}{p(Z\given \Phi)}\right] \\
  &=& \E\left[ \left. \E\left[\log \frac{p(Z\given X,\Phi)}{p(Z\given \Phi)}
      \right| \Phi\right]\right].
\end{subeqnarray}
Now, conditioned on $\Phi$, the compressed data $Z = \Phi X + \gamma
\Delta$ can be viewed as the output of a standard additive noise
Gaussian channel.  We thus obtain the upper bound
\begin{subeqnarray}
\sup_{p(X)} I(Z;X) &\leq& \sup_{p(X)} I(Z,\Phi; X) \\
       &=& \E \left[ \sup_{p(X)} \left. \E\left[\log \frac{p(Z\given X,\Phi)}{p(Z\given \Phi)}
      \right| \Phi\right]\right] \\
\label{eq:bounda}
       &\leq& p \E\left[\log\det\left(I_m + \frac{P}{\gamma^2} \Phi\Phi^T\right)\right] \\
\label{eq:boundb}
     &\leq& pm \log \left(1 + \frac{P}{\gamma^2}\right)
\end{subeqnarray}
where inequality \eqref{eq:bounda} comes from assuming the $p$ columns of
$X$ are independent, and inequality \eqref{eq:boundb} uses Jensen's
inequality and concavity of $\log\det S$.   
Summarizing, we've shown the following result.

\begin{proposition}
\label{prop:priva}
Suppose that $E[X_j^2] \leq P$ and the compressed data are formed by 
\begin{equation}
Z =  \Phi X + \gamma \Delta
\end{equation}
where $\Phi$ is $m\times n$ with independent entries $\Phi_{ij} \sim N(0,1/n)$ and 
$\Delta$ is $m \times p$ with independent entries $\Delta_{ij} \sim
N(0,1)$.  Then the information rate $r_{n,m}$  satisfies
\begin{eqnarray}
r_{n,m}  \;=\; \sup_{p(X)} \frac{I(X;Z)}{np} \;\leq\; \frac{m}{n} \log\left(1 + \frac{P}{\gamma^2}\right).
\end{eqnarray}
\end{proposition}

\subsection{Privacy Under Multiplicative Noise}
\label{sec:zerodelta}

When $\Delta=0$, or equivalently $\gamma=0$, the above analysis yields
the trivial bound $r_{n,m} \leq \infty$.  
Here we derive a separate bound for this case;
the resulting asymptotic order of the information rate is the same, however.

Consider first the case where $p=1$, so that there is a single column
$X$ in the data matrix.  
The entries are independently sampled as $X_{i} \sim F$ where $F$ has mean
zero and bounded variance $\text{Var}(F) \leq P$.  Let
$Z = \Phi X \in \reals^m$.  An upper bound on the mutual information
$I(X;Z)$ again comes from assuming the compression matrix $\Phi$ is
known.  In this case
\begin{eqnarray}
\label{eq:entropy}
I(Z,\Phi;X)  &=& H(Z\given \Phi) - H(Z\given X, \Phi) \\
       &=& H(Z\given \Phi)
\end{eqnarray}
where the second conditional entropy in \eqref{eq:entropy} is zero since $Z = \Phi X$.
Now, the conditional variance of $Z = (Z_1,\ldots, Z_m)^T$ satisfies
\begin{eqnarray}
\text{Var}(Z_i \given \Phi) \;=\; \sum_{j=1}^n \Phi_{ij}^2 \text{Var}{X_j} 
  \;\leq\; P \sum_{j=1}^n \Phi_{ij}^2 
\end{eqnarray}
Therefore, 
\begin{subeqnarray}
I(Z,\Phi; X) &=& H(Z\given \Phi) \\
\label{eq:ineqa}
    &\leq& \sum_{i=1}^m H(Z_i \given \Phi) \\
\label{eq:ineqb}
    &\leq& \sum_{i=1}^m \E\left[\frac{1}{2} \log \left(2\pi e P
        \sum_{j=1}^n \Phi_{ij}^2\right)\right] \\
\label{eq:ineqc}
    &\leq& \sum_{i=1}^m \frac{1}{2} \log \left(2\pi e P \sum_{j=1}^n
      \E(\Phi_{ij}^2)\right) \\
    &=& \frac{m}{2} \log \left(2\pi e P\right)
\end{subeqnarray}
where inequality \eqref{eq:ineqa} follows from the chain rule
and the fact that conditioning reduces entropy, 
inequality \eqref{eq:ineqb} is achieved by taking $F = N(0,P)$, a
Gaussian, and inequality \eqref{eq:ineqc} uses concavity of $\log \det S$.
In the case where there are $p$ columns of $X$, taking each column
to be independently sampled from a Gaussian with variance $P$
gives the upper bound 
\begin{eqnarray}
I(Z,\Phi;X) 
    &\leq& \frac{m p}{2} \log \left(2\pi e P\right).
\end{eqnarray}
Summarizing, we have the following result.

\begin{proposition}
\label{prop:privb}
Suppose that $E[X_j^2] \leq P$ and the compressed data are formed by 
\begin{equation}
Z =  \Phi X 
\end{equation}
where $\Phi$ is $m\times n$ with independent entries $\Phi_{ij} \sim N(0,1/n)$.
Then the information rate $r_{n,m}$  satisfies
\begin{eqnarray}
r_{n,m}  \;=\; \sup_{p(X)} \frac{I(X;Z)}{np} \;\leq\; \frac{m}{2n} \log \left(2\pi e P\right).
\end{eqnarray}
\end{proposition}


\ignore{
Let
$$
\underbrace{Z}_{m\times p} = \underbrace{\Phi}_{m\times n} \underbrace{X}_{n\times p} 
$$
where $\Phi_{ij}\sim N(0,1/n)$.
Let $Z_i$ denote the $i^{\rm th}$ row $Z$.
The rows are iid with
$$
Z_1,\ldots, Z_m \sim N\left( \left(\begin{array}{c}0\\ \vdots\\ 0\end{array}\right),\Sigma_n\right)
$$
where
$$
\Sigma_n = \frac{1}{n} X^T X .
$$

Let ${\cal X}$ be the space of all $n\times p$ matrices.
An observer starts with a prior $\pi$ on ${\cal X}$
which we take to be any smooth distribution.

Let $\pi_m (\cdot) = \pi(\cdot|Z_1,\ldots, Z_m)$.
Let
$d(\pi,\pi_m)$ measure the expected information increase in observing
$Z_1,\ldots, Z_m$.
We make the reasonable assumption that
$d(\pi,\pi_m)$ is nondecreasing in $m$.
Hence we have the bound
$d(\pi,\pi_m)\leq d(\pi,\pi_\infty)$.

When $m=\infty$,
$Z$ provides perfect information about 
$\Sigma$.
Hence,
$\pi_m \rightsquigarrow Q_\Sigma$
where $Q_\Sigma$ is the restriction of $\pi$ onto
the set
$$
{\cal X}_{\Sigma} = \Biggl\{ X :\ \frac{1}{n} X^T X = \Sigma \Biggr\}.
$$
Thus $Q_\Sigma(A) = \pi(A|X\in {\cal X}_\Sigma)$.

As a special case,
suppose that $\pi = N(0,\Lambda)$.
Then $\pi_\infty$ is
uniform on the skin of the ellipse
${\cal X}_\Sigma$.
The variances are
The variance of $X_i$ under $\pi$ is $\Lambda$.
The variance under $\pi_\infty$ is
$p$.
}


\section{Experiments}
\label{sec:experiments}

In this section we report the results of simulations designed to
validate the theoretical analysis presented in the previous sections.
We first present results that indicate the compressed lasso is
comparable to the uncompressed lasso in recovering
the sparsity pattern of the true linear model, in accordance
with the analysis in Section~\ref{sec:sparsistence}. We then
present experimental results on persistence that are in close agreement
with the theoretical results of Section~\ref{sec:persistence}.

\subsection{Sparsistency}

Here we run simulations to compare the compressed lasso with the 
uncompressed lasso in terms of the probability of success in
recovering the
sparsity pattern of $\beta^*$.  
We use random matrices for both $X$ and $\Phi$, and reproduce the
experimental conditions shown in~\cite{Wai06}.  A design parameter
is the \textit{compression factor}
\begin{equation}
f = \frac{n}{m}
\end{equation}
which indicates how much the original data are compressed.
The results show that
when the compression factor $f$ is large enough, the thresholding behaviors 
as specified in~(\ref{eq:wain-succ-bound}) and~(\ref{eq:wain-fail-bound}) for the 
uncompressed lasso carry over to the compressed lasso,
when $X$ is drawn from a Gaussian ensemble.
In general, the compression factor $f$ is well below the requirement 
that we have in Theorem~\ref{thm:recovery} in case $X$ is deterministic.

In more detail, we consider the Gaussian ensemble for the projection
matrix $\Phi$, where $\Phi_{i,j} \sim N(0, 1/n)$ are independent.  The
noise vector is always composed of i.i.d. Gaussian random variables
$\epsilon \sim N(0, \sigma^2)$, where $\sigma^2 =1$.  We consider
Gaussian ensembles for the design matrix $X$ with both diagonal and
Toeplitz covariance.  In the Toeplitz case, the covariance is given by
\begin{eqnarray}
T(\rho) & = & \left[
\begin{array}{cccccc} 
1 & \rho & \rho^2 & \ldots & \rho^{p-1} & \rho^{p-1} \\
\rho& 1 & \rho & \rho^2 & \ldots & \rho^{p-2} \\
\rho^2  & \rho & 1 & \rho & \ldots & \rho^{p-3} \\
\ldots & \ldots & \ldots & \ldots & \ldots & \ldots\\
\rho^{p-1} & \ldots & \rho^3 & \rho^2 & \rho & 1
\end{array}
\right]_{p \times p}.
\end{eqnarray}
We use $\rho = 0.1$. Both $I$ and $T(0.1)$ satisfy
conditions~(\ref{eq:incoa-covariance}),~(\ref{eq:incob-covariance})
and~(\ref{eq:wain-max-eigen})~\citep{ZY07}.  For $\Sigma = I$,
$\theta_u = \theta_{\ell} = 1$, while for $\Sigma = T(0.1)$, $\theta_u
\approx 1.84$ and $\theta_{\ell} \approx 0.46$~\citep{Wai06}, 
for the uncompressed lasso in~(\ref{eq:wain-succ-bound})
and in~(\ref{eq:wain-fail-bound}).

\def\lars{\texttt{lars}}
\def\lasso{\texttt{lasso2}}

In the following simulations, we carry out the lasso
using procedure $\lars(Y, X)$ that implements the LARS 
algorithm of~\cite{EHJ04} to calculate the full regularization path;
the parameter $\lambda$ is then selected along this path to match
the appropriate condition specified by the analysis.
For the uncompressed case, we run $\lars(Y, X)$ such that
\begin{eqnarray}
Y & = & X \beta^* + \epsilon,
\end{eqnarray}
and for the compressed case we run $\lars(\Phi Y, \Phi X)$ such that
\begin{eqnarray}
\Phi Y & = & \Phi X \beta^* + \Phi \epsilon.
\end{eqnarray}
In each individual plot shown below, the covariance $\Sigma =
\inv{n}\expct{X^TX}$ and model $\beta^*$ are fixed across all curves in the plot.
For each curve, a compression factor $f \in \{5, 10, 20, 40, 80,
120\}$ is chosen for the compressed lasso, and we show the probability
of success for recovering the signs of $\beta^*$ as the number of
compressed observations $m$ increases, where $m = 2 \theta \sigma^2 s
\log(p-s) + s + 1$ for $\theta \in [0.1, u]$, for $u \geq 3$.  Thus, the number of
compressed observations is $m$, and the number
of uncompressed observations is $n = f m$.
Each point on a curve, for a particular $\theta$ or $m$, 
is an average over $200$ trials; for each trial, we randomly draw
$X_{n \times p}$, $\Phi_{m \times n}$, and $\e \in \R^n$.
However $\beta^*$ remains the same for all $200$ trials, and is
in fact fixed across different sets of experiments for the same sparsity level.

We consider two sparsity regimes:
\begin{subeqnarray}
\label{eq:sublinear}
\text{\it Sublinear sparsity:} && s(p) = \frac{\alpha p}{\log(\alpha
  p)} \; \text{for $\alpha \in \{0.1, 0.2, 0.4 \}$} \\
\text{\it Fractional power sparsity:} && s(p) = \alpha p^{\gamma} \;
  \text{for $\alpha = 0.2$ and $\gamma = 0.5$}.
\end{subeqnarray}

The coefficient vector $\beta^*$ is selected to be a prefix of a fixed
vector  
\begin{equation}
\beta^\star = (-0.9, -1.7, 1.1, 1.3, 0.9, 2, -1.7, -1.3, -0.9, -1.5,
1.3, -0.9, 1.3, 1.1, 0.9)^T
\end{equation}
That is, if $s$ is the number of nonzero
coefficients, then
\begin{equation}
\beta^*_i = \begin{cases}
\beta^\star_i & \text{if $i\leq s$,} \\
0 & \text{otherwise}.
\end{cases}
\end{equation}
As an exception, for the case $s=2$, 
we set $\beta^* = (0.9, -1.7, 0, \ldots, 0)^T$.

\def\cP{{\mathcal P}}
After each trial, $\lars(Y, X)$ outputs a ``regularization path,''
which is a set of estimated models $\cP_m = \{\beta\}$ such that 
each $\beta \in \cP_m$ is associated with a corresponding 
regularization parameter $\lambda(\beta)$, which is computed as 
\begin{gather}
\label{eq:lambda-from-beta}
\lambda(\beta) = \frac{\twonorm{Y - X \tilde\beta}^2}{m\norm{\tilde\beta}_1}.
\end{gather}
The coefficient vector $\tilde\beta\in\cP_m$ for which $\lambda(\tilde\beta)$ is closest to the
value $\lambda_m$
is then evaluated for sign consistency, where
\begin{equation}
\lambda_m = c \sqrt{\frac{\log(p-s) \log s}{m}}.
\end{equation}
If $\sign(\tilde\beta) = \sign(\beta^*)$, the trial
is considered a success, otherwise, it is a failure.
We allow the constant $c$ that scales $\lambda_m$ to change with the
experimental configuration (covariance $\Sigma$, compression
factor $f$, dimension $p$ and sparsity $s$), 
but $c$ is a fixed constant across all $m$ along the same curve.

Table~1 summarizes the parameter settings that
the simulations evaluate.  In this table
the ratio $m/p$ is for $m$ evaluated at $\theta =1$.  
The plots in Figures~1--4 show the empirical probability of
the event $\event(\sign(\tilde \beta) = \sign(\beta^*))$ for each of these
settings, which is a lower bound for that of the event
$\{\supp(\tilde \beta) = \supp(\beta^*)\}$. 
The figures clearly demonstrate that the compressed lasso recovers the true 
sparsity pattern as well as the uncompressed lasso.

\begin{table*}
\label{table:sparsity-para}
\begin{center}
\begin{tabular}{|c|c||c|c|c|c|c|c|c|c|} 
\hline
 & $\alpha$ 
 & \multicolumn{2}{c|}{$p=128$} 
 & \multicolumn{2}{c|}{$p=256$} 
 & \multicolumn{2}{c|}{$p=512$}
 & \multicolumn{2}{c|}{$p=1024$} \\ 
\cline{3-10}
& & $s(p)$ & $m/p$  
& $s(p)$ & $m/p$
& $s(p)$ & $m/p$
& $s(p)$ & $m/p$ \\
\hline\hline
Fractional Power & $0.2$ & $2$ & $0.24$  & $3$ & $0.20$ & $5$ & $0.19$ & $6$ & $0.12$ \\ \hline
Sublinear & $0.1$ & $3$ & $0.36$ & $5$ & $0.33$ & $9$ & $0.34$ &  & \\
   & $0.2$ & $5$ & $0.59$ & $9$ & $0.60$    & $15$ & $0.56$ & &   \\
   & $0.4$ & $9$ & $1.05$ & $15$ & $1.00$  &  &  & & \\ \hline
\end{tabular}
\caption{Simulation parameters: $s(p)$ and ratio of $m/p$ for $\theta =1$ and $\sigma^2 = 1$.}
\end{center}
\end{table*}

\clearpage

\def\sleft{\hskip-5pt}
\def\lleft{\hskip-25pt}
\begin{figure*}
\begin{center}
\begin{tabular}{cc}
\begin{tabular}{c}
\lleft
\includegraphics[width=.50\textwidth,angle=-90]{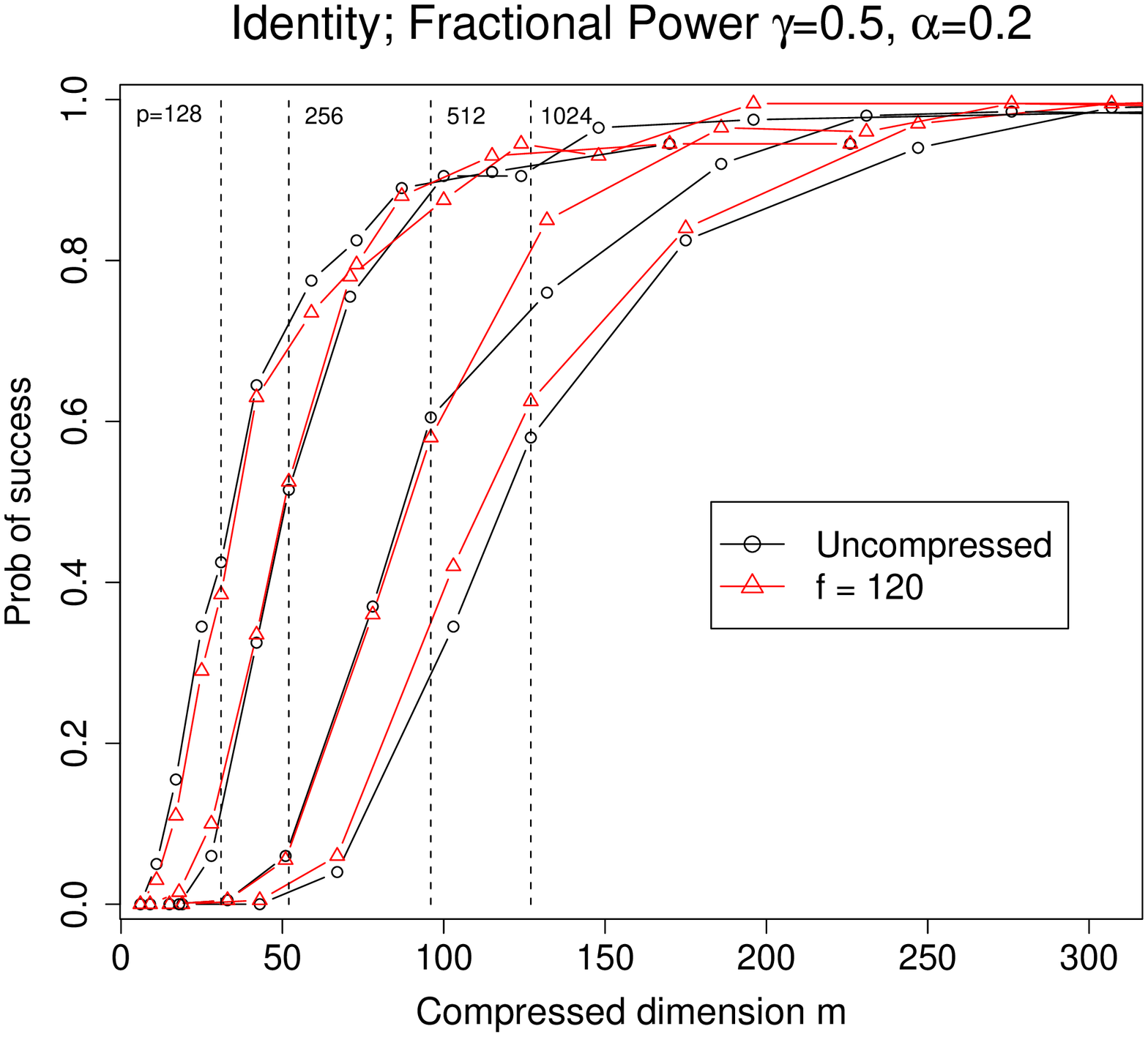} \\
\lleft
\includegraphics[width=.50\textwidth,angle=-90]{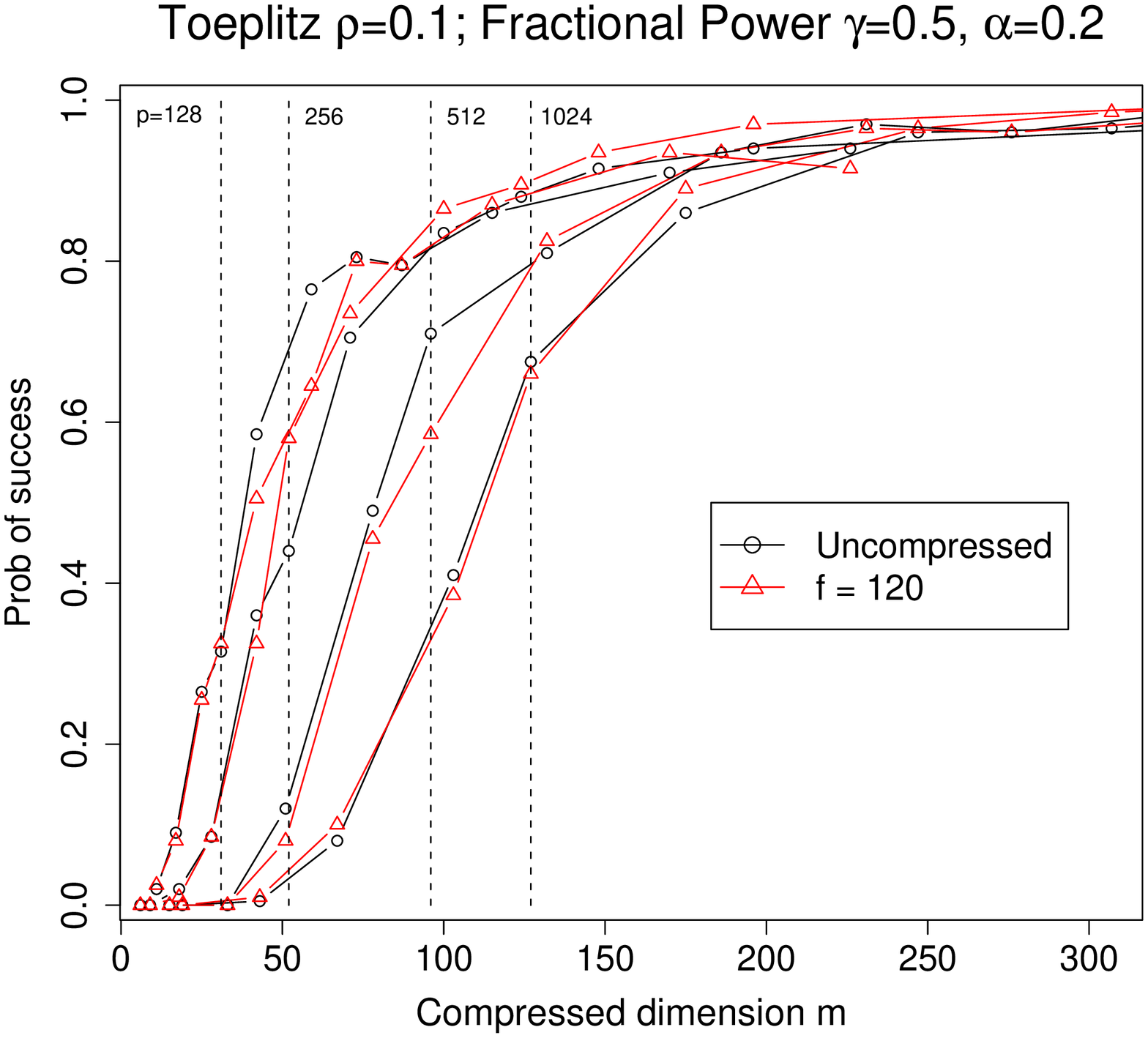} 
\end{tabular}&
\begin{tabular}{c}
\sleft
\includegraphics[width=.25\textwidth,angle=-90]{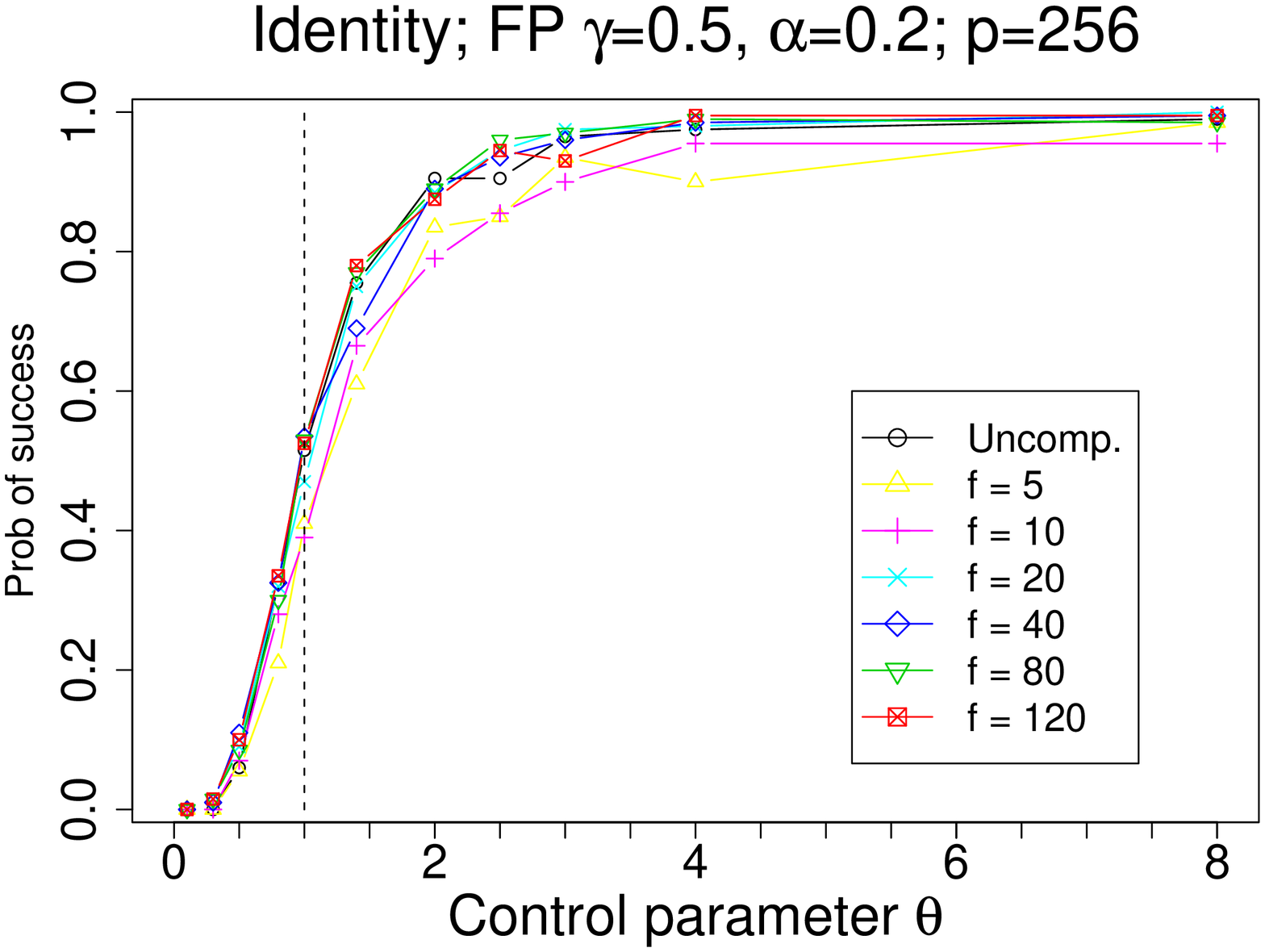} \\
\sleft
\includegraphics[width=.25\textwidth,angle=-90]{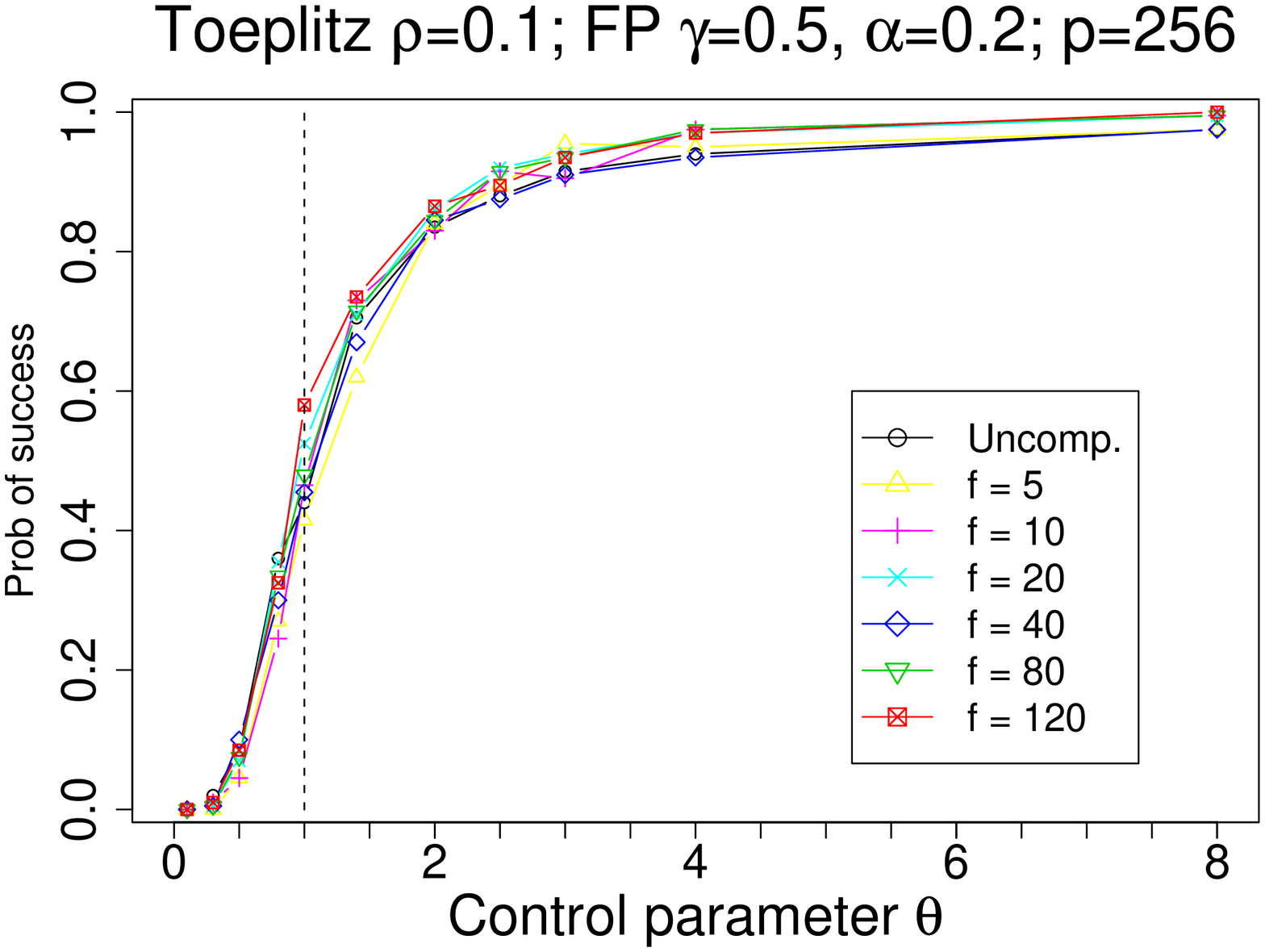}\\
\sleft
\includegraphics[width=.25\textwidth,angle=-90]{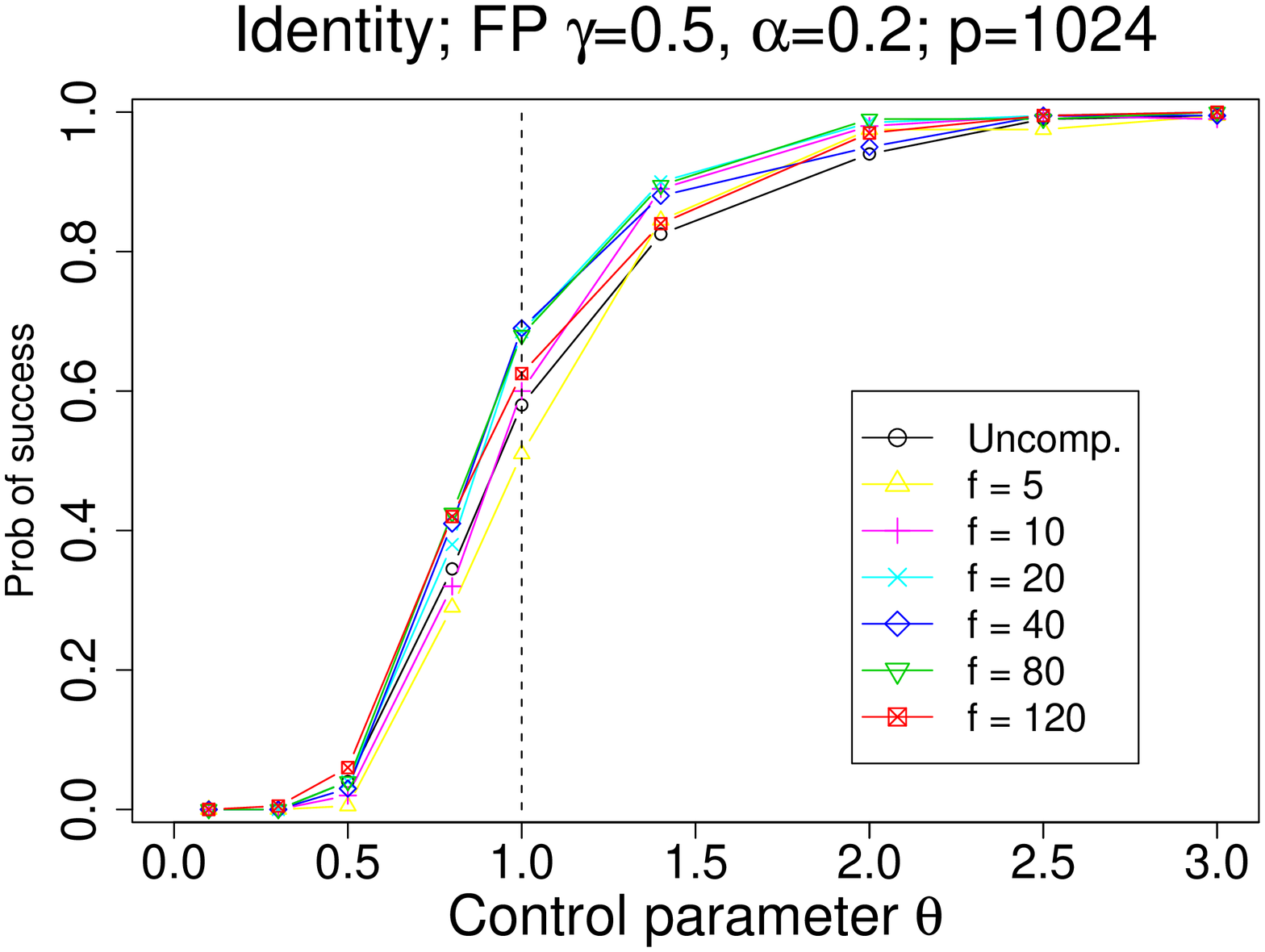} \\
\sleft
\includegraphics[width=.25\textwidth,angle=-90]{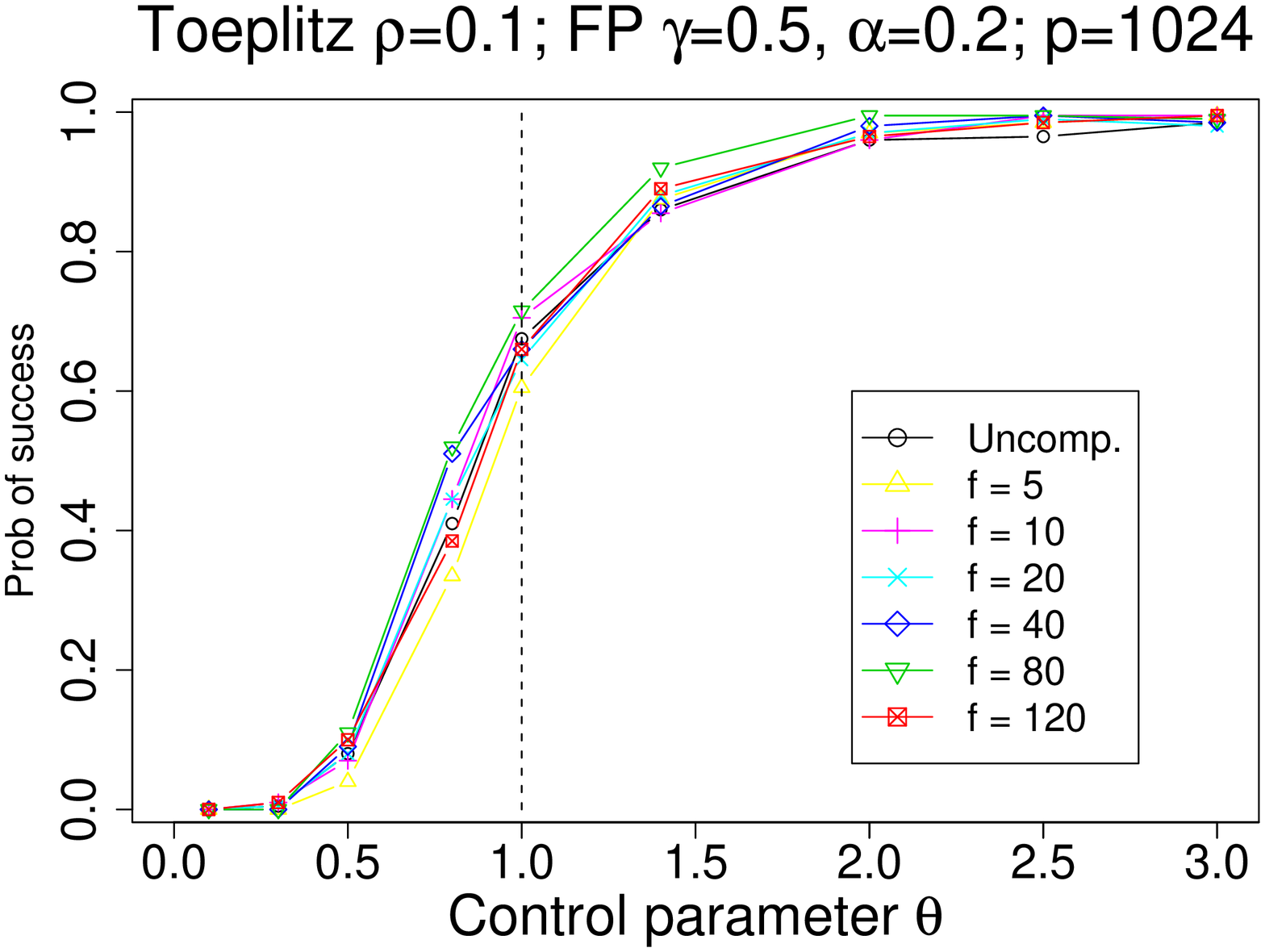}
\end{tabular}
\end{tabular}
\caption{Plots of the number of samples versus the probability of success.
The four sets of curves on the left panel map to
$p=128, 256, 512$ and $1024$, with dashed lines marking 
$m = 2 \theta s \log(p-s) + s + 1$ for $\theta = 1$ and $s =2, 3, 5$ and $6$ 
respectively.  For clarity, the left plots only show the uncompressed
lasso and the compressed lasso with $f = 120$.}
\end{center}
\label{fig:plots_a}
\end{figure*}

\begin{figure*}
\begin{center}
\begin{tabular}{cc}
\begin{tabular}{c}
\lleft
\includegraphics[width=.50\textwidth,angle=-90]{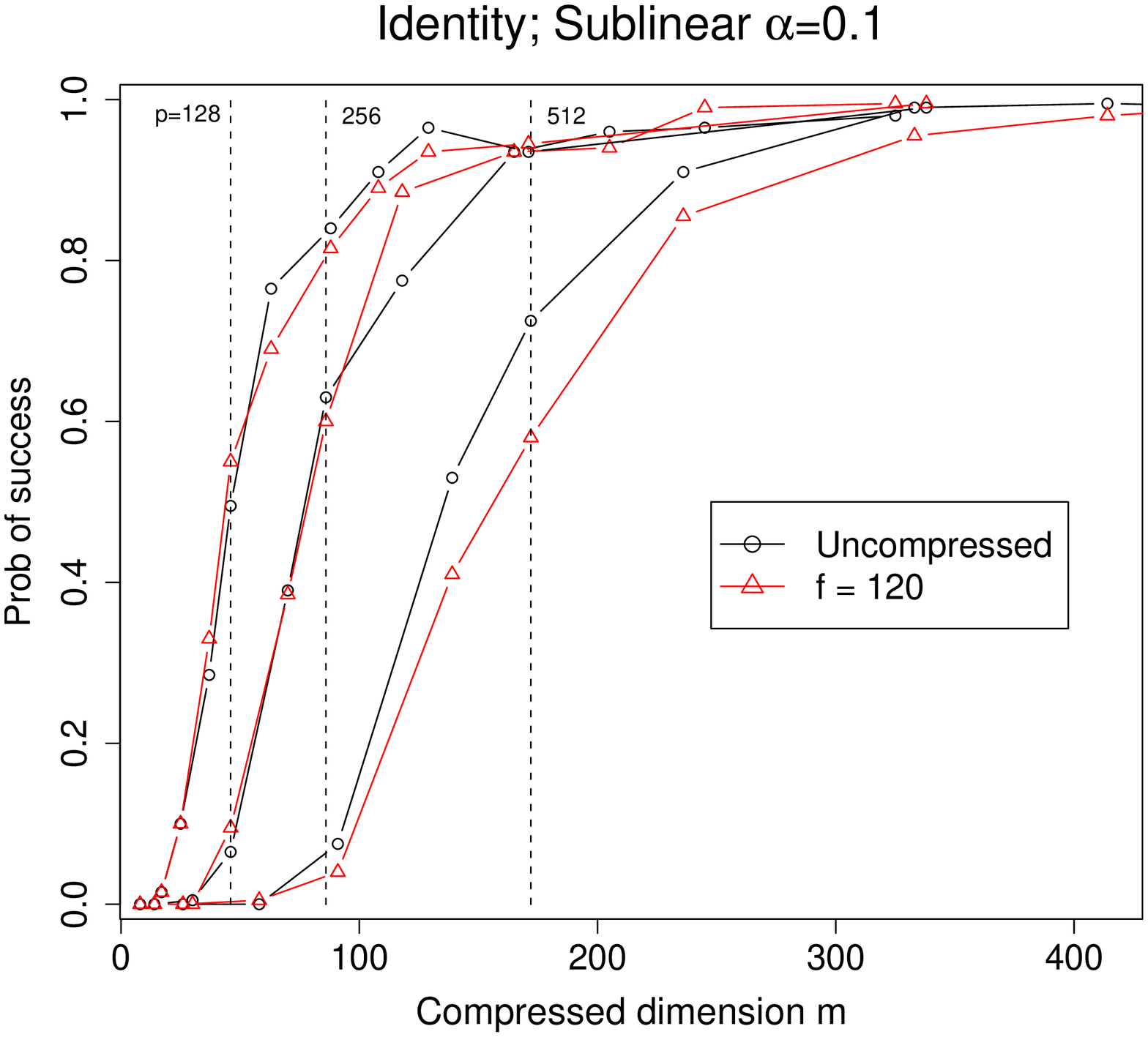}\\ 
\lleft
\includegraphics[width=.50\textwidth,angle=-90]{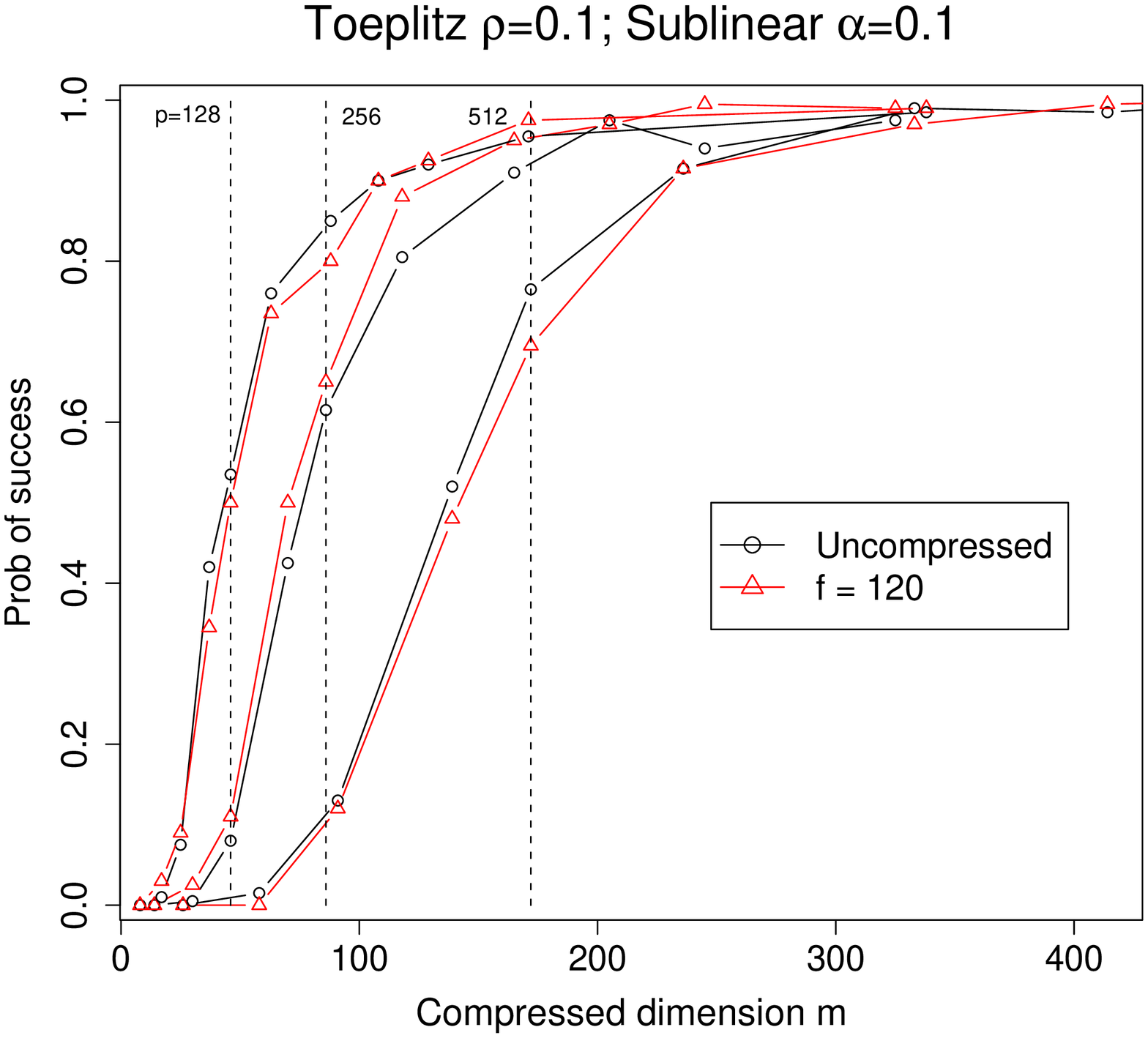}
\end{tabular} &
\begin{tabular}{c}
\sleft
\includegraphics[width=.25\textwidth,angle=-90]{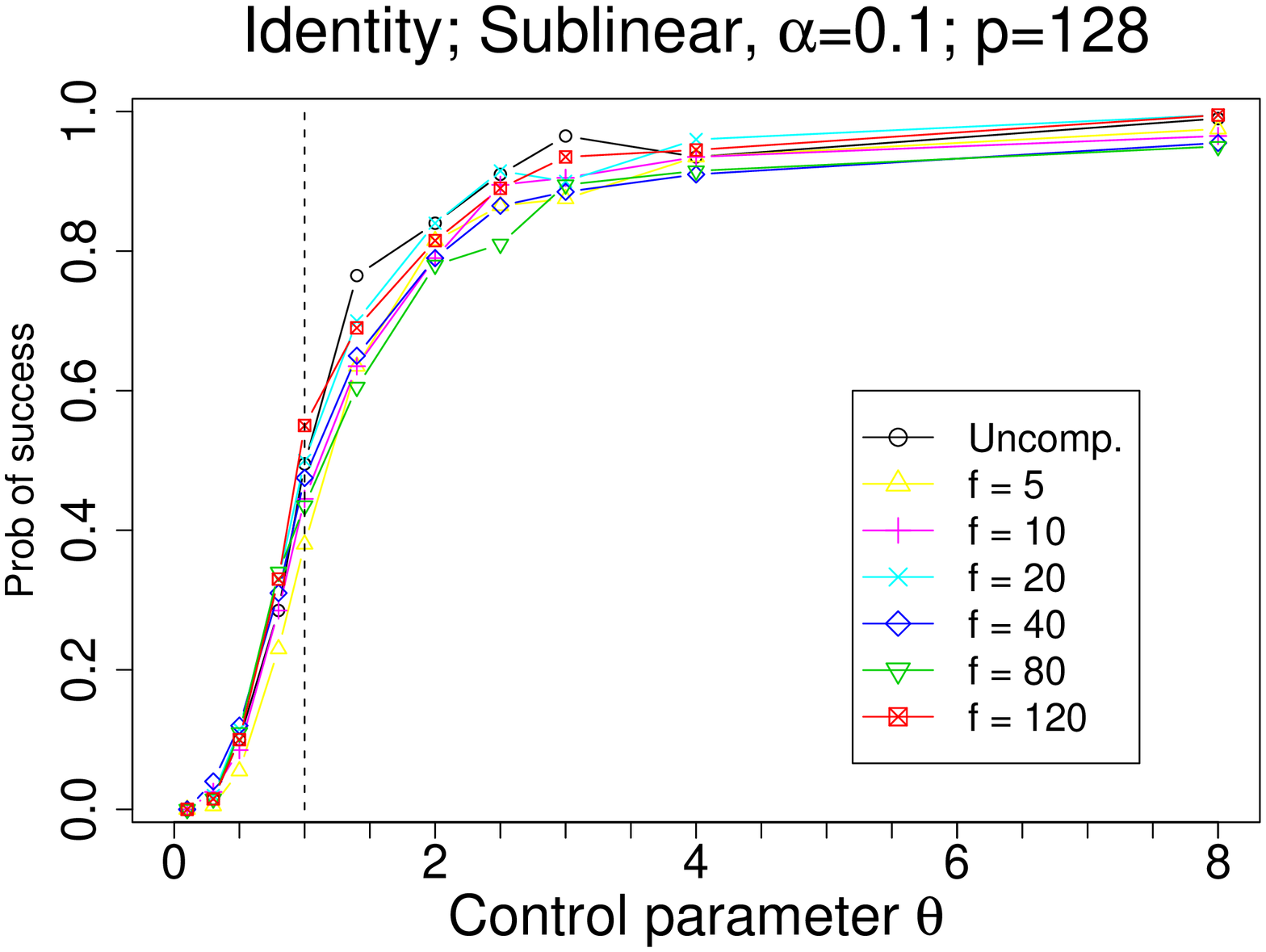}\\
\sleft
\includegraphics[width=.25\textwidth,angle=-90]{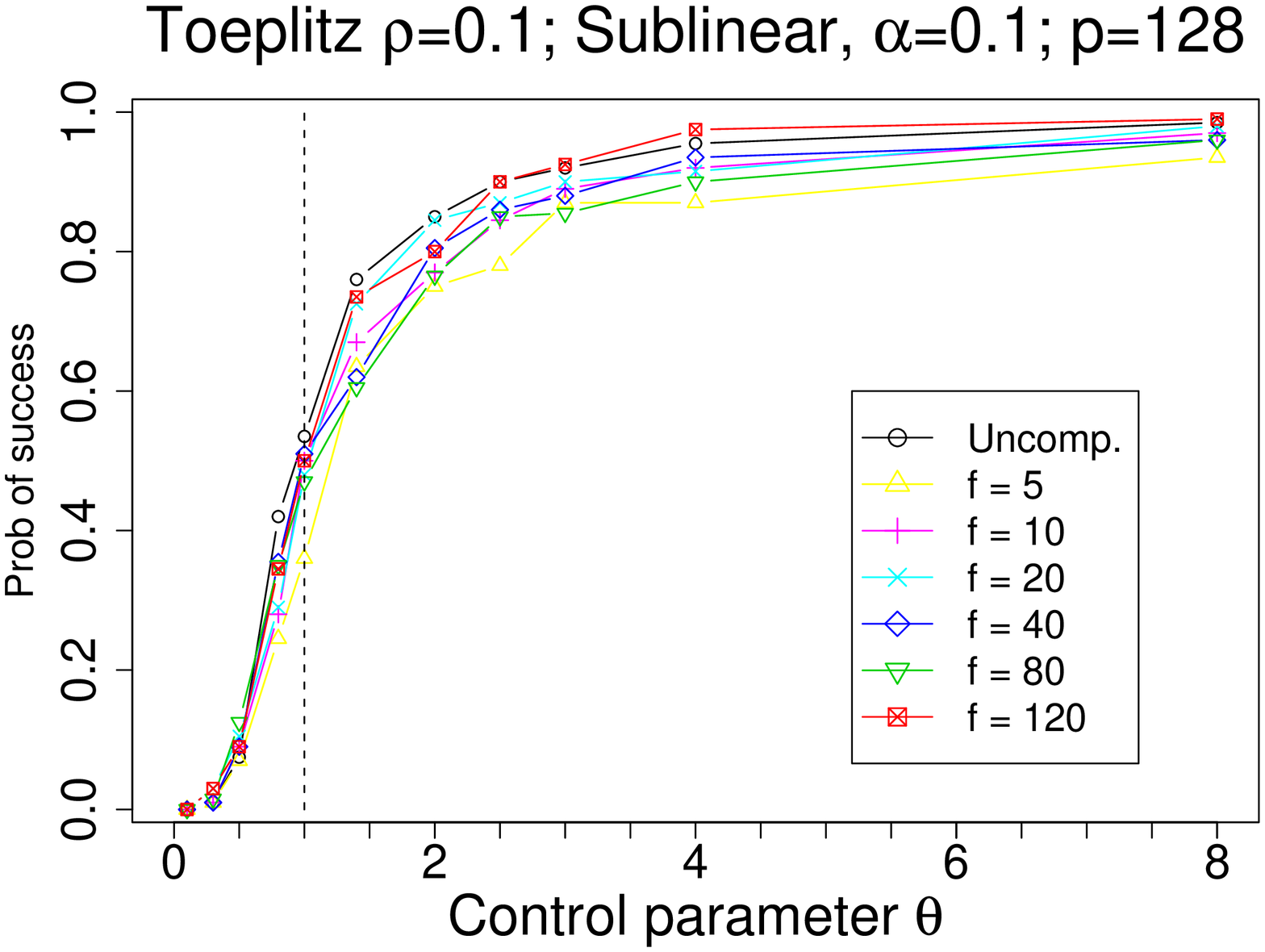}\\
\sleft
\includegraphics[width=.25\textwidth,angle=-90]{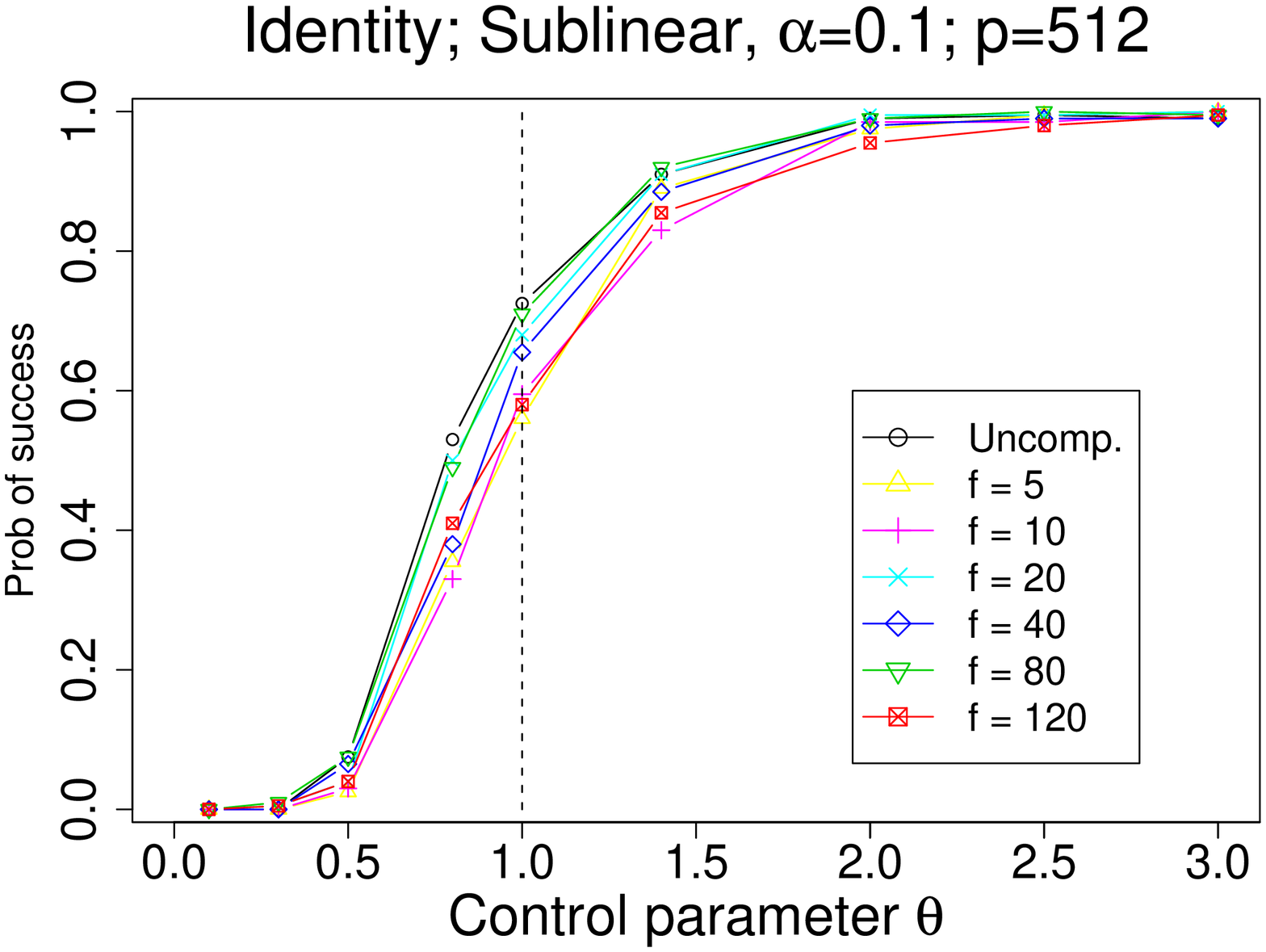}\\
\sleft
\includegraphics[width=.25\textwidth,angle=-90]{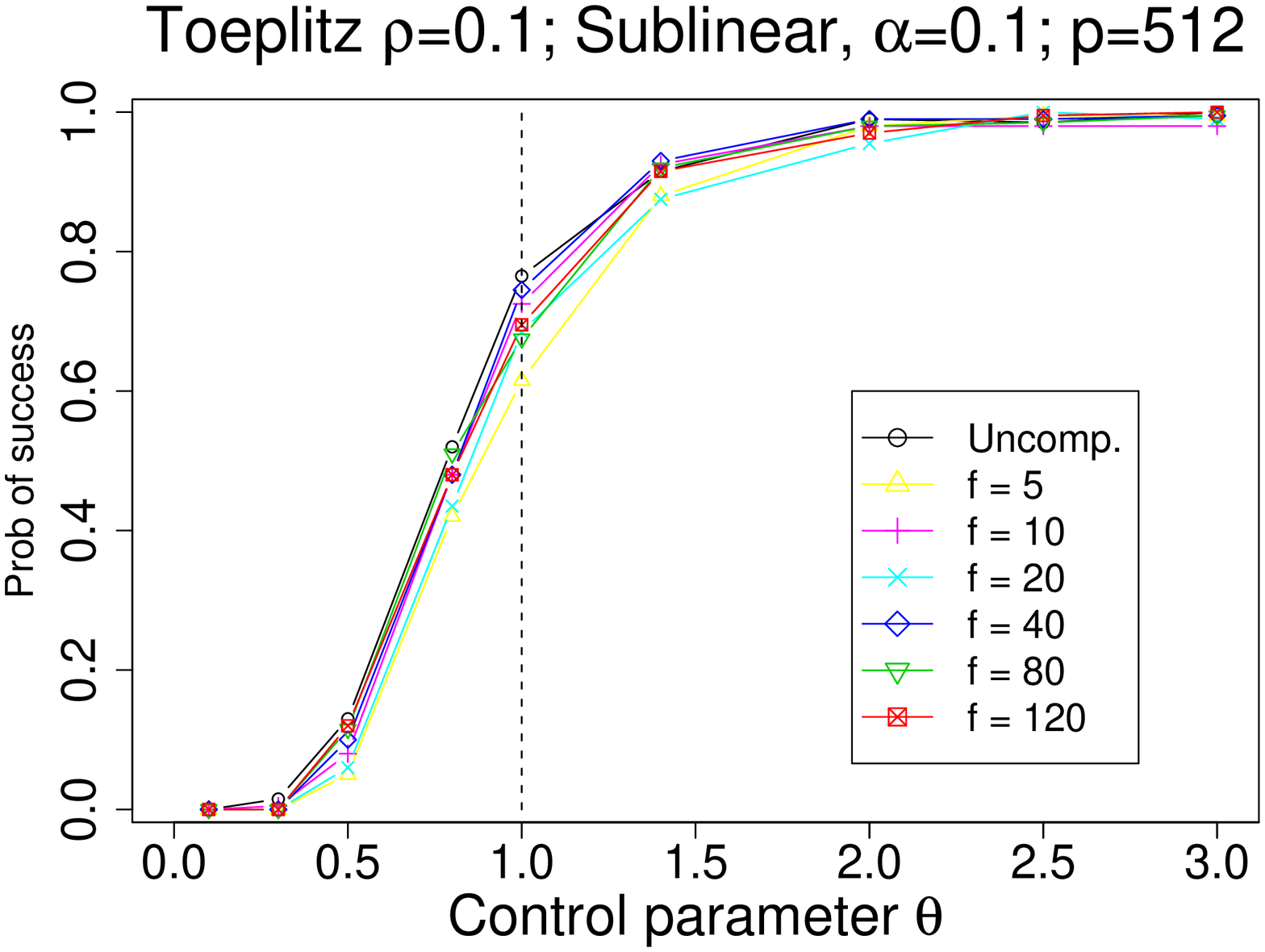}
\end{tabular}
\end{tabular}
\caption{Plots of the number of samples versus the probability of success.
The three sets of curves on the left panel map to 
$p=128, 256$ and $512$ with dashed lines marking 
$m = 2 \theta s \log(p-s) + s + 1$ for $\theta = 1$ and $s =3, 5$ and $9$
respectively.}
\end{center}
\label{fig:plots_b}
\end{figure*}

\begin{figure*}\begin{center}
\begin{tabular}{cc}
\begin{tabular}{c}
\lleft
\includegraphics[width=.50\textwidth,angle=-90]{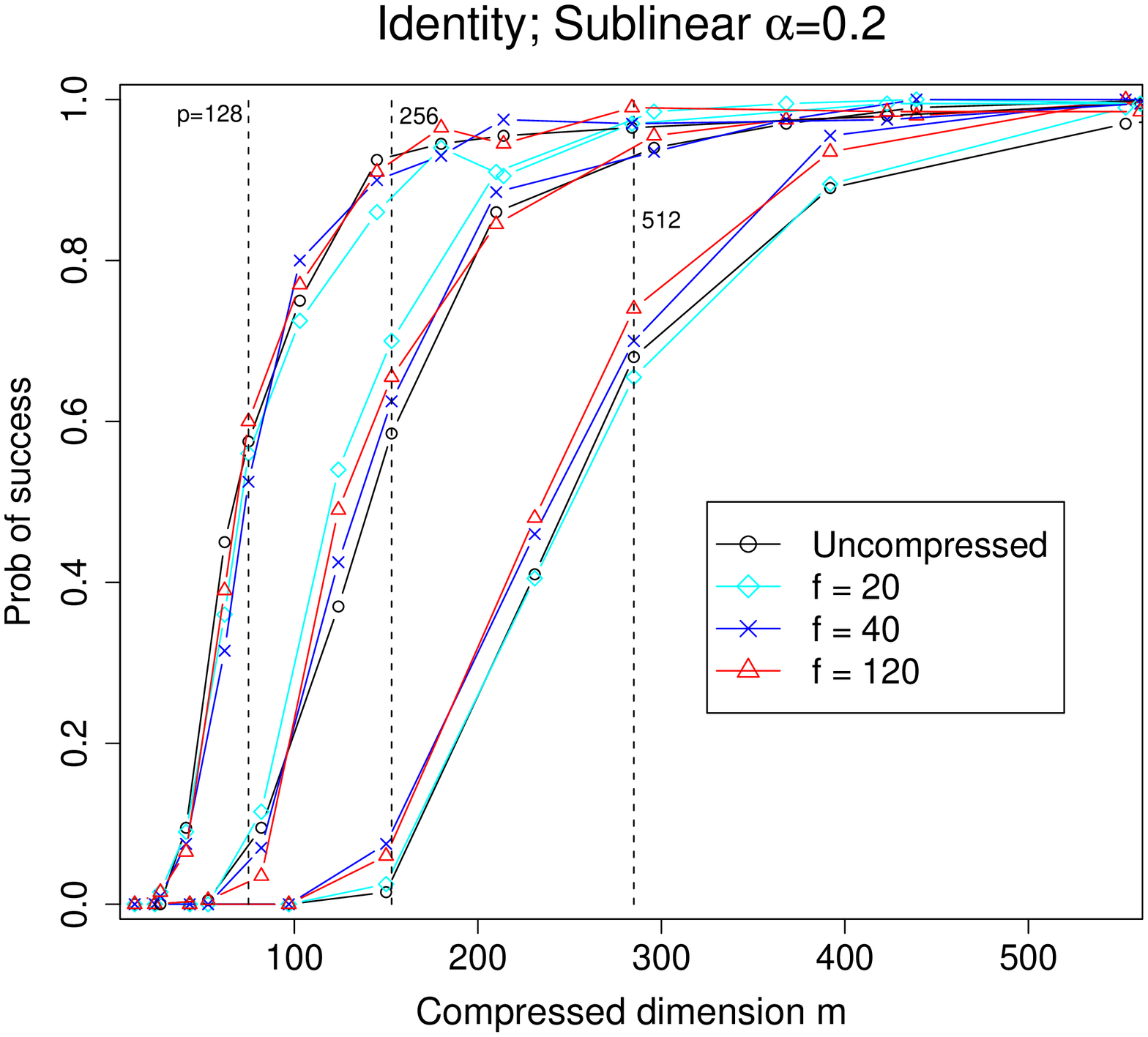}
\\
\lleft
\includegraphics[width=0.50\textwidth,angle=-90]{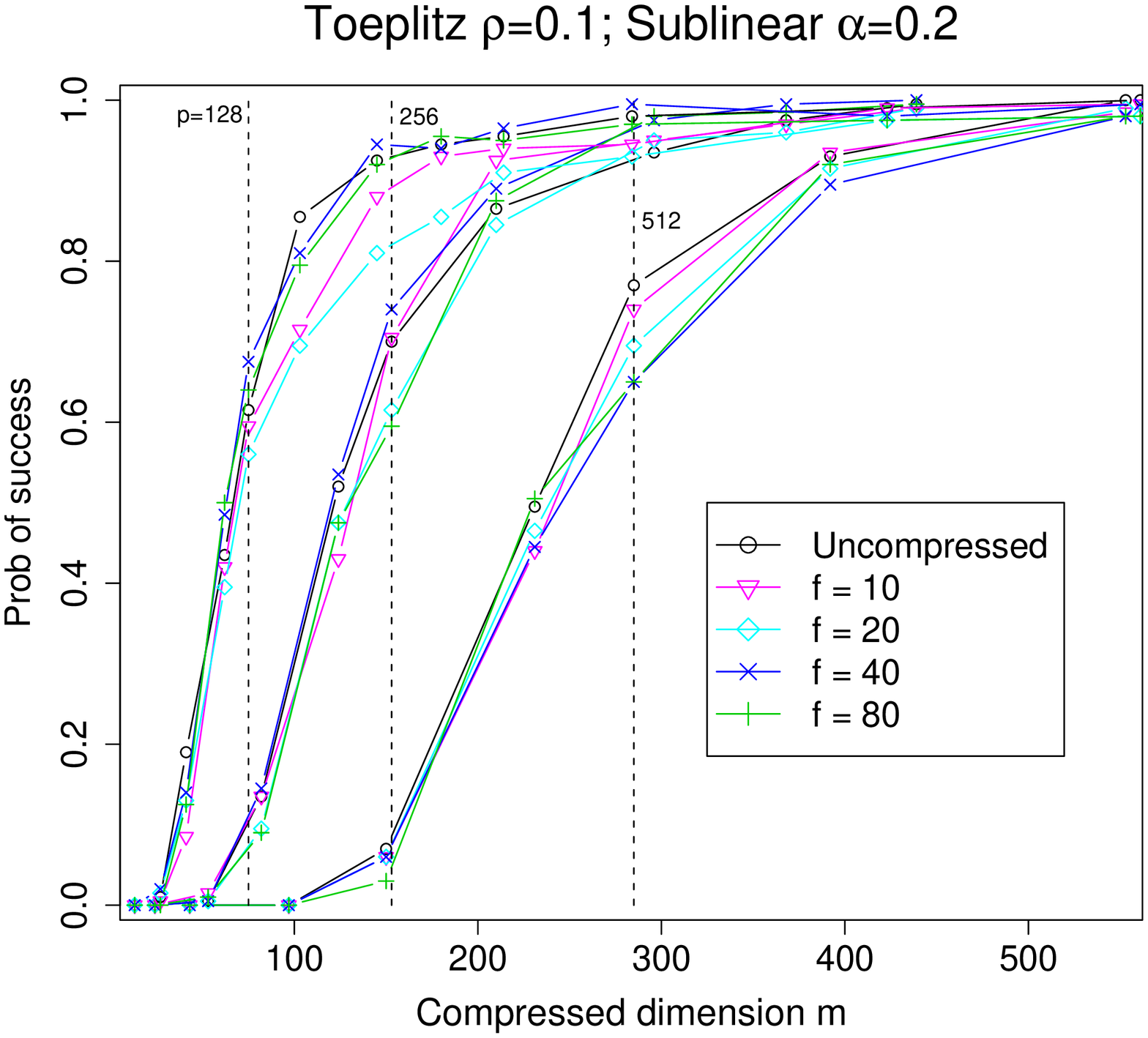} 
\end{tabular}&
\begin{tabular}{c}
\sleft
\includegraphics[width=.25\textwidth,angle=-90]{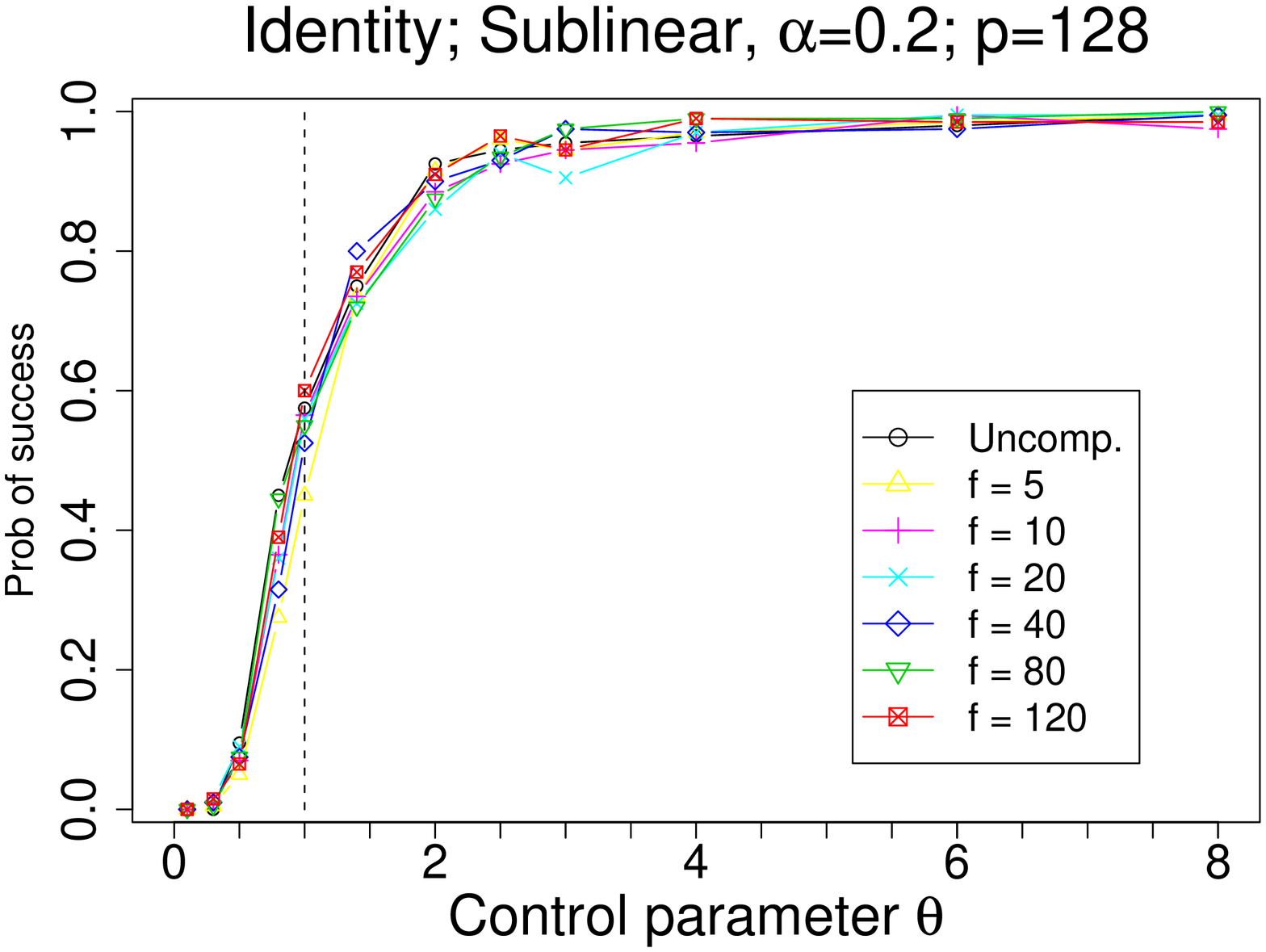}\\
\sleft
\includegraphics[width=.25\textwidth,angle=-90]{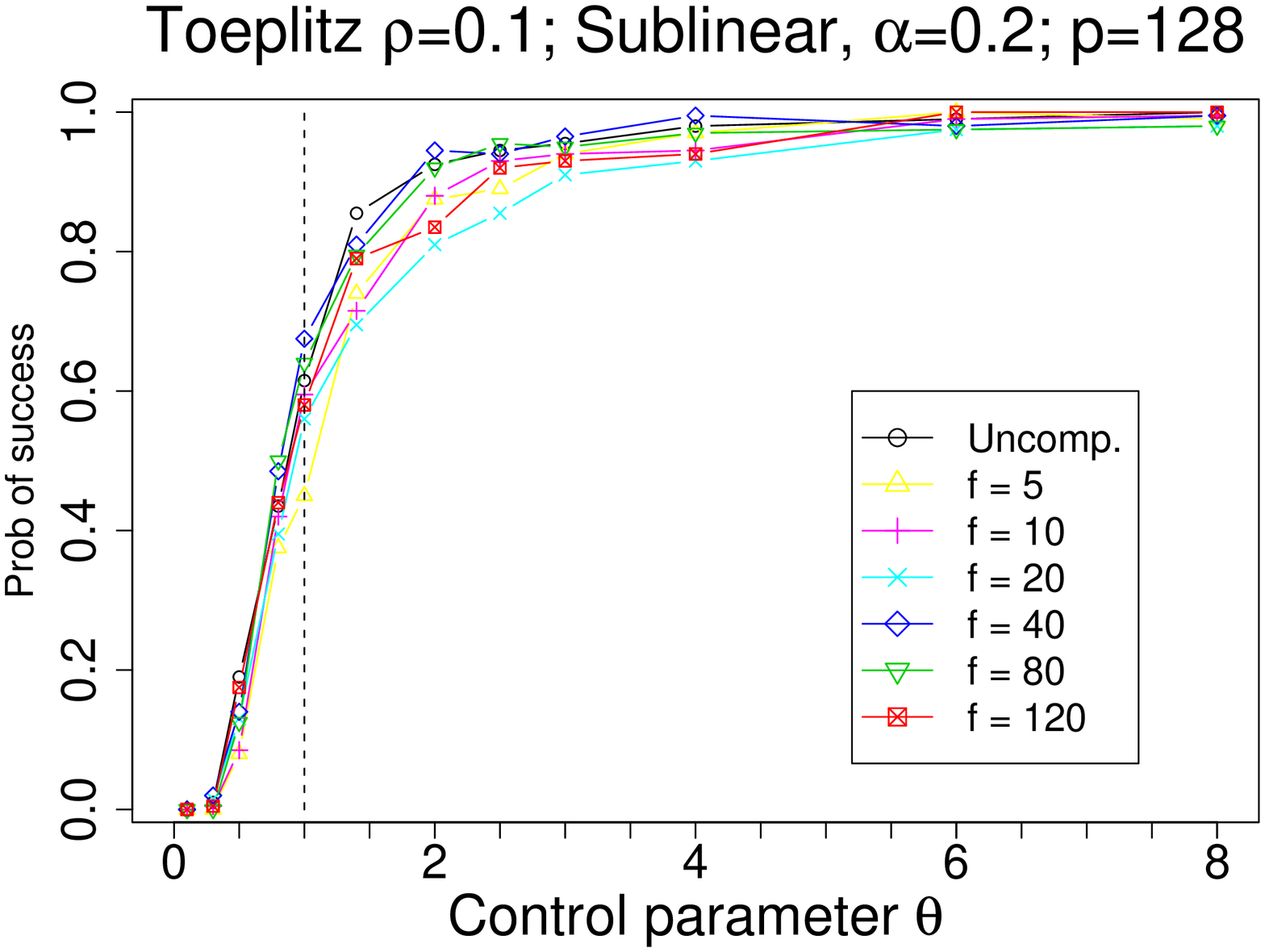} \\
\sleft
\includegraphics[width=.25\textwidth,angle=-90]{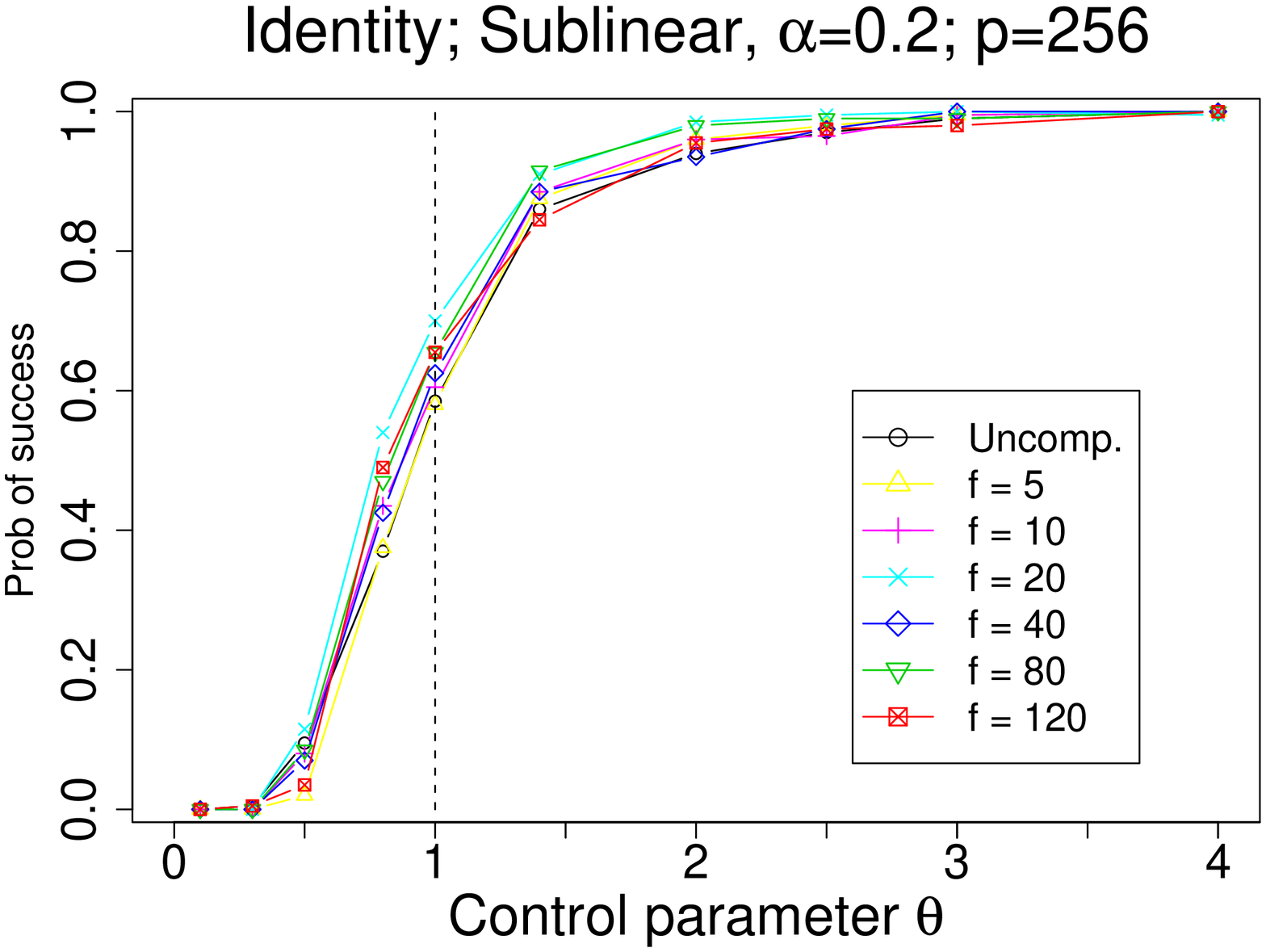} \\
\sleft
\includegraphics[width=.25\textwidth,angle=-90]{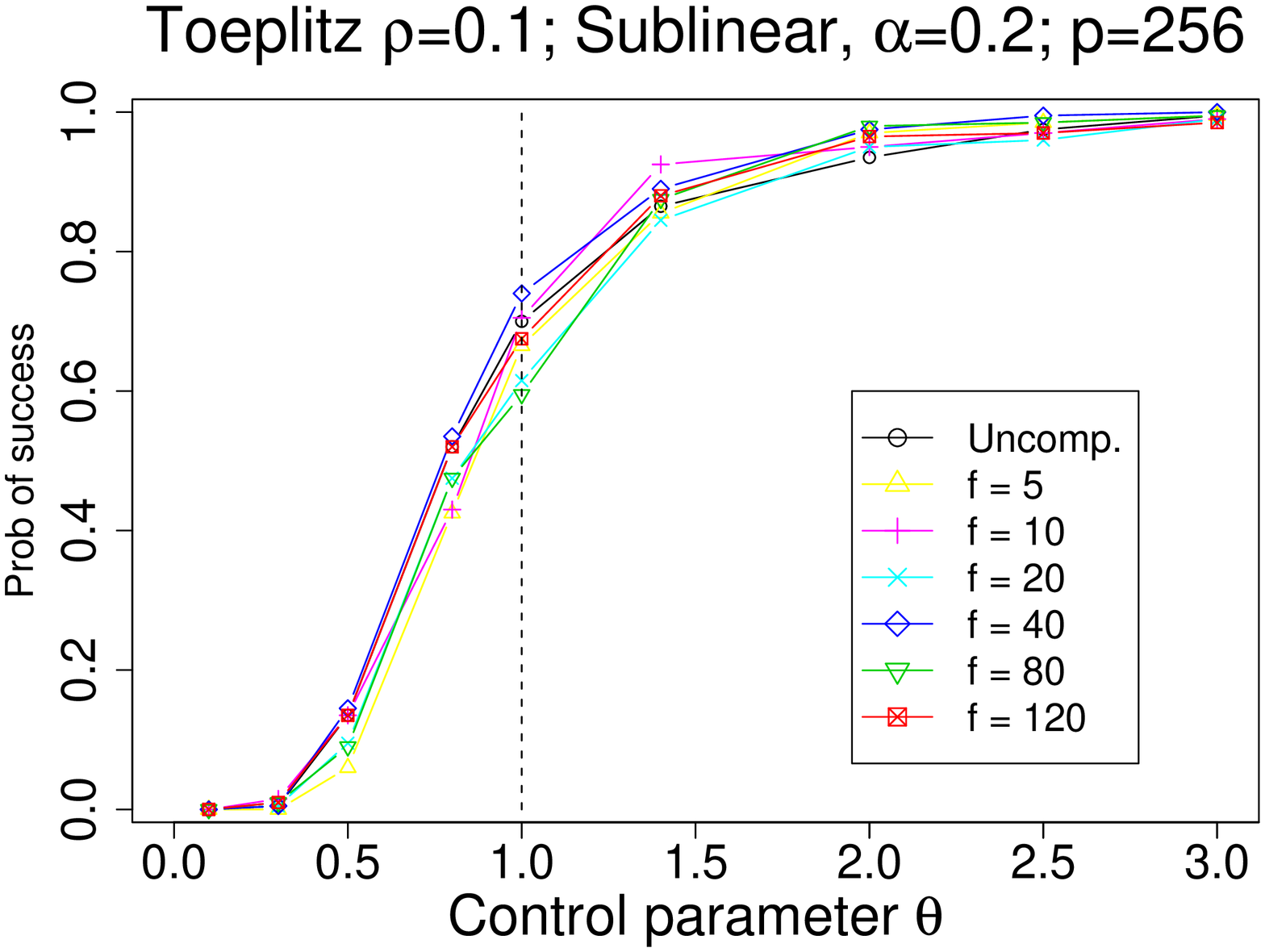}
\end{tabular}
\end{tabular}
\end{center}
\caption{
Plots of the number of samples versus the probability of success.
The three sets of curves on the left panel map to $p=128, 256$ and $512$,
with vertical dashed lines marking $m = 2 \theta s \log(p-s) + s + 1$
for $\theta = 1$, and $s = 5, 9$ and $15$ respectively.}
\label{fig:plots_c}
\end{figure*}

\begin{figure*}
\begin{center}
\begin{tabular}{cc}
\begin{tabular}{c}
\lleft
\includegraphics[width=0.50\textwidth,angle=-90]{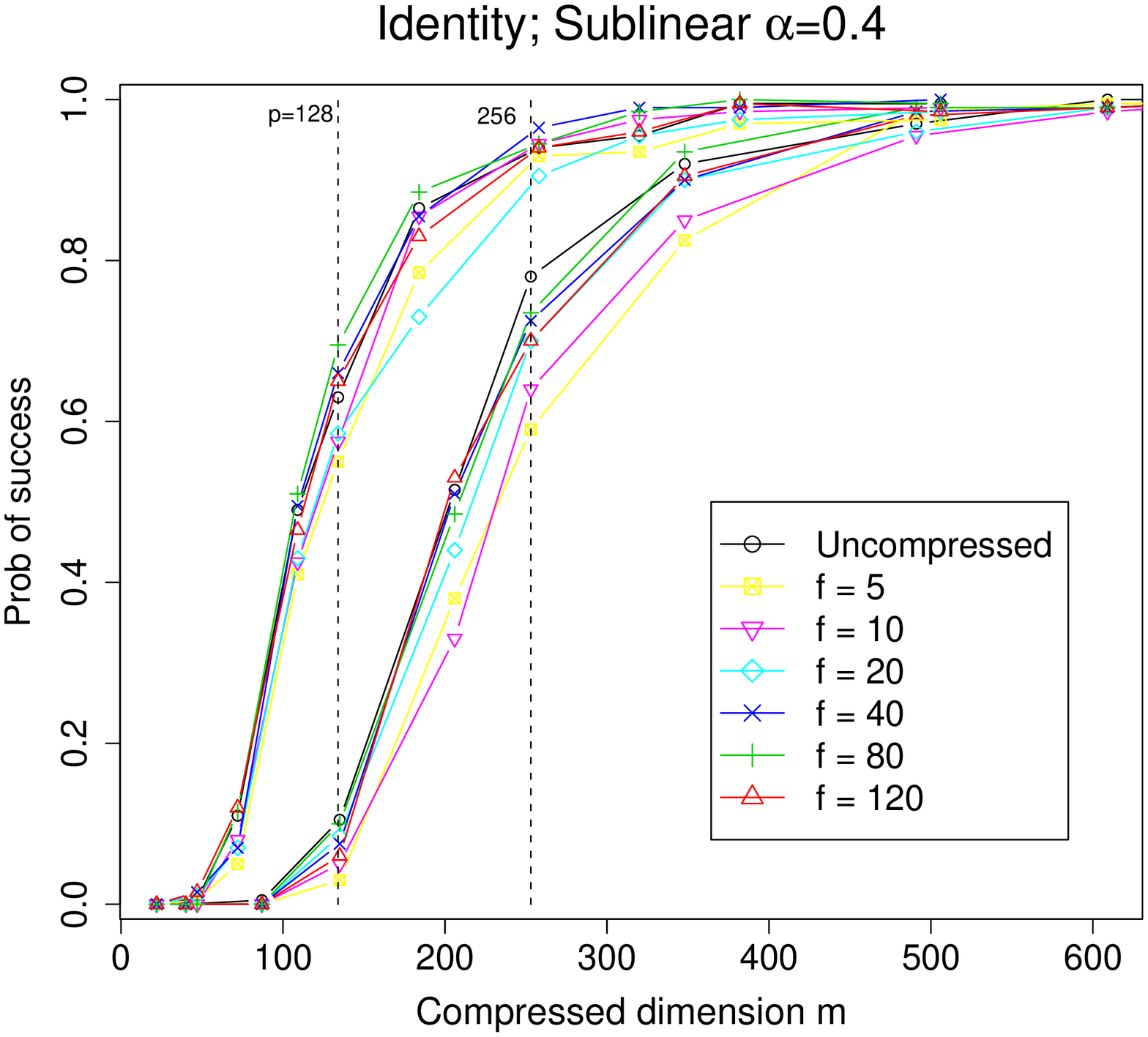}  \\
\lleft
\includegraphics[width=0.50\textwidth,angle=-90]{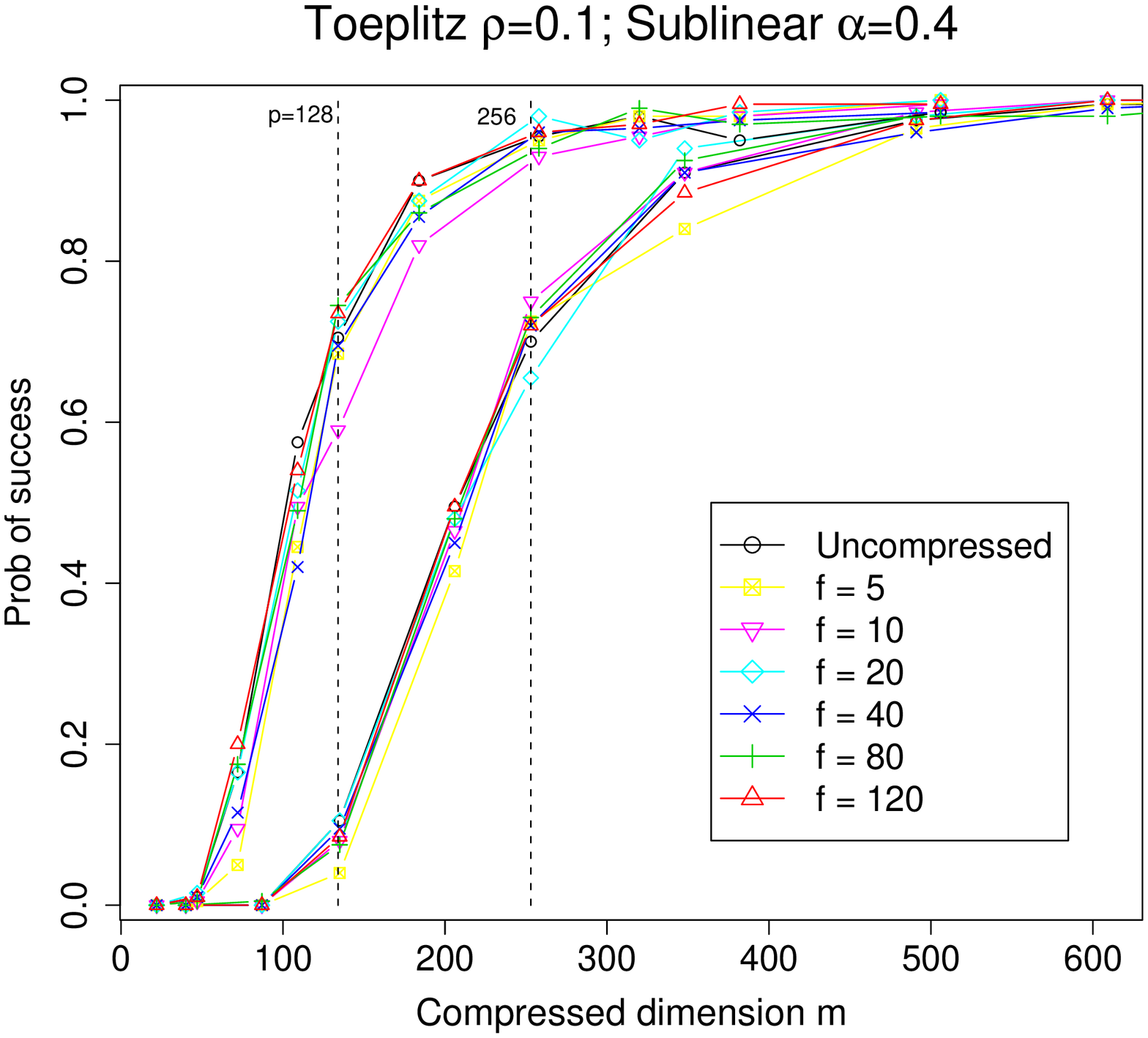}
\end{tabular} & 
\begin{tabular}{c}
\sleft
\includegraphics[width=.25\textwidth,angle=-90]{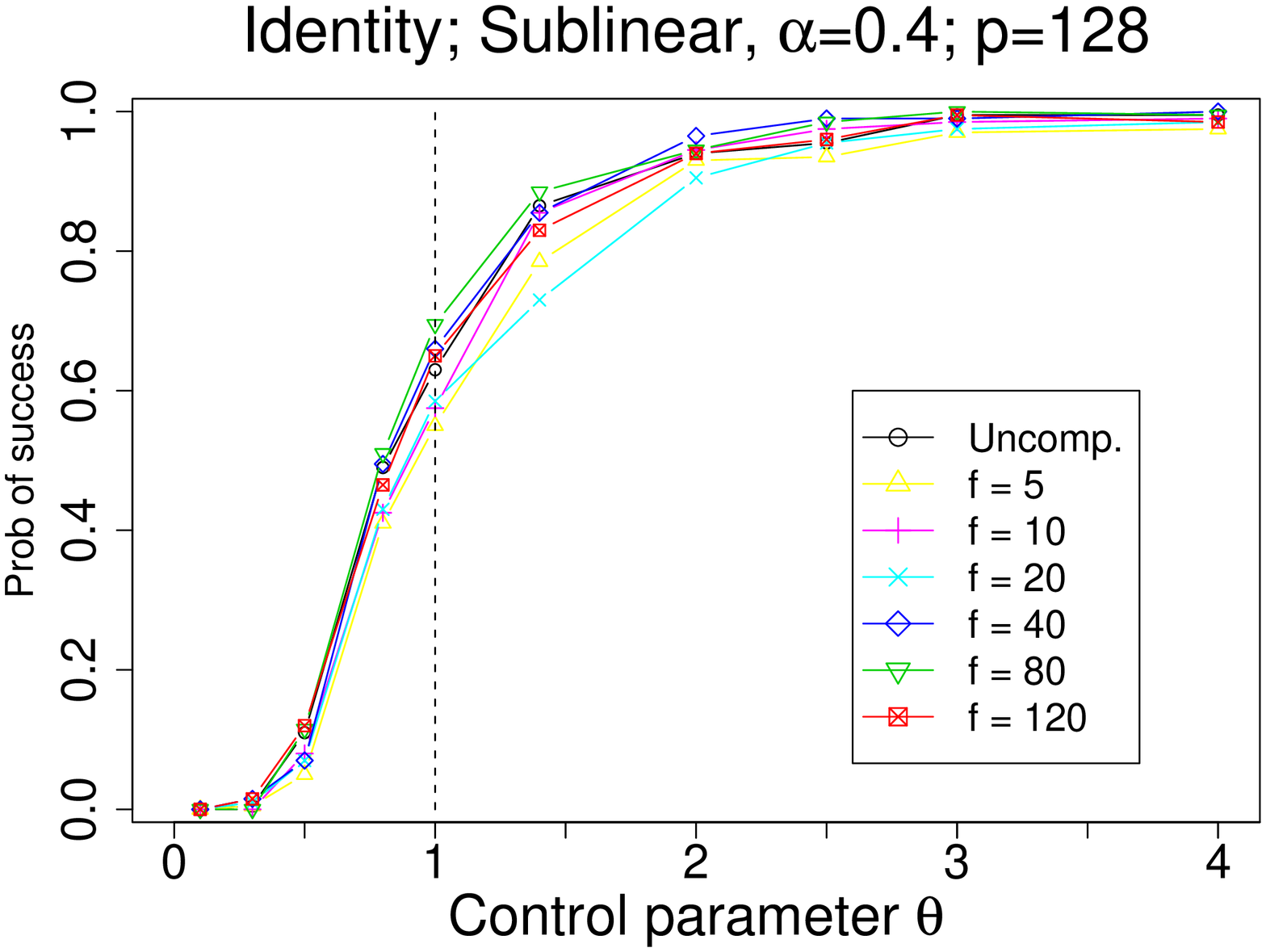} \\ 
\sleft
\includegraphics[width=.25\textwidth,angle=-90]{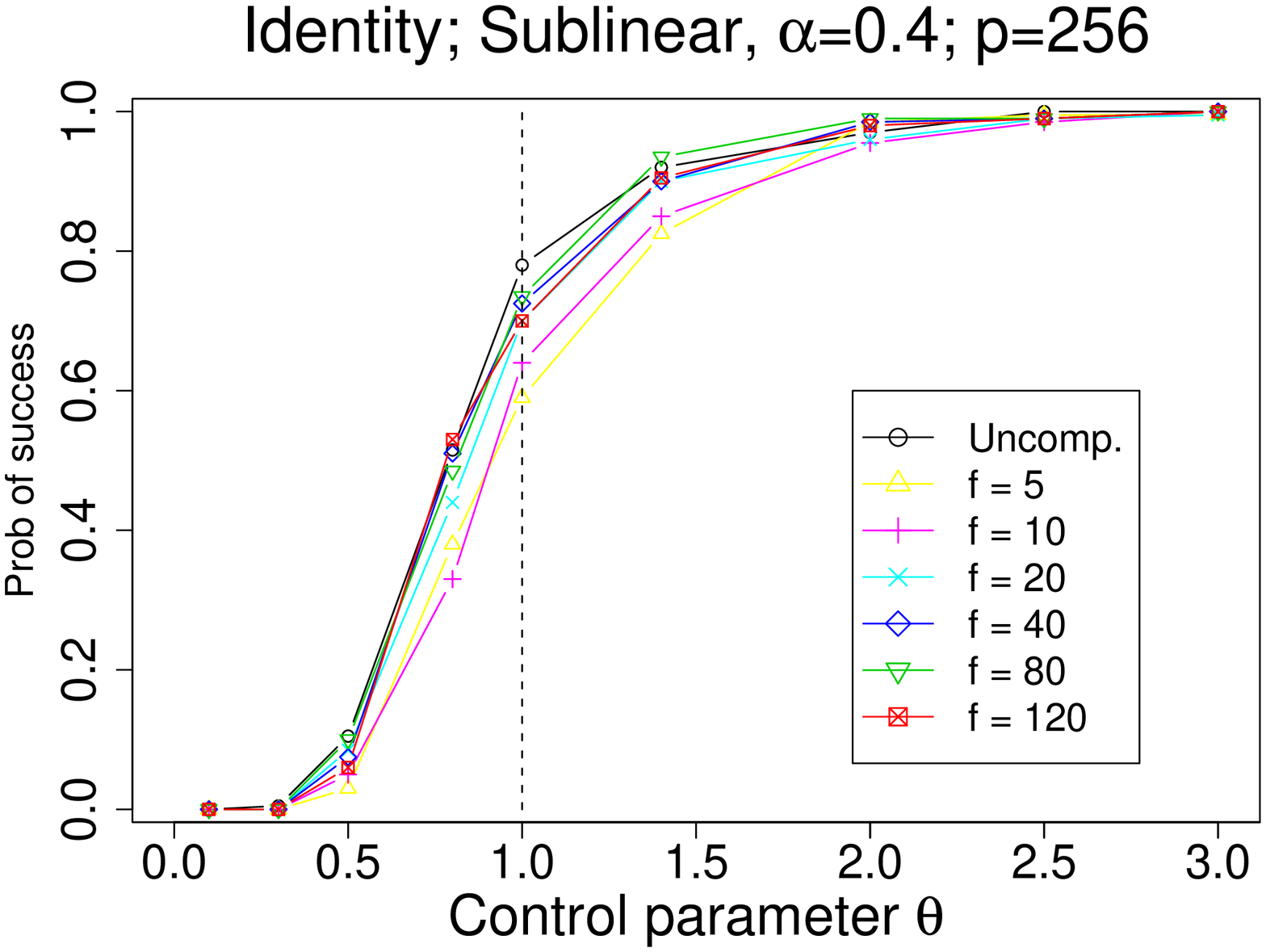} \\
\sleft
\includegraphics[width=.25\textwidth,angle=-90]{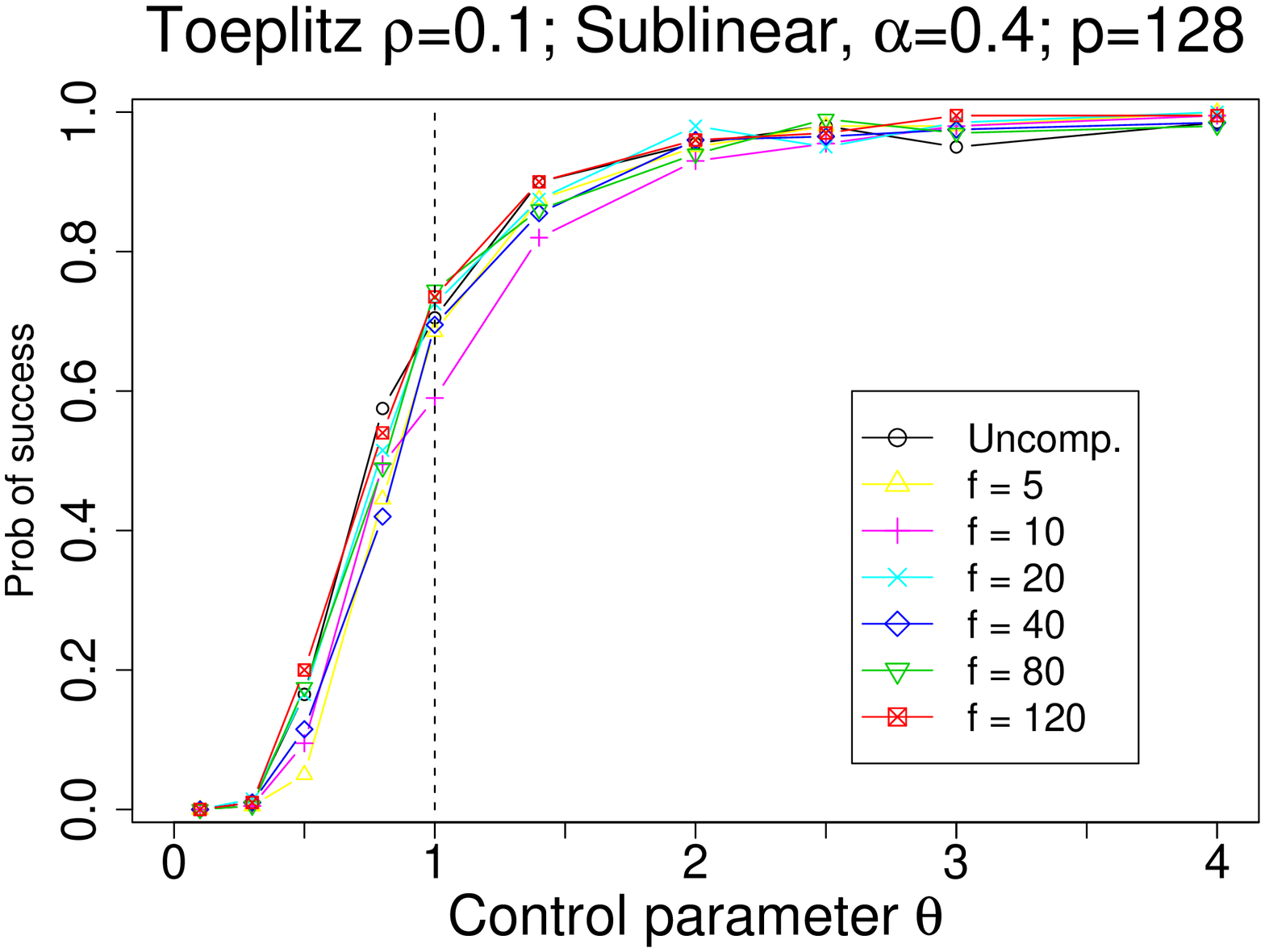} \\
\sleft
\includegraphics[width=.25\textwidth,angle=-90]{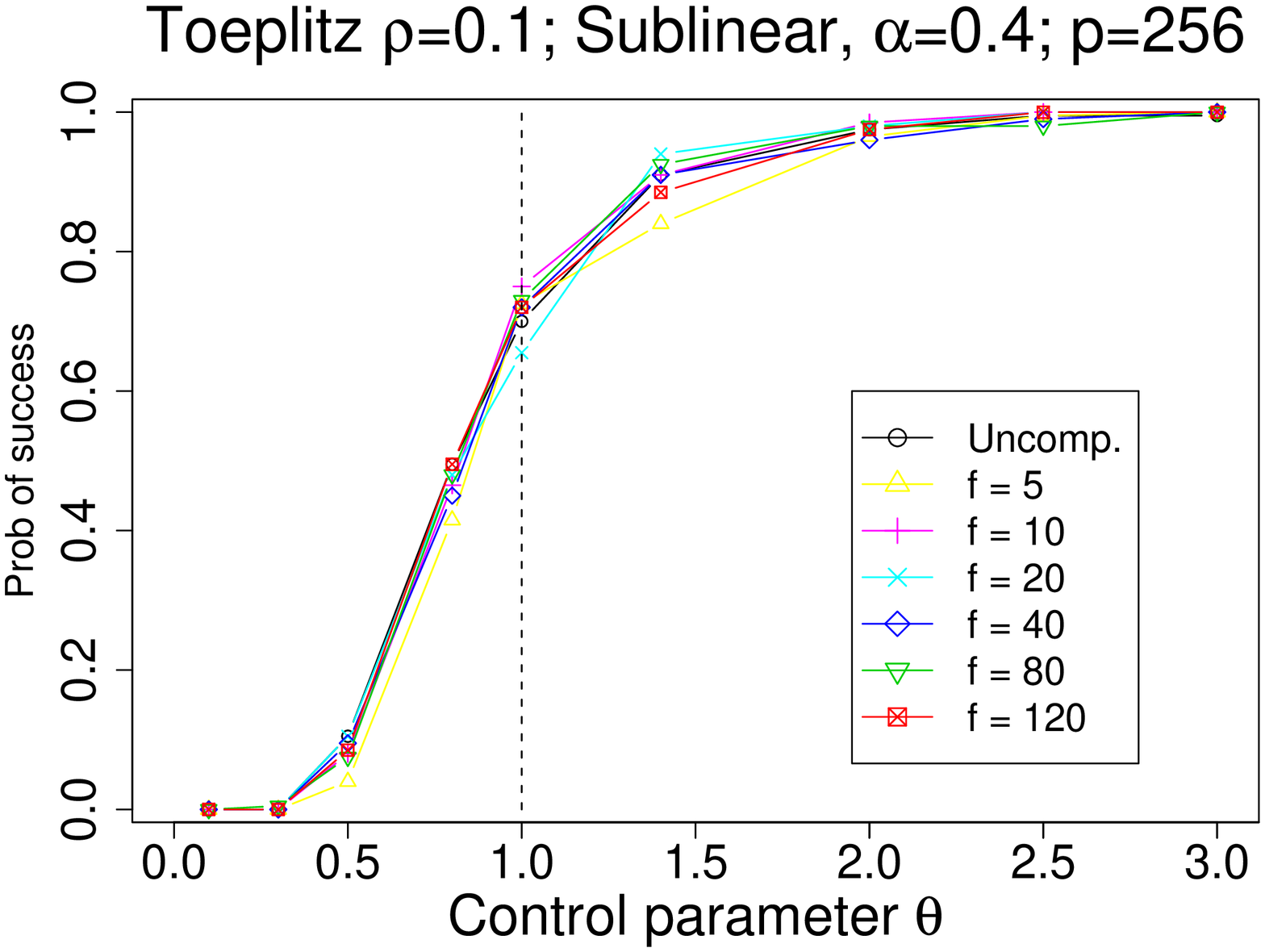}
\end{tabular}
\end{tabular}
\caption{
Plots of the number of samples versus the probability of success.
The two sets of curves on the left panel correspond to $p=128$ and $256$,
with vertical dashed lines mapping to $m = 2 \theta s \log(p-s) + s + 1$
for $\theta = 1$, and $s = 9$ and $15$ respectively.}
\end{center}
\label{fig:plots_d}
\end{figure*}

\clearpage

\subsection{Persistence}

We now study the behavior of predictive and empirical risks under compression.
In this section, we refer to $\lasso(Y \sim X, L)$ as the code that solves
the following $\ell_1$-constrained optimization problem directly, 
based on algorithms described by~\cite{OPT00}:
\begin{subeqnarray}
(P_3) \hspace{1cm} \tilde\beta &=& 
\argmin \twonorm{Y - X \beta}  \\
\hspace{1cm} && \text{such that } \norm{\beta}_1 \leq L.
\end{subeqnarray}
Let us first define the following $\ell_1$-balls $\Ball_n$ and
$\Ball_{n,m}$ for a fixed uncompressed sample size $n$ and dimension $p_n$, 
and a varying compressed sample size $m$.
By~\cite{GR04}, given a sequence of sets of estimators
\begin{gather}
\Ball_n = \{\beta: \norm{\beta}_1 \leq L_n \}, \;
\text{where}\;
L_n = \frac{n^{1/4}}{\sqrt{\log n}},
\end{gather}
the uncompressed Lasso estimator $\hat\beta_n$ as 
in~(\ref{eq:lasso-estimator}) is persistent over $\Ball_n$.
Given $n, p_n$, Theorem~\ref{thm:persistence}
shows that, given a sequence of sets of estimators 
\begin{gather}
\Ball_{n,m} = \{\beta: \norm{\beta}_1 \leq L_{n, m} \},\;
\text{where}\;  L_{n, m} = \frac{m^{1/4}}{\sqrt{\log (n p_n)}},
\end{gather}
for $\log^2 (n p_n) \leq m \leq n $, 
the compressed Lasso estimator $\hat\beta_{n,m}$ as 
in~(\ref{eq:com-lasso-estimator}) is persistent over $\Ball_{n, m}$.
 
We use simulations to illustrate how close the compressed empirical 
risk computed through~(\ref{eq:exp-emp-risk}) is to that of the best 
compressed predictor $\beta_*$ as in~(\ref{eq:comp-beta-star}) for a 
given set $\Ball_{n,m}$, the size of which depends on 
the data dimension $n, p_n$ of an uncompressed design matrix $X$, and the
compressed dimension $m$; we also illustrate how close these two type of 
risks are to that of the best uncompressed predictor defined 
in~(\ref{eq:oracle-unc-estimator}) for a given set $\Ball_n$ for all
$\log n p_n \leq m \leq n$.

We let the row vectors of the design matrix be independent identical
copies of a random vector $X \sim N(0, \Sigma)$.
For simplicity, we generate $Y = X^T \beta^* + \e$, where
$X$ and $\beta^* \in \R^p$, $\expct{\e} = 0$ and $\expct{\e^2} = \sigma^2$; 
note that $\expct{Y|X} = X^T \beta^*$, although the persistence
model need not assume this.
Note that for all $m \leq n$,
\begin{equation}
L_{n, m} = \frac{m^{1/4}}{\sqrt{\log (n p_n)}} \leq L_n
\end{equation} 
Hence the risk of the model constructed on the
compressed data over $\Ball_{n,m}$ is necessarily no smaller than 
the risk of the model constructed on the uncompressed data over 
$\Ball_n$, for all $m \leq n$.  

For $n= 9000$ and $p = 128$, we set $s(p) = 3$ and $9$ respectively, following the 
sublinear sparisty~(\ref{eq:sublinear}) with $\alpha =0.2$ and $0.4$;
correspondingly, two set of coefficients are chosen for $\beta^*$,
\begin{eqnarray}
\beta^*_a &=& (-0.9, 1.1, 0.687, 0, \ldots, 0)^T 
\end{eqnarray}
so that $\norm{\beta^*}_1 < L_n$ and $\beta^*_a \in \Ball_n$, and 
\begin{eqnarray}
\beta^*_b &=& (-0.9, -1.7, 1.1, 1.3, -0.5, 2, -1.7, -1.3, -0.9, 0,
\ldots, 0)^T
\end{eqnarray}
so that $\norm{\beta^*_b}_1 > L_n$ and $\beta^*_b \not\in\Ball_n$.

In order to find $\beta_*$ that minimizes the predictive risk
$R(\beta) = \expct{(Y - X^T \beta)^2}$, 
we first derive the following expression for the risk.  With $\Sigma =
A^T A$, a simple calculation shows that 
\begin{eqnarray}
\E(Y - X^T \beta)^2 -\E(Y^2) & = & 
 - \beta^{*T} \Sigma \beta^* + \twonorm{A \beta^* - A \beta}^2.
\end{eqnarray}
Hence 
\begin{subeqnarray}
R(\beta) & = & 
\E(Y^2) - \beta^{*T} \Sigma \beta^* + \twonorm{A \beta^* - A \beta}^2 \\
&= & 
\E(Y^2) - \beta^{*T} \expct{X X^T} \beta^* 
+ \twonorm{A \beta^* - A \beta}^2 \\
& = & 
\sigma^2 + \twonorm{A \beta^* - A \beta}^2 .
\end{subeqnarray}
\begin{figure*}
\begin{center}
\lleft
\begin{tabular}{c}
\includegraphics[width=0.54\textwidth,angle=-90]{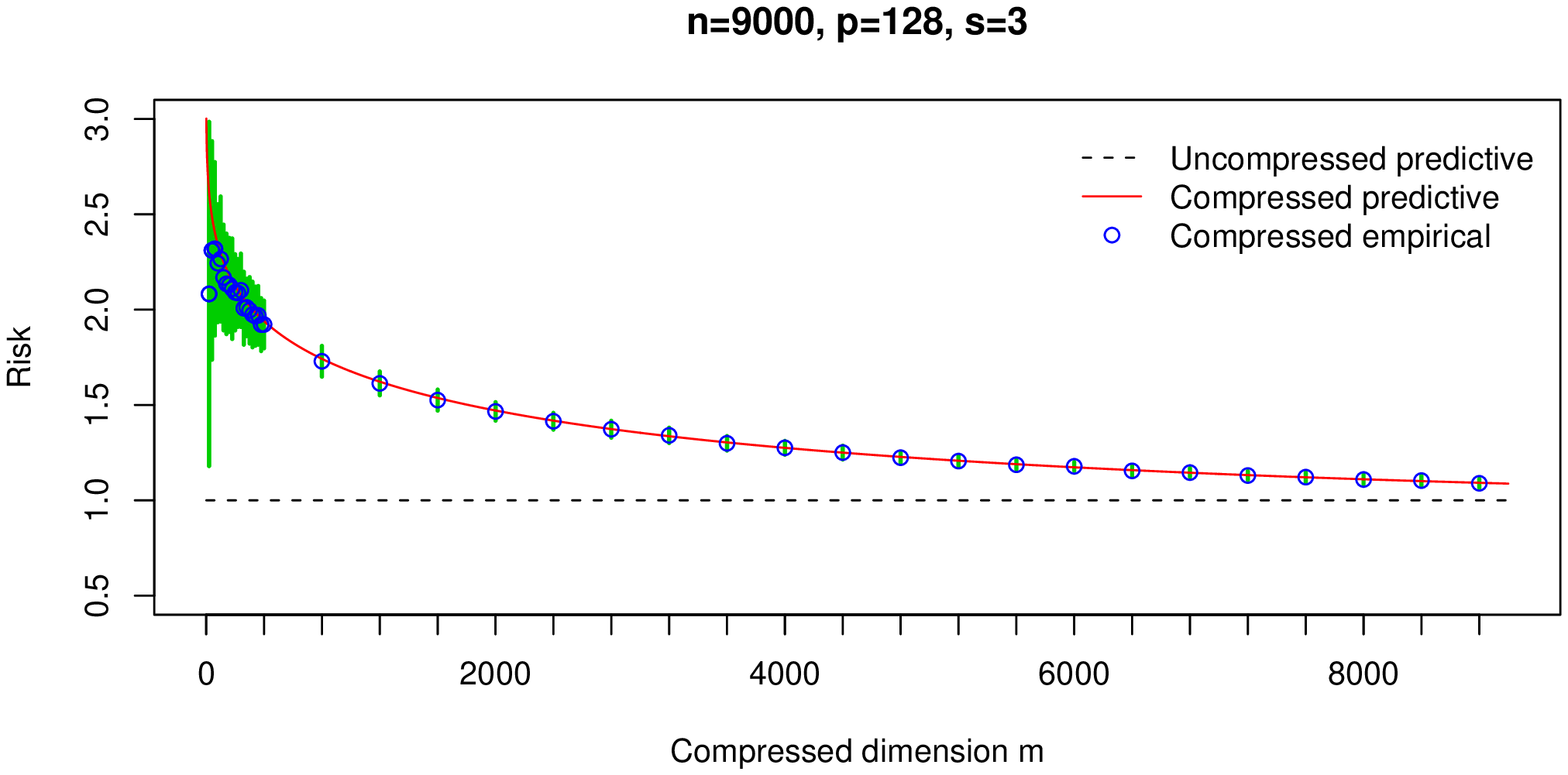} \\
\includegraphics[width=0.54\textwidth,angle=-90]{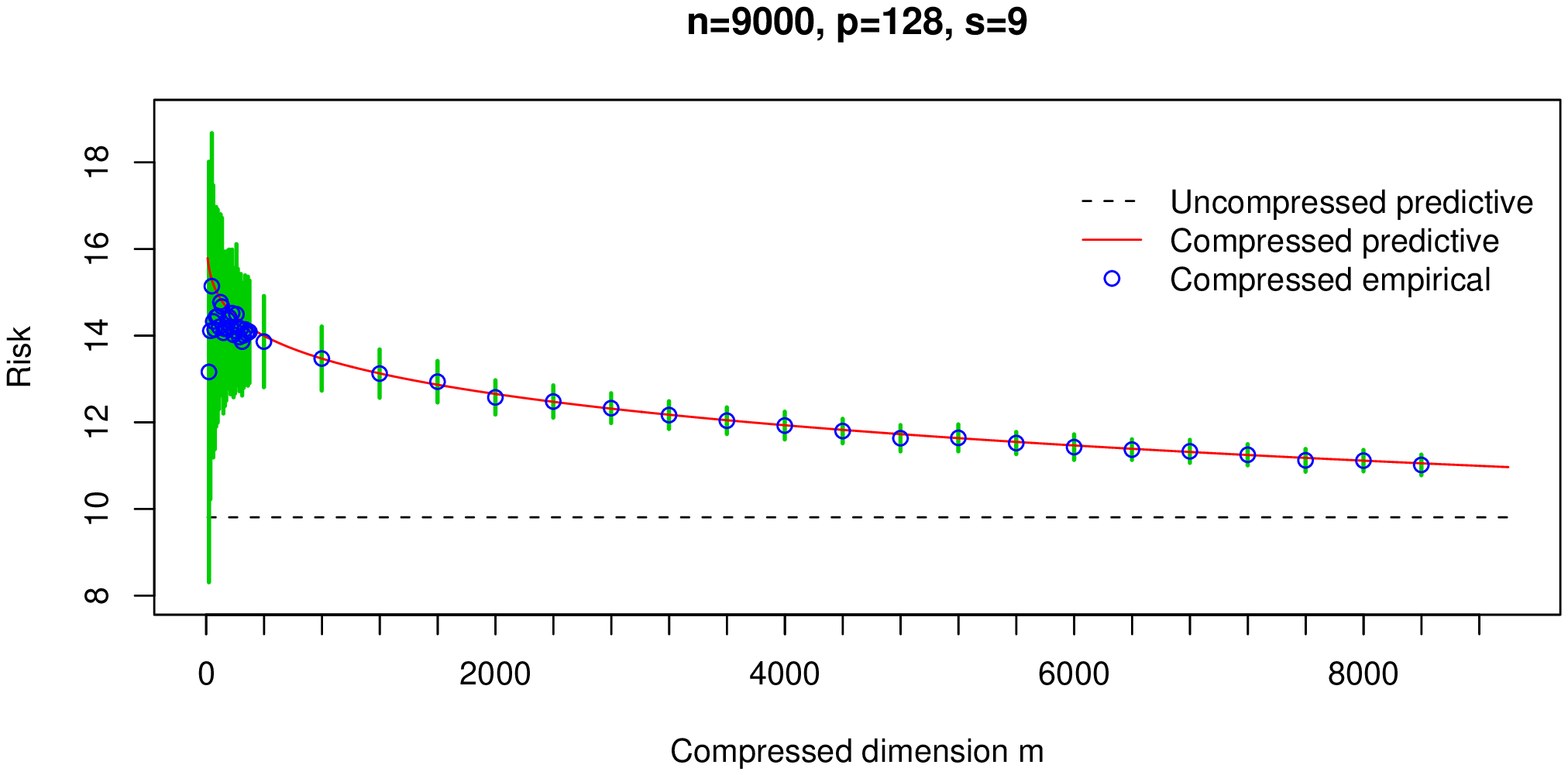} \\
\end{tabular}
\caption{
$L_n = 2.6874$ for $n= 9000$. Each data point corresponds to the mean 
empirical risk over $100$ trials,  and each vertical bar shows one standard 
deviation. Top plot: risk versus compressed dimension for $\beta^* = \beta^*_a$;
the uncompressed oracle predictive risk is $R = 1$.
Bottom plot:  risk versus compressed dimension for $\beta^* = \beta^*_b$;
the uncompressed oracle predictive risk is $R = 9.81$.}
\end{center}
\end{figure*}
For the next two sets of simulations, we fix $n = 9000$ and $p_n = 128$.
To generate the uncompressed predictive (oracle) risk curve, we
let
\begin{gather}
\hat\beta_n = \argmin_{\norm{\beta}_1 \leq L_n} R(\beta)
= \argmin_{\norm{\beta}_1 \leq L_n} \twonorm{A \beta^* - A \beta}^2.
\end{gather}
Hence we obtain $\beta_*$ by running
$\lasso(\Sigma^{\half} \beta^* \sim \Sigma^{\half}, L_n)$.
To generate the compressed predictive (oracle) curve, for each $m$, we let
\begin{gather}
\hat\beta_{n,m} = \argmin_{\norm{\beta}_1 \leq L_{n,m}} R(\beta)
= \argmin_{\norm{\beta}_1 \leq L_{n,m}} \twonorm{A \beta^* - A \beta}^2.
\end{gather}
Hence we obtain $\beta_*$ for each $m$ by running
$\lasso(\Sigma^{\half} \beta^* \sim \Sigma^{\half}, L_{n, m})$.
We then compute oracle risk for both cases as
\begin{gather}
R(\hat\beta) = (\hat\beta - \beta^*)^T \Sigma (\hat\beta - \beta^*) + \sigma^2.
\end{gather}
For each chosen value of $m$, we compute the corresponding empirical 
risk, its sample mean and sample standard deviation
by averaging over $100$ trials.  For each trial, we randomly draw
$X_{n \times p}$ with independent row vectors
$x_i \sim N(0, T(0.1))$, and $Y = X \beta^* + \e$.
If $\beta$ is the coefficient
vector returned by $\lasso(\Phi Y \sim \Phi X, L_{n, m})$, then
the empirical risk is computed as
\begin{gather}
\label{eq:exp-emp-risk}
\hat R(\beta) = \gamma^T \hat\Sigma \gamma, \; \; \text{where} \; \;
\hat\Sigma = \inv{m} \Q^T \Phi^T \Phi \Q.
\end{gather}
where $\Q_{n \times(p+1)} = [Y, X]$ 
and $\gamma = (-1,\beta_1, \ldots, \beta_p)$.

\section{Proofs of Technical Results}

\subsection{Connection to the Gaussian Ensemble Result}
\label{sec:connection}

We first state a result which directly follows from
the analysis of Theorem~\ref{thm:recovery}, and we then compare it with
the Gaussian ensemble result of~\cite{Wai06} that we summarized in 
Section~\ref{sec:background}. 

First, let us state the following slightly relaxed conditions 
that are imposed on the design matrix by~\cite{Wai06}, and also
by~\cite{ZY07}, when $X$ is deterministic:
\begin{subeqnarray}
\label{eq:incoa}
\left\| X^T_\Sc X_S (X_S^T X_S)^{-1}\right\|_\infty & \leq & 1-\eta,
\;\;\text{for some $\eta\in(0,1]$, and} \\
\label{eq:incob}
\Lambda_{\text{min}}\left(\onen X_S^T X_S\right) & \geq & C_{\text{min}} > 0,
\end{subeqnarray}
where $\Lambda_{\text{min}}(A)$ is the smallest eigenvalue of $A$.
In Section~\ref{sec:proofs}, Proposition~\ref{pro:irrep} shows that 
$S$-incoherence implies the conditions in equations~\eqref{eq:incoa} 
and~\eqref{eq:incob}.

From the proof of Theorem~\ref{thm:recovery} it is easy to verify
the following.
Let $X$ be a deterministic matrix satisfying conditions specified in 
Theorem~\ref{thm:recovery}, and let all constants be the same as in
Theorem~\ref{thm:recovery}.
Suppose that, before compression, we have noiseless responses 
$Y = X \beta^*$, and we observe, after compression, $\tilde{X} = \Phi X$, and
\begin{gather}
\label{eq:un-compressed-noise}
\tilde{Y} =  \Phi Y + \e = \tilde{X} \beta^* + \e, 
\end{gather}  
where $\Phi_{m \times n}$ is a Gaussian ensemble with independent 
entries: $\Phi_{i, j} \sim N(0, 1/n), \forall i, j$, and 
$\e \sim N(0, \sigma^2 I_m)$.  
Suppose $m \geq \left(\frac{16 C_1 s^2}{\eta^2} + \frac{4 C_2 s}{\eta}\right)
(\ln p + 2 \log n + \log 2(s+1))$ and 
$\lambda_m \rightarrow 0$ satisfies~(\ref{eq:thm-cond-lambda}).
Let $\tilde\beta_m$ be an optimal solution to the compressed lasso,
given $\tilde{X}, \tilde{Y}, \e$ and $\lambda_m>0$:
\begin{equation}
\label{eq:solution-set-cor}
\tilde\beta_m = \argmin_{\beta \in \R^p} \; 
\frac{1}{2m}\|\W-\Z\beta\|_2^2 + \lambda_m \|\beta\|_1.
\end{equation}
Then the compressed lasso is sparsistent:
$\prob{\supp(\tilde{\beta}_m) = \supp(\beta)} \rightarrow 1
\;\;\text{as}\; m\rightarrow \infty$.
Note that the upper bound on $m \leq \sqrt{\frac{n}{16\log n}}$ 
in~(\ref{eq:thm-m-bounds}) is no longer necessary, since we are handling the 
random vector $\e$ with i.i.d entries rather than the non-i.i.d $\Phi \e$ as 
in Theorem~\ref{thm:recovery}.

We first observe that 
the design matrix $\tilde{X} = \Phi X$ as in~(\ref{eq:un-compressed-noise}) 
is exactly a Gaussian ensemble that~\cite{Wai06} analyzes.  
Each row of $\tilde{X}$ is chosen as an i.i.d. Gaussian random vector 
$\sim N(0, \Sigma)$ with covariance matrix $\Sigma = \onen X^T X$.
In the following, let 
$\Lambda_{\min}(\Sigma_{SS})$ be the minimum eigenvalue of $\Sigma_{SS}$ 
and $\Lambda_{\max}(\Sigma)$ be the maximum eigenvalue of $\Sigma$.
By imposing the $S$-incoherence condition on $X_{n \times p}$, 
we obtain the following two conditions on the covariance matrix $\Sigma$, 
which are required by~\cite{Wai06} 
for deriving the threshold conditions~(\ref{eq:wain-succ-bound}) 
and~(\ref{eq:wain-fail-bound}), when the design matrix is a Gaussian 
ensemble like $\tilde{X}$:
\begin{subeqnarray}
\label{eq:incoa-covariance}
\norm{\Sigma_{\Sc S} (\Sigma_{S S})^{-1}}_\infty 
&\leq& 1-\eta, \;\;\text{for $\eta\in(0,1]$, and} \\
\label{eq:incob-covariance}
\Lambda_{\min}(\Sigma_{SS}) 
&\geq& C_{\text{min}} > 0.
\end{subeqnarray}
When we apply this to $\tilde X = \Phi X$ where $\Phi$ is from the
Gaussian ensemble and $X$ is deterministic, this condition requires
that
\begin{subeqnarray}
\norm{X_\Sc^T X_S (X_S^T X_S)^{-1}}_\infty 
&\leq& 1-\eta, \;\;\text{for $\eta\in(0,1]$, and} \\
\Lambda_{\text{min}}\left(\onen X_S^T X_S\right) 
&\geq& C_{\text{min}} > 0.
\end{subeqnarray}
since in this case 
$\expct{{\textstyle \frac{1}{m}} X^T \Phi^T \Phi X} = {\textstyle \frac{1}{n}}X^T X$.
In addition, it is assumed in~\cite{Wai06} that
there exists a constant $C_{\max}$ such that
\begin{gather}
\label{eq:wain-max-eigen}
\Lambda_{\text{max}}(\Sigma) \leq C_{\max}.
\end{gather}
This condition need not hold for $\onen X^T X$;
In more detail, given $\Lambda_{\text{max}}(\onen X^T X) =
\inv{n} \Lambda_{\text{max}}(X^T X) = \inv{n} \twonorm{X}^2$,
we first obtain a loose upper and lower bound for $\twonorm{X}^2$ through the 
Frobenius norm $\norm{X}_F$ of $X$. Given 
that $\twonorm{X_j}^2 = n, \forall j\in \{1, \ldots, p\}$, we have
$\norm{X}^2_F = \sum_{j=1}^p \sum_{i=1}^n |X_{ij}|^2 = p n$.
Thus by $\twonorm{X} \leq \norm{X}_F \leq \sqrt{p} \twonorm{X}$, we obtain
\begin{gather}
n = \inv{p} \norm{X}^2_F \leq \twonorm{X}^2 \leq \norm{X}^2_F = pn,
\end{gather}
which implies that $1 \leq \Lambda_{\max}(\onen X^T X) \leq p$.  
Since we allow $p$ to grow with $n$,~(\ref{eq:wain-max-eigen}) need not 
hold.

Finally we note that the conditions on $\lambda_m$ in
the Gaussian Ensemble result of~\cite{Wai06} 
are~(\ref{eq:thm-cond-lambda}~$a$) and a slight variation 
of~(\ref{eq:thm-cond-lambda}~$b$):
\begin{eqnarray}
\label{eq:wain-cond-b}
&\displaystyle \inv{\rho_m}\left\{ 
\sqrt{\frac{\log s}{m}}+ \lambda_m \right \} \rightarrow 0;&
\end{eqnarray}
hence if we further assume that 
$\norm{(\onen X_S^T X_S)^{-1}}_{\infty} \leq D_{\max}$ for some constant 
$D_{\max} \leq +\infty$, as required by~\cite{Wai06} on 
$\norm{\Sigma_{SS}^{-1}}_{\infty}$,
~(\ref{eq:thm-cond-lambda}~$b$) and~(\ref{eq:wain-cond-b}) are equivalent.

Hence by imposing the $S$-incoherence condition on a deterministic 
$X_{n \times p}$ with all columns of $X$ having $\ell_2$-norm $n$, 
when $m$ satisfies the lower bound in~(\ref{eq:thm-m-bounds}), rather
than~(\ref{eq:wain-succ-bound}) with 
$\theta_u = \frac{C_{\max}}{\eta^2 C_{\min}}$ with $C_{\max}$ as 
in~(\ref{eq:wain-max-eigen}),
we have shown that the probability of sparsity recovery through lasso 
approaches one, given $\lambda_m$ satisfies~(\ref{eq:thm-cond-lambda}),
when the design matrix is a Gaussian Ensemble generated through 
$\Phi X$ with $\Phi_{m \times n}$ having independent 
$\Phi_{i, j} \in N(0, 1/n), \forall i, j$.
We do not have a comparable result for the failure of recovery 
given~(\ref{eq:wain-fail-bound}).

\def\fatnorm#1{|\kern-.2ex|\kern-.2ex| #1 |\kern-.2ex|\kern-.2ex|}
\def\tA{\tilde A}

\subsection{$S$-Incoherence}

\label{sec:proofs}
We first state some generally useful results about matrix norms.
\begin{theorem}
\textnormal{\bf{\citep[p.~301]{HJ90}}}
\label{thm:inverse}
If $\fatnorm{\cdot}$ is a matrix norm and $\fatnorm{A} < 1$, then
$I + A$ is invertible and 
\begin{gather}
(I + A)^{-1} = \sum_{k = 0}^{\infty} (-A)^k.
\end{gather}
\end{theorem}

\begin{proposition}
\label{pro:inverse}
If the matrix norm $\norm{\cdot}$ has the property that $\norm{I} = 1$, and
if $A \in M_n$ is such that $\norm{A} < 1$, we have 
\begin{gather} 
\inv{1 + \norm{A}} \leq \norm{(I + A)^{-1}} \leq \inv{1 - \norm{A}}.
\end{gather}
\end{proposition}

\begin{proof}
The upper bound follows from Theorem~\ref{thm:inverse} and triangle-inequality;
\begin{eqnarray}
\norm{(I + A)^{-1}} =  \norm{\sum_{k=0}^{\infty}(-A)^k} 
 \leq   \sum_{k=0}^{\infty}\norm{-A}^k 
= \sum_{k=0}^{\infty}\norm{A}^k  = \frac{1}{1 - \norm{A}}.
\end{eqnarray}

The lower bound follows that general inequality 
$\norm{B^{-1}} \geq \inv{\norm{B}}$, given that 
$\norm{I} \leq \norm{B} \norm{B^{-1}}$ and the triangle inequality:
$\norm{A + I} \leq \norm{A} + \norm{I} = \norm{A} + 1$.
\begin{gather}
\norm{(A + I)^{-1}} \geq  \inv{\norm{A + I}} \geq \inv{1 +\norm{A}}
\end{gather}
\end{proof}

Let us define the following symmetric matrices, that we use throughout the
rest of this section.
\begin{subeqnarray}
\label{eq:defA}
A & = & \inv{n} X_{S}^T X_S - I_{\size{S}} \\
\label{eq:defAp}
\tA & = & \inv{m}(\Phi X)_{S}^T (\Phi X)_S - I_s = 
\inv{m}Z_{S}^T Z_S - I_s.
\end{subeqnarray}
We next show the following consequence of the $S$-Incoherence condition. 
\label{sec:append-inc}
\begin{proposition}
\label{pro:A-twonorm}
Let $X$ be an $n \times p$ that satisfies the $S$-Incoherence condition.
Then for the symmetric matrix $A$ in~\ref{eq:defA} , we have
$\norm{A}_{\infty} = \norm{A}_{1} \leq 1 - \eta$, for some $\eta \in (0, 1]$, and
\begin{gather}
\twonorm{A} \leq \sqrt{\norm{A}_{\infty} \norm{A}_{1}} \leq 1 - \eta.
\end{gather}
and hence $\Lambda_{\min}(\inv{n} X^T_S X_S) \geq \eta$, i.e., 
the $S$-Incoherence condition implies condition~(\ref{eq:incob}).
\end{proposition}

\begin{proof}
Given that $\twonorm{A} < 1$, $\twonorm{I} = 1$, 
and by Proposition~\ref{pro:inverse},
\begin{eqnarray}
\Lambda_{\min}(\inv{n} X^T_S X_S) = \inv{\twonorm{(\inv{n} X^T_S X_S)^{-1}}} 
=  \inv{\twonorm{(I + A)^{-1}}} \geq 1 - \twonorm{A} \geq \eta >0
\end{eqnarray}
\end{proof}

\begin{proposition}
\label{pro:irrep}
The $S$-Incoherence condition on an $n \times p$ matrix $X$ implies
conditions ~(\ref{eq:incoa}) and~(\ref{eq:incob}).
\end{proposition}

\begin{proof}
It remains to show~(\ref{eq:incoa}) given Proposition~\ref{pro:A-twonorm}.
Now suppose that the incoherence condition holds for some $\eta \in (0, 1]$,
i.e.,$\norm{\inv{n}X_{S^c}^T X_S}_{\infty} + \norm{A}_{\infty} \leq 1 - \eta$,
we must have
\begin{eqnarray}
\frac{\norm{\inv{n}X_{S^c}^T X_S}_{\infty}}{1 - \norm{A}_{\infty}}
\leq  1 - \eta,
\end{eqnarray}
given that $\norm{\inv{n} X_{S^c}^T X_S}_{\infty} + \norm{A}_{\infty} (1 - \eta) 
\leq  1 - \eta$ and  $1 - \norm{A}_{\infty} \geq \eta > 0$.

Next observe that, given $\norm{A}_{\infty} < 1$, by 
Proposition~\ref{pro:inverse}
\begin{eqnarray}
\norm{(\inv{n}X_S^T X_S)^{-1}}_{\infty} = \norm{(I + A)^{-1}}_{\infty}
 \leq \frac{1}{1 - \norm{A}_{\infty}}.
\end{eqnarray}

Finally, we have
\begin{subeqnarray}
\norm{X_{S^c}^T X_S (X_S^T X_S)^{-1}}_{\infty}
& \leq &  
\norm{\inv{n}X_{S^c}^T X_S}_{\infty}  
\norm{(\inv{n}X_S^T X_S)^{-1}}_{\infty} \\
& \leq & 
\frac{\norm{\inv{n}X_{S^c}^T X_S}_{\infty}}{1 - \norm{A}_{\infty}}
 \leq 1 - \eta.
\end{subeqnarray}
\end{proof}

\subsection{Proof of Lemma~\ref{lemma:adapt-RSV}}
\label{sec:append-RSV}

\begin{proofof}{Lemma~\ref{lemma:adapt-RSV}}
Let $\Phi_{ij} = \inv{\sqrt{n}} g_{ij}$, where 
$g_{ij}, \forall i = 1, \ldots, m, j = 1, \ldots, n$ are independent
$N(0,1)$ random variables.  We define 
\begin{gather}
Y_{\ell} := \sum_{k = 1}^n \sum_{j = 1}^n 
g_{\ell, k} g_{\ell, j} x_k y_j,
\end{gather}
and we thus have the following:
\begin{subeqnarray}
\ip{\Phi x}{\Phi y} & = & 
\inv{n}
 \sum_{\ell = 1}^m \sum_{k = 1}^n \sum_{j = 1}^n
g_{\ell, k} g_{\ell, j} x_k y_j \\
&  = & 
\inv{n} \sum_{\ell = 1}^m Y_{\ell}, 
\end{subeqnarray}
where $Y_{\ell}, \forall \ell$, are independent random variables, and
\begin{subeqnarray}
\expct{Y_{\ell}} & =  &
\expct{\sum_{k = 1}^n \sum_{j = 1}^n g_{\ell, k} g_{\ell, j} x_k y_j} \\
& = & 
\sum_{k = 1}^n  x_k y_k \expct{g_{\ell, k}^2} \\
& = & \ip{x}{y}
\end{subeqnarray}

Let us define a set of zero-mean independent random variables
$Z_1, \ldots, Z_m$,
\begin{gather}
Z_{\ell} := Y_{\ell} - \ip{x}{y} = Y_{\ell} - \expct{Y_{\ell}},
\end{gather}
 such that 
\begin{subeqnarray}
\frac{n}{m} \ip{\Phi x}{\Phi y} - \ip{x}{y}
& = & 
\inv{m}\sum_{\ell = 1}^m Y_{\ell} - \ip{x}{y}\\
& = & 
\inv{m}\sum_{\ell = 1}^m (Y_{\ell} - \ip{x}{y})\\
& = & 
\inv{m}\sum_{\ell = 1}^m Z_{\ell}.
\end{subeqnarray}

In the following, we analyze the integrability and tail behavior 
of $Z_{\ell}, \forall \ell$,
which is known as ``Gaussian chaos'' of order $2$. 

We first simplify notation by defining 
$Y := \sum_{k = 1}^n \sum_{j = 1}^n g_{k} g_{j} x_k y_j$,
where $g_k, g_j$ are independent $N(0,1)$ variates, and $Z$,
\begin{gather}
Z := Y - \expct{Y} = 
\sum_{k =1}^n \sum_{j = 1, j \not= k}^n g_{k} g_{j} x_k y_j + 
\sum_{k = 1}^n (g_{k}^2 -1) x_k y_k,
\end{gather}
where $\expct{Z} = 0$. Applying a general bound of~\cite{LT91} for Gaussian chaos gives that 
\begin{gather}
\label{eq:chaos}
\expct{\abs{Z}^q} \leq (q-1)^q (\expct{\abs{Z}^2})^{q/2}
\end{gather}
for all $q > 2$.

The following claim is based on~(\ref{eq:chaos}), whose proof 
appears in~\cite{RSV07}, which we omit.
\begin{claim}\textnormal{\bf{~(\cite{RSV07})}}
Let $M = e (\expct{|Z|^2}^{1/2}$ and 
$s = \frac{2e}{\sqrt{6 \pi}} \expct{\abs{Z}^2}$.
\[
\forall q > 2, \; \; \expct{Z^q} \leq q! M^{q-2} s/2.
\]
\end{claim}

Clearly the above claim holds for $q = 2$,
since trivially $\expct{\abs{Z}^q} \leq q! M^{q-2} s/2$ given that
for $q = 2$
\begin{subeqnarray}
q! M^{q-2} s/2 & = & 2 M^{2-2} s/2 = s  \\
& = & \frac{2e}{\sqrt{6 \pi}} \expct{\abs{Z}^2} \approx 
1.2522  \expct{\abs{Z}^2}.
\end{subeqnarray}
 
Finally, let us determine $\expct{\abs{Z}^2}$.
\begin{subeqnarray}
\label{eq:C1C2}
\expct{\abs{Z}^2} & = & 
\expct{
\left(\sum_{k =1}^n \sum_{j = 1, j \not= k}^n g_{k} g_{j} x_k y_j + 
+ \sum_{k = 1}^n (g_{k}^2 -1) x_k y_k \right)^2} \\
& = & 
\sum_{k \not= j} \expct{g_j^2} \expct{g_k^2} x_j^2 y_k^2  + 
\sum_{k =1}^n \expct{g_k^2 -1} x_k^2 y_k^2  \\
&  = & 
\sum_{k \not= j} x_j^2 y_k^2  + 
2 \sum_{k =1}^n x_k^2 y_k^2 \\
& \leq & 2 \twonorm{x}^2 \twonorm{y}^2 \\ 
& \leq & 2,
\end{subeqnarray}
given that $\twonorm{x}, \twonorm{y} \leq 1$.

Thus for independent random variables $Z_i, \forall i = 1, \ldots, m$, 
we have
\begin{gather}
\expct{Z_{i}^q} \leq q! M^{q-2} v_{i}/2,
\end{gather}
where $M = e (\expct{|Z|^2}^{1/2} \leq e \sqrt{2}$ and 
$v_i = \frac{2e}{\sqrt{6 \pi}} \expct{\abs{Z}^2} \leq \frac{4e}{\sqrt{6 \pi}}
\leq 2.5044, \forall i$.

Finally, we apply the following theorem, the proof of which follows arguments
from \cite{Ben62}:
\begin{theorem}\textnormal{\bf{(Bennett Inequality~\citep{Ben62})}}
Let $Z_1, \ldots, Z_m$ be independent random variables with zero mean such
that
\begin{gather}
\expct{\abs{Z_i}^q} \leq q! M^{q-2} v_i/2,
\end{gather}
for every $q \geq 2$ and some constant $M$ and $v_i, \forall i =1, \ldots, m$.
Then for $x > 0$,
\begin{gather}
\prob{\abs{\sum_{i=1}^m \abs{Z_i}} \geq \tau} \leq 2 
\exp\left(- \frac{\tau^2}{v + M\tau}\right)
\end{gather}
with $v = \sum_{i=1}^m v_i$.
\end{theorem}

We can then apply the Bennett Inequality to obtain the following:
\begin{subeqnarray}
\prob{\abs{\frac{n}{m}\ip{\Phi x}{\Phi y} - \ip{x}{y}} \geq \tau} 
& = & 
\prob{\abs{\inv{m}\sum_{\ell = 1}^m Z_{\ell}} \geq \tau} \\
& = & 
\prob{\abs{\sum_{\ell = 1}^m Z_{\ell}} \geq m \tau} \\
& \leq & 
 2 \exp \left(-\frac{(m \tau)^2}{2 \sum_{i=1}^m v_i + 2 M m \tau}\right) \\
& = & 
 2 \exp \left(-\frac{m \tau^2}{2/m \sum_{i=1}^m v_i + 2 M\tau}\right) \\
& \leq & 
2 \exp \left(-\frac{m \tau^2}{C_1 + C_2 \tau}\right)
\end{subeqnarray}
with $C_1 = \frac{4 e}{\sqrt{6\pi}} \approx 2.5044$ and 
$C_2 = \sqrt{8e} \approx 7.6885$.
\end{proofof}

\subsection{Proof of Proposition~\ref{pro:Phi-X}}
\label{sec:append-irre-cond}
\begin{proofof}{Proposition~\ref{pro:Phi-X}}
We use Lemma~\ref{lemma:adapt-RSV}, except that 
we now have to consider the change in absolute row sums of 
$\norm{\inv{n} X_{S^c}^T X_S}_{\infty}$ and  $\norm{A}_{\infty}$
after multiplication by $\Phi$. We first prove the following claim.

\begin{claim}
\label{claim:incoh}
Let $X$ be a deterministic matrix that satisfies the incoherence condition.
If 
\begin{equation}
\abs{\inv{m}\ip{\Phi X_{i}}{\Phi X_{j}} - \inv{n} \ip{X_{i}}{X_{j}}} \leq \tau,
\end{equation}
for any two columns $X_{i}, X_{j}$ of $X$ that are involved 
in~(\ref{eq:inner-product}), then
\begin{gather}
\label{eq:infty-norm}
\norm{\inv{m}(\Phi X)_{S^c}^T (\Phi X)_S}_{\infty} + 
\norm{\tA}_{\infty} \leq 1 - \eta + 2 s \tau, 
\end{gather}
and 
\begin{gather}
\label{eq:twonorm}
\Lambda_{\min}\left(\onem Z^T_S Z_S\right) \geq \eta - s \tau.
\end{gather}
\end{claim}

\begin{proof}
It is straightforward to show~(\ref{eq:infty-norm}).
Since each row in 
$\inv{m}(\Phi X)_{S^c}^T (\Phi X)_S$ and $A$ has $s$ entries, where 
each entry changes by at most $\tau$ compared to 
those in $\inv{n} X^T X$,  the absolute sum of any row can 
change by at most $s \tau$,
\begin{subeqnarray}
\abs{\norm{\inv{m}(\Phi X)_{S^c}^T (\Phi X)_S}_{\infty} - 
\norm{\inv{n}X_{S^c}^T X_S}_{\infty}}
& \leq & s \tau, \\
\abs{\norm{\tA}_{\infty} - \norm{A}_{\infty}} 
& \leq & s \tau,
\end{subeqnarray}
and hence
\begin{subeqnarray}
\norm{\inv{m}(\Phi X)_{S^c}^T (\Phi X)_S}_{\infty} 
+ \norm{\tA}_{\infty} 
& \leq & 
\norm{\inv{n}X_{S^c}^T X_S}_{\infty} + 
\norm{A}_{\infty} + 2 s \tau \\
& \leq & 1 - \eta + 2 s \tau.
\end{subeqnarray}

We now prove~(\ref{eq:twonorm}).  Defining $E = \tA - A$, we have
\begin{gather}
\twonorm{E} \leq s \max_{i, j}|\tA_{i, j} - A_{i, j}| \leq s \tau,
\end{gather}
given that each entry of $\tA$ deviates from that of $A$ by at most $\tau$.
Thus we have that
\begin{subeqnarray} 
\twonorm{\tA} & = & \twonorm{A + E} \\
& \leq & \twonorm{A} + \twonorm{E} \\
& \leq & \twonorm{A} + s  \max_{i, j}|E_{i, j}|\\
& \leq & 1 - \eta +  s \tau,
\end{subeqnarray}
where $\twonorm{A} \leq 1 - \eta$ is due to Proposition~\ref{pro:A-twonorm}.

Given that $\twonorm{I} = 1$ and $\twonorm{A} < 1$, 
by Proposition~\ref{pro:inverse}
\begin{subeqnarray}
\Lambda_{\min}\left(\onem Z^T_S Z_S\right) & = & 
\inv{\twonorm{(\inv{m} Z^T_S Z_S)^{-1}}} \\
& = & 
\inv{\norm{(I + \tA)^{-1}}_{2}} \\
& \geq & 1 - \twonorm{\tA} \\
& \geq & \eta - s \tau.
\end{subeqnarray}
\end{proof}

We let $\event$ represents union of the following events, where
$\tau = \frac{\eta}{4s}$:
\begin{enumerate}
\item
$\exists i \in S, j \in S^c$, such that 
$\abs{\inv{m} \ip{\Phi X_{i}}{\Phi X_{j}} - \inv{n} \ip{X_{i}}{X_{j}}} \geq \tau$,
\item
$\exists i, i'\in S$, such that 
$\abs{\inv{m}\ip{\Phi X_{i}}{\Phi X_{i'}} - \inv{n}\ip{X_{i}}{X_{i'}}} \geq \tau$,
\item
$\exists j \in S^c$, such that 
\begin{subeqnarray}
\abs{\inv{m} \ip{\Phi X_j}{\Phi X_{j}} - \inv{n}\ip{X_j}{X_j}} & = & 
\abs{\inv{m} \twonorm{\Phi X_j}^2 - \inv{n} \twonorm{X_j}^2} \\
& > & \tau.
\end{subeqnarray}
\end{enumerate}
Consider first the implication of $\event^c$, i.e., when none of
the events in $\event$ happens.  We immediately have that
~(\ref{eq:inner-product}),~(\ref{eq:twonorm}) and~(\ref{eq:small-eigen})
all simultaneously hold by Claim~\ref{claim:incoh}; and 
~(\ref{eq:inner-product}) implies that the incoherence condition
is satisfied for $Z = \Phi X$ by Proposition~\ref{pro:irrep}.

We first bound the probability of a single event counted in $\event$.
Consider two column vectors $x =  \frac{X_{i}}{\sqrt{n}} , 
y= \frac{X_{j}}{\sqrt{n}} \in \R^n$ in matrix $\frac{X}{\sqrt{n}}$, 
we have $\twonorm{x} = 1, \twonorm{y} = 1$, and 
\begin{subeqnarray}
\label{eq:ip-bound}
\lefteqn{
\prob{\abs{\inv{m}\ip{\Phi X_{i}}{\Phi X_{j}} - \inv{n}\ip{X_{i}}{X_{j}}} \geq \tau}}
& & \\ 
& = & \prob{\abs{\frac{n}{m}\ip{\Phi x}{\Phi y} - \ip{x}{y}} \geq \tau}  \leq 
2 \exp \left(\frac{- m \tau^2}{C_1 + C_2 \tau}\right) \\
& \leq & 
2 \exp \left(-\frac{m \eta^2/16s^2}{C_1 + C_2 \eta/4s}\right)
\end{subeqnarray}
given that $\tau = \frac{\eta}{4s}$.

We can now bound the probability that any such large-deviation event happens.
Recall that $p$ is the total number of columns of $X$ and 
$s = \size{S}$; the total number of events in $\event$ is less than $p
(s+1)$.   Thus
\begin{subeqnarray}
\prob{\event}
& \leq & p(s+1) \prob{\abs{\inv{m}\ip{\Phi X_{i}}{\Phi X_{j}} - \inv{n}\ip{X_{i}}{X_{j}}} \geq \frac{\eta}{4s}} \\
& \leq &
2p(s+1) \exp \left(-\frac{m \eta^2/16s^2}{C_1 + C_2 \eta/4s}\right) \\
& = & 2 p(s+1) \exp \left(-(\ln p + c \ln n + \ln 2(s+1))\right)  \leq \inv{n^c},
\end{subeqnarray}
given that $m \geq 
\left(\frac{16 C_1 s^2}{\eta^2} + \frac{4 C_2 s}{\eta}\right)
(\ln p + c \ln n + \ln 2(s+1))$.
\end{proofof}

\subsection{Proof of Theorem~\ref{thm:tight-Phi}}
\label{sec:append-phi-bound}
\begin{proofof}{Theorem~\ref{thm:tight-Phi}}
We first show that each of the diagonal entries of $\Phi \Phi^T$
is close to its expected value.

We begin by stating state a deviation bound for the $\chi^2_n$ distribution in
Lemma~\ref{lemma:chi-dev} and its corollary, from which we will
eventually derive a  bound on $|R_{i, i}|$.
Recall that the random variable $Q \sim \chi^2_n$ is distributed 
according to the chi-square distribution if $Q = \sum_{i=1}^n Y^2_i$ with 
$Y_i \sim N(0, 1)$ that are independent and normally distributed.

\begin{lemma}\textnormal{\bf{~(\cite{Joh01})}}
\label{lemma:chi-dev} 
\begin{subeqnarray}
\prob{\frac{\chi^2_n}{n} - 1 < -\e} & \leq & 
\exp\left(\frac{-n \e^2}{4}\right),\; \text{for}\; 0 \leq \e \leq 1, \\
\prob{\frac{\chi^2_n}{n} - 1 > \e} & \leq & 
\exp\left(\frac{-3 n \e^2}{16}\right),\; \text{for} \;0 \leq \e \leq \half.
\end{subeqnarray}
\end{lemma}
\begin{corollary}\textnormal
{\bf{(Deviation Bound for Diagonal Entries of $\Phi \Phi^T$)}}
\label{coro:diag-dev}
Given a set of independent normally distributed random variables
$X_1, \ldots, X_n \sim N(0, \sigma^2_X)$, for $0 \leq \e < \half$, 
\begin{gather}
\label{eq:diag-dev}
\prob{\left|\inv{n} \sum_{i=1}^n X^2_i - \sigma^2_X\right|  > \e}
\leq  
\exp\left(\frac{-n \e^2}{4 \sigma_X^4}\right)
+ \exp\left(\frac{-3 n \e^2}{16 \sigma_X^4}\right).
\end{gather}
\end{corollary}

\begin{proof}
Given that $X_1, \ldots, X_n \sim N(0, \sigma^2_X)$, we have
$\frac{X_i}{\sigma_X} \sim N(0, 1)$, and
\begin{gather} 
\sum_{i=1}^n \left(\frac{X_i}{\sigma_X}\right)^2 \sim \chi^2_n,
\end{gather}
Thus by Lemma~\ref{lemma:chi-dev}, we obtain the following:
\begin{subeqnarray}
\prob{\inv{n}{\sum_{i=1}^n \frac{X^2_i}{\sigma^2_X}} - 1 < -\e} & \leq & 
\exp\left(\frac{-n \e^2}{4}\right),\; 0 \leq \e \leq 1 \\
\prob{\inv{n}{\sum_{i=1}^n \frac{X^2_i}{\sigma^2_X}} - 1 > \e} & \leq & 
\exp\left(\frac{-3 n \e^2}{16}\right),\; 0 \leq \e \leq \half.
\end{subeqnarray}
Therefore we have the following by a union bound, for $\e < \half$,
\begin{subeqnarray}
\lefteqn{
\prob{\left|\inv{n} \sum_{i=1}^n X^2_i - \sigma^2_X\right|  > \e} \leq} \\
& & \prob{\sigma^2_X  \left(\frac{\chi^2_n}{n} - 1\right) < -\e}  
+ 
\prob{\sigma^2_X  \left(\frac{\chi^2_n}{n} - 1\right) > \e} \\  
& \leq & 
\prob{\frac{\chi^2_n}{n} - 1 < -\frac{\e}{\sigma^2_X}} +
\prob{\frac{\chi^2_n}{n} - 1 > \frac{\e}{\sigma^2_X}} \\
\\
& \leq & \exp\left(\frac{-n \e^2}{4 \sigma_X^4}\right)
+ \exp\left(\frac{-3 n \e^2}{16 \sigma_X^4}\right).
\end{subeqnarray}
\end{proof}

We next show that the non-diagonal entries of $\Phi \Phi^T$
are close to zero, their expected value.

\begin{lemma}\textnormal{\bf{~(\cite{Joh01})}}
\label{lemma:nondiag-dev} 
Given independent random variables $X_1, \ldots, X_n$, where
$X_1 = z_1 z_2$, with $z_1$ and $z_2$ being independent $N(0, 1)$ 
variables,
\begin{eqnarray}
\prob{\inv{n}\sum_{i=1}^n X_i > 
\sqrt{\frac{b \log n}{n}}} & \leq & C n^{-3b/2}.
\end{eqnarray}
\end{lemma}

\begin{corollary}
\textnormal{\bf{(Deviation Bound for Non-Diagonal Entries of $\Phi \Phi^T$)}}
\label{coro:nondiag-dev} 
Given a collection of i.i.d. random variables 
$Y_1, \ldots, Y_n$, where $Y_i = x_1 x_2$ is a product of two independent 
normal random variables $x_1, x_2 \sim N(0, \sigma^2_X)$, we have
\begin{gather}
\prob{\left|\inv{n}\sum_{i=1}^n Y_i\right|
> \sqrt{\frac{A \log n}{n}}} \;\leq\; 2 C n^{-3A/2\sigma_X^4}.
\end{gather}
\end{corollary}

\begin{proof}
First, we let
\begin{gather}
X_i= \frac{Y_i}{\sigma_X^2} = \frac{x_1}{\sigma_X}\frac{x_2}{\sigma_X}.
\end{gather}
By Lemma~\ref{lemma:nondiag-dev}, symmetry of the events
$\left\{\inv{n}\sum_{i=1}^n X_i < - \sqrt{\frac{b \log n}{n}} \right\}$
and $\left\{\inv{n}\sum_{i=1}^n X_i >\sqrt{\frac{b \log n}{n}} \right\}$,
and a union bound, we have
\begin{gather}
\prob{\left|\inv{n}\sum_{i=1}^n X_i \right|
> \sqrt{\frac{b \log n}{n}}} \leq 2 C n^{-3b/2}.
\end{gather}

Thus we have the following
\begin{subeqnarray}
\label{eq:nondiag-dev}
\prob{\left|\inv{n}\sum_{i=1}^n \frac{Y_i}{\sigma_X^2} \right|
> \sqrt{\frac{b \log n}{n}}}
& = & \prob{\left|\inv{n}\sum_{i=1}^n Y_i\right| > 
{\sigma_X^2 \sqrt{\frac{b \log n}{n}}}} \\
& \leq & 2 C n^{-3b/2},
\end{subeqnarray}
and thus the statement in the Corollary.
\end{proof}

We are now ready to put things together.
By letting each entry of $\Phi_{m \times n}$ to be i.i.d. $N(0, \inv{n})$,
we have for each diagonal entry 
$D = \sum_{i=1}^n X^2_i$, where $X_i \sim N(0, \inv{n})$, 
\begin{gather}
\expct{D} = 1,
\end{gather} 
and 
\begin{subeqnarray}
\prob{\left|\sum_{i=1}^n X^2_i - 1\right| > \sqrt{\frac{b \log n}{n}}}
& = & \prob{\left|\inv{n}\sum_{i=1}^n X^2_i - \sigma^2_X\right| 
> \sqrt{\frac{b \log n}{n^3}}}  \\
& \leq & n^{-b/4} + n^{-3b/16},
\end{subeqnarray}
where the last inequality is obtained by plugging in
 $\e = \sqrt{\frac{b \log n}{n^3}}$ and $\sigma_X^2 = \inv{n}$
 in~(\ref{eq:diag-dev}).

For a non-diagonal entry $W = \sum_{i=1}^n Y_i$, where $Y_i = x_1 x_2$
with independent $x_1, x_2 \sim N(0, \inv{n})$, we have 
\begin{gather}
\expct{W} = 0,
\end{gather}
and 
\begin{eqnarray}
\prob{\left|\sum_{i=1}^n Y_i\right| > \sqrt{\frac{b \log n}{n}}} 
& \leq & 2 C n^{-3b/2},
\end{eqnarray}
by plugging in 
$\sigma_X^2 = \inv{n}$ in (~\ref{eq:nondiag-dev}) directly.

Finally, we apply a union bound, where $b = 2$ for non-diagonal entries
and $b = 16$ for diagonal entries in the following:
\begin{subeqnarray}
\prob{\exists i, j, s.t. |R_{i,j}| > \sqrt{\frac{b \log n}{n}}}
& \leq & 2 C (m^2 - m) n^{-3} + m n^{-4} + m n^{-3} \\
&  = &  O\left(m^2 n^{-3}\right) \; = \; O\left(\inv{n^2 \log n}\right),
\end{subeqnarray}
given that $m^2 \leq \frac{n}{b \log n}$ for $b = 2$.
\end{proofof}

\subsection{Proof of Lemma~\ref{lemma:KKT}}
\label{sec:append-KKT}
\begin{proofof}{Lemma~\ref{lemma:KKT}}
Recall that $Z = \Z = \Phi X$, $W = \W = \Phi Y$, and 
$\omega = \tilde\e = \Phi \e$, and we observe 
$W = Z \beta^* + \omega$.

First observe that the KKT conditions imply that 
$\tilde{\beta} \in \R^p$ is optimal, i.e., $\tilde{\beta} \in \tilde\Omega_m$
for $\tilde\Omega_m$ as defined in~(\ref{eq:solution-set}),
if and only if there exists a subgradient 
\begin{gather}
\tilde{z} \in \partial \norm{\tilde{\beta}}_1 =
\left\{
z \in \R^p \,|\, \text{$z_i = \sign(\tilde{\beta}_i)$
 for $\tilde{\beta}_i \neq 0$, and
$\abs{\tilde{z}_j} \leq 1$ otherwise}
\right\}
\end{gather}
such that
\begin{gather}
\inv{m} Z^T Z \tilde{\beta} - \inv{m} Z^T W + \lambda_m \tilde{z} = 0,
\end{gather}
which is equivalent to the following linear system
by substituting $W = Z \beta^* + \omega$ and re-arranging,
\begin{gather}
\label{eq:opt-beta}
\inv{m} Z^T Z(\tilde{\beta} - \beta^*) - \inv{m} Z^T \omega + 
\lambda_m \tilde{z} = 0.
\end{gather}

Hence, given $Z, \beta^*, \omega$ and $\lambda_m >0$ the event
$\event\left(\sign(\tilde\beta_m) = \sign(\beta^*)\right)$ holds
if and only if
\begin{enumerate}
\item
there exist a point $\tilde{\beta} \in \R^p$ and a subgradient 
 $\tilde{z} \in \partial \norm{\tilde{\beta}}_1$ such that 
~(\ref{eq:opt-beta}) holds, and
\item
$\sign(\tilde{\beta_S}) = \sign(\beta^*_S)$ and 
$\tilde{\beta}_{S^c} = \beta^*_{S^c} = 0$, which 
implies that
$\tilde{z}_{S} = \sign(\beta_S^*)$ and $\abs{\tilde{z}_{S^c}} \leq 1$
by definition of $\tilde{z}$.
\end{enumerate}

\silent{
can be shown to be equivalent to requiring the existence of 
a solution $\tilde{\beta} \in \R^p$ such that
$\sign(\tilde\beta) = \sign(\beta^*)$, 
and a subgradient $\tilde{z} \in \partial \norm{\tilde{\beta}}_1$, 
i.e.,$\tilde{z}_{S} = \sign(\beta_S^*)$ and $\abs{\tilde{z}_{S^c}} \leq 1$,
such that the following equations hold:
\begin{subeqnarray}
\label{leq:Sc}
\inv{m} Z_{S^c}^T Z_S(\tilde{\beta_S} - \beta_S^*) - 
\inv{m} Z_{S^c}^T \tilde\e & = & -\lambda_m \tilde{z}_{S^c}, \\ 
\label{leq:S}
\inv{m} Z_{S}^T Z_S(\tilde{\beta_S} - \beta_S^*) - 
\inv{m} Z_{S}^T \tilde\e & = & -\lambda_m \tilde{z}_{S}.
\end{subeqnarray}
}

Plugging $\tilde{\beta}_{S^c} = \beta^*_{S^c} = 0$ and
$\tilde{z}_S = \sign(\beta^*_S)$ in~(\ref{eq:opt-beta}) allows
us to claim that the event
\begin{equation}
\event\left(\sign(\tilde\beta_m) = \sign(\beta^*)\right)
\end{equation} 
holds if and only
\begin{enumerate}
\item
there exists a point $\tilde{\beta} \in \R^p$
and a subgradient $\tilde{z} \in \partial \norm{\tilde{\beta}}_1$ 
such that the following two sets of equations hold:
\begin{subeqnarray}
\label{eq:Sc}
\inv{m} Z_{S^c}^T Z_S(\tilde{\beta_S} - \beta_S^*) - 
\inv{m} Z_{S^c}^T \omega & = & -\lambda_m \tilde{z}_{S^c}, \\
\label{eq:S}
\inv{m} Z_{S}^T Z_S(\tilde{\beta_S} - \beta_S^*) - 
\inv{m} Z_{S}^T \omega & = & -\lambda_m \tilde{z}_{S} 
= -\lambda_m \sign(\beta_S^*),
\end{subeqnarray}
\item
$\sign(\tilde{\beta}_S) = \sign(\beta^*_S)$ and
$\tilde{\beta}_{S^c} = \beta^*_{S^c} = 0$.
\end{enumerate}

Using invertability of $Z_S^T Z_S$, we can solve 
for $\tilde{\beta_{S}}$ and  $\tilde{z}_{S^c}$ using~(\ref{eq:Sc}) and
~(\ref{eq:S}) to obtain
\begin{subeqnarray}
- \lambda_m \tilde{z}_{S^c} & = &
Z_{S^c}^T Z_S (Z_S^T Z_S)^{-1} \left[
\inv{m}Z_S^T \omega - \lambda_m \sign(\beta^*_S)\right] - 
\inv{m}Z_{S^c}^T \omega, \\
\tilde{\beta}_S & = & 
\beta^*_S + (\inv{m} Z_S^T Z_S)^{-1} 
\left[\inv{m}Z_S^T \omega - \lambda_m \sign(\beta^*_S)\right].
\end{subeqnarray}
Thus, given invertability of $Z_S^T Z_S$, the event
$\event\left(\sign(\tilde\beta_m) = \sign(\beta^*)\right)$ holds
if and only if
\begin{enumerate}
\item
there exists simultaneously a point $\tilde{\beta} \in \R^p$ and a 
subgradient $\tilde{z} \in \partial \norm{\tilde{\beta}}_1$ such that
the following two sets of equations hold:
\begin{subeqnarray}
\label{eq:last-set-a}
- \lambda_m \tilde{z}_{S^c} & = &
Z_{S^c}^T Z_S (Z_S^T Z_S)^{-1} \left[
\inv{m}Z_S^T \omega - \lambda_m \sign(\beta^*_S)\right] - 
\inv{m}Z_{S^c}^T \omega, \\
\label{eq:last-set-b}
\tilde{\beta}_S & = & 
\beta^*_S + (\inv{m} Z_S^T Z_S)^{-1} 
\left[\inv{m}Z_S^T \omega - \lambda_m \sign(\beta^*_S)\right],
\end{subeqnarray}
\item
$\sign(\tilde{\beta}_S) = \sign(\beta^*_S)$ and 
$\tilde{\beta}_{S^c} = \beta^*_{S^c} = 0$.
\end{enumerate}

The last set of necessary and sufficient conditions for the event
$\event\left(\sign(\tilde\beta_m) = \sign(\beta^*)\right)$ to hold
implies that
there exists simultaneously a point $\tilde{\beta} \in \R^p$ 
and a subgradient $\tilde{z} \in \partial \norm{\tilde{\beta}}_1$
such that
\begin{subeqnarray}
\abs{Z_{S^c}^T Z_S (Z_S^T Z_S)^{-1} \left[
\inv{m}Z_S^T \omega - \lambda_m \sign(\beta^*_S)\right] - 
\inv{m}Z_{S^c}^T \omega} & = &
\abs{-\lambda_m \tilde{z}_{S^c}} \leq \lambda_m \\
\sign(\tilde{\beta}_S) = 
\sign\left(\beta^*_S + (\inv{m} Z_S^T Z_S)^{-1} 
\left[\inv{m}Z_S^T \omega - \lambda_m \sign(\beta^*_S)\right]\right)
& = & \sign(\beta^*_S),
\end{subeqnarray}
given that $\abs{\tilde{z}_{S^c}} \leq 1$ by definition of $\tilde{z}$.
Thus~(\ref{eq:lemma-Sc}) and~(\ref{eq:lemma-S}) hold for the given
$Z, \beta^*, \omega$ and $\lambda_m >0$.
Thus we have shown the lemma in one direction.

For the reverse direction, given $Z, \beta^*, \omega$, 
and supposing that~(\ref{eq:lemma-Sc}) and~(\ref{eq:lemma-S}) hold for 
some $\lambda_m > 0$, we first construct a point 
$\tilde{\beta} \in \R^p$ by letting
$\tilde{\beta}_{S^c} = \beta^*_{S^c} = 0$ and
\begin{eqnarray}
\tilde{\beta}_{S} = \beta^*_S + (\inv{m} Z_S^T Z_S)^{-1} 
\left[\inv{m}Z_S^T \omega - \lambda_m \sign(\beta^*_S)\right],
\end{eqnarray}
which guarantees that 
\begin{equation}
\sign(\tilde{\beta}_S) = \sign\left(\beta^*_S + (\inv{m} Z_S^T Z_S)^{-1} 
\left[\inv{m}Z_S^T \omega - \lambda_m \sign(\beta^*_S)\right]\right) =
\sign(\beta^*_S)
\end{equation}
by~(\ref{eq:lemma-S}).
We simultaneously construct $\tilde{z}$ by letting
$\tilde{z}_S = \sign(\tilde{\beta}_S) = \sign(\beta^*_S)$ and
\begin{gather}
\tilde{z}_{S^c} = - \inv{\lambda_m}
\left(Z_{S^c}^T Z_S (Z_S^T Z_S)^{-1} \left[
\inv{m}Z_S^T \omega - \lambda_m \sign(\beta^*_S)\right] - 
\inv{m} Z_{S^c}^T \omega \right),
\end{gather}
which guarantees that $\abs{\tilde{z}_{S^c}} \leq 1$ due to~(\ref{eq:lemma-S});
hence $\tilde{z} \in \partial \norm{\tilde{\beta}}_1$.
Thus we have found a point $\tilde{\beta} \in \R^p$ and a 
subgradient $\tilde{z} \in \partial \norm{\tilde{\beta}}_1$ such that 
 $\sign(\tilde{\beta}) = \sign(\beta^*)$ and the set of 
equations~(\ref{eq:last-set-a}) and~(\ref{eq:last-set-b}) is satisfied.
Hence, assuming the invertability of $Z_S^T Z_S$, the event
$\event\left(\sign(\tilde\beta_m) = \sign(\beta^*)\right)$ holds
for the given $Z, \beta^*, \omega, \lambda_m$.
\end{proofof}

\subsection{Proof of Lemma~\ref{lemma:convergence}}
\label{sec:append-convergence}
\begin{proofof}{Lemma~\ref{lemma:convergence}}
Given that $\inv{m}Z_{S}^T Z_S = \tA + I_s$,
we bound $\norm{(\inv{m} Z_S^T Z_S)^{-1}}_{\infty}$ through
$\norm{(\tA + I_s)^{-1}}$.

First we have for 
$m \geq \left(\frac{16 C_1 s^2}{\eta^2} + \frac{4 C_2 s}{\eta}\right) 
(\ln p + c \ln n + \ln 2(s+1))$,
\begin{eqnarray}
\norm{\tA}_{\infty} \leq \norm{A}_{\infty} + \frac{\eta}{4}
\leq 1 - \eta + \eta/4 =  1- 3\eta/4, 
\end{eqnarray}
where $\eta \in (0, 1]$, due to~(\ref{eq:eta}) and~(\ref{eq:CNorm2}).
Hence, given that $\norm{I}_{\infty} =1$ and $\norm{\tA}_{\infty} < 1$,
by Proposition~\ref{pro:inverse}, 
\begin{eqnarray}
\label{eq:CNorm}
\norm{\left(\inv{m}Z_{S}^T Z_S\right)^{-1}}_{\infty} 
= \norm{(\tA + I_s)^{-1}}_{\infty} 
 \leq \inv{1 - \norm{\tA}_{\infty}} \leq \frac{4}{3 \eta}.
\end{eqnarray}

Similarly, given $\norm{A}_{\infty} < 1$, we have 
\begin{gather}
\inv{1 +\norm{A}_{\infty}} \leq 
\norm{\left(\inv{n}X_{S}^T X_S\right)^{-1}}_{\infty} =
\norm{(A + I_s)^{-1}}_{\infty} 
\leq \inv{1 - \norm{A}_{\infty}}.
\end{gather}

Given that 
$\frac{\lambda_m}{\rho_m}
\norm{\left(\inv{n} X_{S}^T X_S \right)^{-1}}_{\infty} \rightarrow 0$,
we have 
$\frac{\lambda_m}{\rho_m}\inv{1 +\norm{A}_{\infty}} \rightarrow 0$, 
and thus
\begin{subeqnarray}
\frac{\lambda_m}{\rho_m} \inv{1 - \norm{\tA}_{\infty}}
& = & 
\frac{\lambda_m}{\rho_m}\inv{1 +\norm{A}_{\infty}}
\frac{1 +\norm{A}_{\infty}}{1 - \norm{\tA}_{\infty}} \\
& \leq & 
\frac{\lambda_m}{\rho_m}\inv{1 +\norm{A}_{\infty}}
\left(\frac{4(2 - \eta)}{3 \eta}\right) \\
& \rightarrow & 0,
\end{subeqnarray}
by~(\ref{eq:CNorm}) and the fact that by~\eqref{eq:eta},
$1 + \norm{A}_{\infty} \leq 2 - \eta $. 
\end{proofof}

\subsection{Proof of Claim~\ref{claim:MNorm}}
\label{sec:append-MNorm}
\begin{proofof}{Claim~\ref{claim:MNorm}}
We first prove the following.
\begin{claim}
\label{claim:BNorm}
If $m$ satisfies~(\ref{eq:thm-m-bounds}),  then
$\inv{m}\max_{i, j} (B_{i, j}) \leq 1 + \frac{\eta}{4s}$.
\end{claim}

\begin{proof}
Let us denote the $i^{th}$ column in $Z_S$ with $Z_{S, i}$. 
Let $x = Z_{S, i}$ and $y = Z_{S, j}$ be $m \times 1$ vectors.
By Proposition~\ref{pro:Phi-X},
$\twonorm{x}^2, \twonorm{y}^2 \leq m\left((1 + \frac{\eta}{4s}\right)$.
We have by function of $x, y$,
\begin{subeqnarray}
B_{i, j} & = &  Z^T_{S, i} R Z_{S, j} 
= \sum_{i=1}^m \sum_{j=1}^m x_i y_j R_{i, j} 
\leq  \sum_{i=1}^m \sum_{j=1}^m |x_i| |y_j| |R_{i, j}| \\
& \leq & \max_{i, j}|R_{i, j}| \sum_{i=1}^m \sum_{j=1}^m |x_i| |y_j| 
 = \max_{i, j}|R_{i, j}| (\sum_{i=1}^m |x_i|)(\sum_{j=1}^m |y_j|) \\
& \leq & \max_{i, j}|R_{i, j}| m \twonorm{x} \twonorm{y} \leq
 \max_{i, j}|R_{i, j}| m^2 \left(1 + \frac{\eta}{4s}\right).
\end{subeqnarray}
Thus the claim follows given that $\max_{i, j}|R_{i, j}|
\leq  4 \sqrt{\frac{\log n}{n}}$ and $4m \leq \sqrt{\frac{n}{\log n}}$.
\end{proof}

Finally, to finish the proof of Claim~\ref{claim:MNorm} we have
\begin{subeqnarray}
\max_{i} M_{i, i} & = & 
\max_{i} \frac{C^T_i B C_i}{m} = \inv{m} \max_{i} C^T_i B C_i =
\inv{m}\max_{i} 
\left(\sum_{j = 1}^m \sum_{k=1}^m C_{i, j} C_{i, k} B_{j, k}\right) \\
& \leq & 
\inv{m} \max_{i, j} |B_{i, j}| 
\max_{i} \left(\sum_{j = 1}^m |C_{i, j}| \sum_{k=1}^m |C_{i, k}|\right) \\
& \leq & 
\left(1 + \frac{\eta}{4s}\right)
\max_{i} \left(\sum_{j = 1}^m |C_{i, j}|\right)^2  \leq
\left(1 + \frac{\eta}{4s}\right)\left(\max_{i} \sum_{j = 1}^m |C_{i, j}| \right)^2 \\
& \leq & 
\left(1 + \frac{\eta}{4s}\right)\norm{C}_{\infty}^2 \leq 
\left(1 + \frac{\eta}{4s}\right)\left(\frac{4}{3 \eta}\right)^2,
\end{subeqnarray}
where $\norm{C}_{\infty} = 
\norm{\left(\inv{m} Z_S^T Z_S\right)^{-1}}_{\infty}
\leq  \frac{4}{3 \eta}$ as in~(\ref{eq:CNorm}) for
$m \geq \left(\frac{16 C_1 s^2}{\eta^2} + \frac{4 C_2 s}{\eta}\right) 
(\ln p + c \ln n + \ln 2(s+1))$.
\end{proofof}

\begin{remark}
In fact, $\max_{i, j} M_{i, j} = \max_{i, i} M_{i, i}$.
\end{remark}

\section{Discussion}
\label{sec:discuss}

The results presented here suggest several directions for future work.
Most immediately, our current sparsity analysis holds for
compression using random linear transformations.  However,
compression with a random affine mapping $X\mapsto \Phi X +
\Delta$ may have stronger privacy properties; we expect
that our sparsity results can be extended to this case.
While we have studied data compression by random projection of columns
of $X$ to low dimensions, one also would like to consider projection of the
rows, reducing $p$ to a smaller number of effective variables.
However, simulations suggest that the strong sparsity recovery properties
of $\ell_1$ regularization are not preserved under projection of
the rows.

It would be natural to investigate the effectiveness of other
statistical learning techniques under compression of the data.  For
instance, logistic regression with $\ell_1$-regularization has
recently been shown to be effective in isolating relevant variables in
high dimensional classification problems \citep{wain:07}; we expect
that compressed logistic regression can be shown to have similar
theoretical guarantees to those shown in the current paper.  It would
also be interesting to extend this methodology to nonparametric
methods.  As one possibility, the rodeo is an approach to sparse
nonparametric regression that is based on thresholding derivatives of
an estimator \citep{Rodeo}.  Since the rodeo is based on kernel
evaluations, and Euclidean distances are approximately preserved under
random projection, this nonparametric procedure may still be effective
under compression.

The formulation of privacy in Section~\ref{sec:privacy} is, arguably, weaker than the
cryptographic-style guarantees sought through, for example,
differential privacy \citep{Dwork:06}.  In particular,
our analysis in terms of average mutual information may not preclude
the recovery of detailed data about a small number of individuals.
For instance, suppose that a column $X_j$ of $X$ is very sparse,
with all but a few entries zero.  Then the
results of compressed sensing \citep{Candes:Romberg:Tao:06} imply that, given knowledge of
the compression matrix $\Phi$, this column can be
approximately recovered by solving the compressed sensing linear
program 
\begin{subeqnarray}
\min && \hskip-10pt \|X_j\|_1 \\
\text{such that} && \hskip-10pt Z_j = \Phi X_j.
\end{subeqnarray}
However, crucially, this requires knowledge of the compression matrix
$\Phi$; our privacy protocol requires that this matrix is not known to
the receiver.  Moreover, this requires that the column is sparse; such
a column cannot have a large impact on the predictive accuracy of the
regression estimate.  If a sparse column is removed, the resulting
predictions should be nearly as accurate as those from an estimator
constructed with the full data.  We leave the analysis of this case 
this as an interesting direction for future work.

\section{Acknowledgments} This research was supported in part by NSF
grant CCF-0625879.  We thank Avrim Blum, Steve Fienberg, 
and Pradeep Ravikumar for helpful comments on this work,
and Frank McSherry for making~\cite{Dwork:07} available to us.
\bibliography{cr}

\end{document}